\documentclass{article}

\usepackage[nonatbib, final]{neurips_2021}
\usepackage[sort,numbers]{natbib}

\usepackage[utf8]{inputenc} 
\usepackage[T1]{fontenc}    
\usepackage{hyperref}       
\usepackage{url}            
\usepackage{booktabs}       
\usepackage{amsfonts}       
\usepackage{nicefrac}       
\usepackage{microtype}      
\usepackage{xcolor}         

\usepackage{microtype}
\usepackage{graphicx}
\usepackage{subfig}
\usepackage{floatrow}

\usepackage{soul}
\usepackage{xcolor}
\usepackage{booktabs} 
\usepackage[multiple]{footmisc}
\usepackage{bm}
\usepackage{soul,color}
\usepackage{amsmath}
\usepackage{amssymb}
\usepackage{amsthm}
\usepackage{graphicx}
\usepackage{url}

\usepackage{placeins}
\usepackage{comment}
\newsavebox{\imagebox}
\usepackage{bbm}
\usepackage{refcount}

\usepackage{wrapfig}

\newtheorem{proposition}{Proposition}

\newtheorem{definition}{Definition}
\newtheorem{theorem}{Theorem}

\newcommand{\HV}{\textsc{HV}}
\newcommand{\HVI}{\textsc{HVI}}
\newcommand{\EHVI}{\textsc{EHVI}}
\newcommand{\EHVIPM}{\textsc{EHVI-PM}}
\newcommand{\NEHVI}{\textsc{NEHVI}}
\newcommand{\HVIc}{\textsc{HVI}_\textsc{c}}

\newcommand{\qEHVI}{$q$\textsc{EHVI}}
\newcommand{\qNEHVI}{$q$\textsc{NEHVI}}
\newcommand{\TSHVI}{$q\textsc{NEHVI-1}$}

\newcommand{\xcand}{\mathcal{X}_\text{cand}}
\newcommand{\xcandvec}{\bm{x}_\text{cand}}
\newcommand{\ycandvec}{\bm{y}_\text{cand}}
\newcommand{\xprev}{\bm{x}_\text{prev}}
\newcommand{\aEHVI}{\alpha_{\textsc{EHVI}}}
\newcommand{\aNEHVI}{\alpha_{\textsc{NEHVI}}}
\newcommand{\aqNEHVI}{\alpha_{q\textsc{NEHVI}}}
\newcommand{\aqEHVI}{\alpha_{q\textsc{EHVI}}}

\newcommand{\hataqEHVI}{\hat{\alpha}_{q\textsc{EHVI}}}
\newcommand{\hataNEHVI}{\hat{\alpha}_{\textsc{NEHVI}}}
\newcommand{\hataqNEHVI}{\hat{\alpha}_{q\textsc{NEHVI}}}
\newcommand{\aqEHVIc}{\alpha_{q\textsc{EHVI}_\textsc{c}}}
\newcommand{\aEHVIc}{\alpha_{\textsc{EHVI}_\textsc{c}}}
\newcommand{\aNEHVIc}{\alpha_{\textsc{NEHVI}_\textsc{c}}}
\newcommand{\qParego}{$q$\textsc{ParEGO}}
\newcommand{\cbd}{\textsc{CBD}}
\newcommand{\iep}{\textsc{IEP}}

\DeclareMathOperator*{\argmax}{arg\,max}

\newcommand{\papertitle}{Parallel Bayesian Optimization of Multiple Noisy Objectives with Expected Hypervolume Improvement}
\title{\papertitle}

\author{%
  Samuel Daulton\\
  Facebook, University of Oxford\\
  \texttt{sdaulton@fb.com} \\
  \And
  Maximilian Balandat\\
  Facebook\\
  \texttt{balandat@fb.com} \\
  \And
  Eytan Bakshy\\
  Facebook\\
  \texttt{ebakshy@fb.com} \\
}

\begin{document}

\maketitle

\begin{abstract}
Optimizing multiple competing black-box objectives is a challenging problem in many fields, including science, engineering, and machine learning. Multi-objective Bayesian optimization (MOBO) is a sample-efficient approach for identifying the optimal trade-offs between the objectives. However, many existing methods perform poorly when the observations are corrupted by noise. 
We propose a novel acquisition function, \NEHVI{}, that overcomes this important practical limitation by applying a Bayesian treatment to the popular expected hypervolume improvement (\EHVI{}) criterion and integrating over this uncertainty in the Pareto frontier. 
We argue that, even in the noiseless setting, generating multiple candidates in parallel is an incarnation of \EHVI{} with uncertainty in the Pareto frontier and therefore can be addressed using the same underlying technique. Through this lens, we derive a natural parallel variant, \qNEHVI{}, that reduces computational complexity of parallel \EHVI{} from \emph{exponential} to \emph{polynomial} with respect to the batch size. 
\qNEHVI{} is one-step Bayes-optimal for hypervolume maximization in both noisy and noiseless environments, and we show that it can be optimized effectively with gradient-based methods via sample average approximation. 
Empirically, we demonstrate not only that \qNEHVI{} is substantially more robust to observation noise than existing MOBO approaches, but also that it achieves state-of-the-art optimization performance and competitive wall-times in large-batch environments.
\end{abstract}

\vspace{-1ex}
\section{Introduction}
\vspace{-1ex}
\label{sec:introduction}
Black-box optimization problems that involve multiple competing noisy objectives are ubiquitous in science and engineering. For example, a real-time communications service may be interested in tuning the parameters of a control policy to adapt video quality in real time in order to maximize video quality and minimize latency \citep{daulton2019thompson, feng2020}.
In robotics, scientists may seek to design hardware components that maximize locomotive speed and minimize energy expended \citep{ calandra2016gait,liao2019data}.  In agriculture, development agencies may seek to balance crop yield and environmental impact \citep{jiang2020multi}. For such multi-objective optimization (MOO) problems, there typically is no single solution that is best with respect to all objectives. Rather, the goal is to identify the \emph{Pareto frontier}: a set of optimal trade-offs such that improving one objective means deteriorating another. In many cases, the objectives are expensive to evaluate. For instance, randomized trials used in agriculture and the internet industry may take weeks or months to conduct and incur opportunity costs, and manufacturing and testing hardware is both costly and time-consuming. Therefore, it is imperative to be able to identify good trade-offs with as few objective evaluations as possible.

Bayesian optimization (BO), a method for efficient global black-box optimization, is often used to tackle such problems. BO employs a probabilistic surrogate model in conjunction with an acquisition function to navigate the trade-off between exploration (evaluating designs with high uncertainty) and exploitation (evaluating designs that are believed to be optimal). 
Although a significant number of works have explored multi-objective Bayesian optimization (MOBO), most available methods \citep{yang_palar19, belakaria2019, pfes, konakovic2020diversity} do not take into account the fact that, in practice, observations are often subject to noise. For example, results of an A/B test are highly variable due to heterogeneity in the underlying user population and other factors. Agricultural trials are affected by the stochastic nature of plant growth and environmental factors such as soil composition or wind currents. In robotics, devices are subject to manufacturing tolerances, and observations of quantities such as locomotive speed and efficiency may be corrupted by measurement error from noisy sensors and environmental factors such as temperature or surface friction. While previous work has shown that a principled treatment of noisy observations can significantly improve optimization performance in the single-objective case \citep{hernandezlobato2014pes,letham2019noisyei}, this issue is understudied in the multi-objective setting.
Furthermore, many applications in which evaluations take a long time require evaluating large \emph{batches} of candidates in parallel in order to achieve reasonable throughput. For example, when firms optimize systems via A/B tests, it may take several weeks to test any particular configuration. Because of this, large batches of candidate policies are tested simultaneously  \citep{letham2019mtbo}. In biochemistry and materials design, dozens of tests can be conducted parallel on a single microplate~\citep{zhang2009using}. Even in sophisticated high throughput chemistry settings, these batches may take several hours or days to set up and evaluate~\citep{mennen2019pharma}. 
Most existing MOBO methods, however, are either designed for purely sequential optimization \citep{belakaria2019, pfes} or do not scale well to large batch sizes \citep{daulton2020ehvi}.

\textbf{Contributions:} In this work, we propose a novel MOBO algorithm, based on expected hypervolume improvement (EHVI), that scales to highly parallel evaluations of noisy objectives.  Our approach is made possible by a general-purpose, differentiable, cached box decomposition (CBD) implementation that dramatically speeds up critical computations needed to account for uncertainty introduced by noisy observations and generate new candidate points for highly parallel batch or asynchronous evaluation. 
In particular, our CBD-based approach solves the fundamental problem of scaling parallel EHVI-based methods to large batch sizes,  \emph{reducing time and space complexity from exponential to polynomial}.
Our proposed algorithm, \emph{noisy} expected hypervolume improvement (\NEHVI{}), is the  one-step Bayes-optimal policy for hypervolume improvement and provides state-of-the-art performance across a variety of benchmarks.
To our knowledge, our work provides the most extensive evaluation of noisy parallel MOBO to date.
A high-quality implementation of \qNEHVI{}, as well as many of the baselines considered here, will be made available as open-source software upon publication.

\vspace{-2ex}
\section{Preliminaries}
\vspace{-1ex}
Our goal 
is to find the set of optimal designs $\bm x$ over a bounded set $\mathcal X \subset \mathbb R^d$ that maximize one or more objectives $\bm f(\bm x) \in \mathbb R^M$, with no known analytical expression nor gradient information of $\bm f$.

\textbf{Multi-Objective Optimization (MOO)} aims to identify the set of \emph{Pareto optimal} objective trade-offs. We say a solution $\bm f(\bm x) = \begin{bmatrix}
f^{(1)}(\bm x), ..., f^{(M)}(\bm x)
\end{bmatrix}$ \emph{dominates} another solution $\bm f(\bm x) \succ \bm f(\bm x')$ if $f^{(m)}(\bm x) \geq f^{(m)}(\bm x')$ for $m=1, ..., M$ and $\exists \,m \in \{1, ..., M\}$ s.t. $f^{(m)}(\bm x) > f^{(m)}(\bm x')$. We define the \emph{Pareto frontier} as $\mathcal P^* = \{\bm f(\bm x) : \bm x \in \mathcal X, ~\nexists ~\bm x'\in \mathcal X ~s.t.~ \bm f(\bm x') \succ \bm f(\bm x)\}$, and denote the set of Pareto optimal designs as $\mathcal X^* =\{\bm x : \bm f(\bm x) \in \mathcal P^*\}$. Since the Pareto frontier (PF) is often an infinite set of points, MOO algorithms usually aim to identify a finite approximate PF ~$\mathcal P$. A natural measure of the quality of a PF is the hypervolume of the region of objective space that is dominated by the PF and bounded from below by a reference point. Provided with the approximate PF, the decision-maker can select a particular Pareto optimal trade-off according to their preferences.

\textbf{Bayesian Optimization (BO)} is a sample-efficient optimization method that leverages a probabilistic surrogate model to make principled decisions to balance exploration and exploitation~\citep{shahriari16, frazier2018tutorial}.
Typically, the surrogate is a Gaussian Process (GP), a flexible, non-parametric model known for its well-calibrated predictive uncertainty \citep{Rasmussen2004}. To decide which points to evaluate next, BO employs an acquisition function $\alpha(\cdot)$ that specifies the value of evaluating a set of new points $\bm x$ based on the surrogate's predictive distribution at $\bm$. 
While evaluating the true black-box function $\bm f$ is time-consuming or costly, evaluating the surrogate is cheap and relatively fast; therefore, numerical optimization can be used to find the maximizer of the acquisition function $\bm x^* = \argmax_{\bm x \in \mathcal X} \alpha(\bm x)$ to evaluate next on the black-box function. BO sequentially selects new points to evaluate and updates the model to incorporate the new observations. 

Evolutionary algorithms (EAs) such as NSGA-II \citep{deb02nsgaii} are a popular choice for solving MOO problems (see \citet{zitzler2000comparison} for a review of various other approaches). However, EAs generally suffer from high sample complexity, rendering them infeasible for optimizing expensive-to-evaluate black-box functions. Multi-objective Bayesian optimization (MOBO), which combines a Bayesian surrogate with an acquisition function designed for MOO, provides a much more sample-efficient alternative. 

\vspace{-1ex}
\section{Related Work}
\vspace{-1ex}
\label{sec:related_work}
Methods based on hypervolume improvement (HVI) seek to expand the volume of the objective space dominated by the Pareto frontier. Expected hypervolume improvement (\EHVI{}) \citep{emmerich2006} is a natural extension of the popular expected improvement (EI) \citep{jones98} acquisition function to the MOO setting. Recent work has led to efficient computational paradigms using box decomposition algorithms \citep{yang_emmerich2019} and practical enhancements such as support for parallel candidate generation and gradient-based acquisition optimization \citep{yang2019,daulton2020ehvi}. However, \EHVI{} still suffers from some limitations, including (i) the assumption that observations are noise-free, and (ii) the exponential scaling of its batch variant, \qEHVI{}, in the batch size $q$, which precludes large-batch optimization.  DGEMO \citep{konakovic2020diversity} is a recent method for parallel MOBO that greedily maximizes HVI while balancing the diversity of the design points being sampled.  Although DGEMO scales well to large batch sizes, it does not account for noisy observations. TSEMO \citep{bradford2018tesmo} is a Thompson sampling (TS) heuristic that can acquire batches of points by optimizing a random fourier feature (RFF) \citep{rahimi_rff}  approximation of a GP surrogate using NSGA-II and selecting a subset of points from the EA's population to sequentially greedily maximize HVI. This heuristic approach for maximizing HVI currently has no theoretical guarantees and relies on zeroth-order optimization methods, which tend to be slower and exhibit worse optimization performance than gradient-based approaches.

Entropy-based methods such as PESMO \citep{pesmo}, MESMO \citep{belakaria2019}, and PFES \citep{pfes} are an alternative to \EHVI{}.
Of these three methods, PESMO is the only one that accounts for observation noise. However, PESMO involves intractable entropy computations and therefore relies on complex approximations, as well as challenging and time-consuming numerical optimization procedures \citep{pesmo}.
\citet{garridomerchn2020parallel} recently proposed an extension to PESMO that supports parallel candidate generation. However, the authors of this work provide limited evaluation and have not provided code to reproduce their results.\footnote{We contacted the authors twice asking for code to reproduce their results, but they graciously declined.}

MOO can also be cast into a single-objective problem by applying a random scalarization of the objectives. ParEGO maximizes the expected improvement using random augmented Chebyshev scalarizations \citep{parego}. MOEA/D-EGO \citep{zhang2010moad} extends ParEGO to the batch setting using multiple random scalarizations and the genetic algorithm MOEA/D \citep{moead} to optimize these scalarizations in parallel. Recently, $q$ParEGO, another batch variant of ParEGO was proposed that uses compositional Monte Carlo objectives and sequential greedy candidate selection \citep{daulton2020ehvi}. Additionally, the authors proposed a noisy variant, $q$NParEGO, but the empirical evaluation of that variant was limited. TS-TCH~\citep{paria2018randscalar} combines random Chebyshev scalarizations with Thompson sampling \citep{thompson}, which is naturally robust to noise when the objective is scalarized. \citet{golovin2020random} propose to use a hypervolume scalarization with the property that the expected value of the scalarization over a specific distribution of weights is equivalent to the hypervolume indicator. The authors propose a upper confidence bound algorithm using randomly sampled weights, but provide a very limited empirical evaluation.

Many prior attempts by the simulation community to handle MOO with noisy observations found that accounting for the noise did not improve optimization performance: \citet{horn2017} suggest that the best approach is to ignore noise, and \citet{koch2015} concluded that further research was needed to determine if modeling techniques such as re-interpolation could improve BO performance with noisy observations. In contrast, we find that accounting for noise \emph{does substantially} improve performance in noisy settings.

Lastly, previous works have considered methods for quantifying and monitoring uncertainty in the Pareto frontiers during the optimization \citep{Calandra2014ParetoFM, Binois2015QuantifyingUO}. In contrast, we provide a solution to performing MOBO in noisy settings, rather than purely reasoning about the uncertainty in the Pareto frontier.

\section{Background on Expected Hypervolume Improvement}
In this section, we review hypervolume, hypervolume improvement, and expected hypervolume improvement as well as efficient methods for computing these metrics using box decompositions.

\begin{definition}
The hypervolume indicator (\HV{}) of a finite approximate Pareto frontier $\mathcal P$ is the $M$-dimensional Lebesgue measure $\lambda_M$ of the space dominated by $\mathcal P$ and bounded from below by a reference point. 
$\bm r \in \mathbb R^M$: $\HV(\mathcal P | \bm r) = \lambda_M \big(\bigcup_{\bm v \in \mathcal P} [\bm r, \bm v]\big)$, where $[\bm r, \bm v]$ denotes the hyper-rectangle bounded by vertices $\bm r$ and $\bm v$.
\end{definition}
As in previous work, we assume that the reference point $\bm r$ is known and specified by the decision maker \citep{yang2019}.
\begin{definition}
\label{def:hvi}
The hypervolume improvement (\HVI{}) of a set of points $\mathcal P'$ w.r.t. 
an existing approximate Pareto frontier $\mathcal P$ and reference point $\bm r$ is defined as\footnote{For brevity we omit the reference point $\bm r$ when referring to \HVI{}.} $\HVI{}(\mathcal P' | \mathcal P, \bm r) = \HV{}(\mathcal P \cup \mathcal P'| \bm r) -   \HV{}(\mathcal P | \bm r)$.
\end{definition}
\vspace{-1ex}
Computing HV requires calculating the volume of a typically non-rectangular polytope and is known to have time complexity that is super-polynomial in the number of objectives \citep{yang_emmerich2019}. An efficient approach for computing HV is to (i) decompose the region that is dominated by the Pareto frontier $\mathcal P$ and bounded from below by the reference point $\bm r$ into disjoint axis-aligned hyperrectangles \citep{LACOUR2017347}, (ii) compute the volume of each hyperrectangle in the decomposition, and (iii) sum over all hyperrectangles. So-called box decomposition algorithms have also been applied to partition the region that is \emph{not dominated} by the Pareto frontier $\mathcal P$, which can be used to compute the HVI from a set of new points \citep{DACHERT2017, yang_emmerich2019}. See Appendix \ref{appdx:sec:box_decompositions} for further details.

\textbf{Expected Hypervolume Improvement:}
Since function values at unobserved points are unknown in black-box optimization, so is the HVI of an out-of-sample point. However, in  BO the probabilistic surrogate model provides a posterior distribution $p(\bm f(\bm x)| \mathcal D)$ over the function values for each $\bm x$, which can be used to compute the expected hypervolume improvement (\EHVI) acquisition function: $\aEHVI{}(\bm x | \mathcal P) = \mathbb E\big[\HVI{}(\bm f(\bm x)| \mathcal P)\big].$
Although $\alpha_{\EHVI{}}$ can be expressed analytically when (i) the objectives are assumed to be conditionally independent given $\bm x$  and (ii) the candidates are generated and evaluated sequentially \citep{yang2019}, Monte Carlo (MC) integration is commonly used since it does not require either assumption \citep{emmerich2006}. The more general parallel variant using MC integration is given by\\[-3ex]
\begin{align}
    \aqEHVI{}(\xcand | \mathcal P) \approx  \hataqEHVI{}(\xcand | \mathcal P) = \frac{1}{N}\sum_{t=1}^N \HVI(\tilde{\bm{f}}_t(\xcand)| \mathcal P),
\end{align}
where $\tilde{\bm f}_t \sim p(\bm f|\mathcal D)$ for $t=1,..., N$ and $\xcand = \{x_i\}_{i=1}^q$ \citep{daulton2020ehvi}. The same box decomposition algorithms used to compute \HVI{} can be used to compute \EHVI{} (either analytic or via MC) using piece-wise integration. \EHVI{} computation is agnostic to the choice of box decomposition algorithm (and can also use approximate methods \citep{couckuyt12}).
Similar to EI in the single-objective case, \EHVI{} is a one-step  Bayes-optimal algorithm for maximizing hypervolume in the MOO setting
under the following assumptions: (i) only a single design will be generated and evaluated, (ii) the observations are noise-free, (iii) the final approximate Pareto frontier (and final design that will be deployed) will be drawn from the set of observed points \citep{frazier2018tutorial}.
\vspace{-1ex}
\section{Expected Hypervolume Improvement with Noisy Observations}
\vspace{-1ex}
\label{sec:nehvi}
\FloatBarrier
We consider the case that frequently arises in practice where we only receive noisy observations $\bm y_i = \bm f(\bm x_i) + \bm \epsilon_i$, $\bm \epsilon_i \sim \mathcal N(0, \Sigma_i)$,
where $\Sigma_i$ is the noise covariance.
%
In this setting, \EHVI{} is no longer (one-step) Bayes-optimal. This is because we can no longer compute the true Pareto frontier $\mathcal P_n = \{\bm f(\bm x) ~|~ \bm x \in X_n, ~\nexists~ \bm x' \in X_n  ~s.t.~ \bm f(\bm x') \succ \bm f(\bm x)\}$ over the previously evaluated points $X_n = \{\bm x_i\}_{i=1}^n$. Simply using the observed Pareto frontier, $\mathcal Y_n = \{\bm y~|~ \bm y \in Y_n, ~\nexists~ \bm y' \in Y_n  ~s.t.~ \bm y' \succ \bm y, y\}$ where $Y_n = \{\bm y_i\}_{i=1}^n$, can have strong detrimental effects on optimization performance. This is illustrated in Figure~\ref{fig:noise_study}, which shows how \EHVI{} is misled by noisy observations that appear to be Pareto optimal. \EHVI{} proceeds to spend its evaluation budget trying to optimize noise,  resulting in a clumped Pareto frontier that lacks diversity. 
Although the posterior mean could serve as a "plug-in" estimate of the true function values at the observed points and provide some regularization \citep{yang_multipoints2019}, we find that this heuristic also leads to clustered Pareto frontiers (\EHVIPM{} in Fig.~\ref{fig:noise_study}). 
Similar patterns emerge with DGEMO (which does not account for noise), and other baselines that utilize the posterior mean rather than the observed values when computing hypervolume improvement (see Appendix~\ref{appdx:sec:extra_experiments}). To our knowledge, all previous work on \EHVI{} assumes that observations are noiseless \citep{emmerich2006, yang2019} or imputes the unknown true function values with the posterior mean. 

\begin{figure}[t]
\centering
\floatbox[{\capbeside\thisfloatsetup{capbesideposition={right,center},capbesidewidth=0.65\textwidth}}]{figure}[\FBwidth]
{\caption{\label{fig:noise_study} An illustration of the effect of noisy observations on the true noiseless Pareto frontiers identified by \NEHVI{} (our proposed algorithm), \EHVI{}, and \EHVI{}-PM, which uses the modeled posterior mean as point estimate of the true in-sample function values. All algorithms are tested on a BraninCurrin synthetic problem, where observations are corrupted with zero-mean, additive Gaussian noise with a standard deviation of 5\% of the range of respective objective. All methods use sequential ($q=1$) optimization. See Appendix \ref{appdx:sec:experiment_details} for details.}}
{\includegraphics[width=0.3\textwidth]{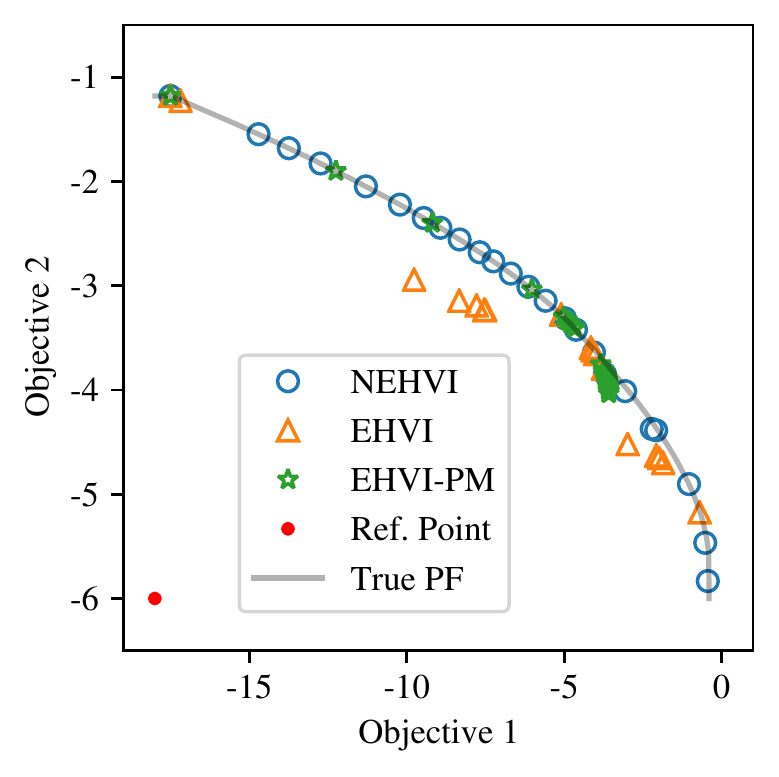}}
\vspace{-4ex}
\end{figure}


\subsection{A Bayes-optimal algorithm for hypervolume maximization in noisy environments}
\label{subsec:ideal_nehvi}
In contrast with EHVI(-PM), we instead approach the problem of hypervolume maximization under noisy observations from a Bayesian perspective and derive a novel one-step Bayes-optimal expected hypervolume improvement criterion that iterates the expectation over the posterior distribution $p(\bm f(X_n) |\mathcal D_n)$ of the function values at the previously evaluated points $X_n$ given noisy observations $\mathcal D_n = \{\bm x_i, \bm y_i, (\Sigma_i)\}_{i=1}^n$. Our acquisition function, \emph{noisy expected hypervolume improvement} (\NEHVI{}), is defined as
\begin{equation}
    \label{eqn:ideal_nehvi}
    \aNEHVI{}(\bm x) = \int \aEHVI{}(\bm x| \mathcal P_n)p(\bm f|\mathcal D_n) d\bm{f}
\end{equation}
where $P_n$ denotes the Pareto frontier over $\bm f(X_n)$. 

By integrating over the uncertainty in the function values at the observed points, \NEHVI{} retains one-step Bayes-optimality in noisy environments (in noiseless environments, \NEHVI{} is equivalent to \EHVI{}). Empirically, Figure~\ref{fig:noise_study} shows that \NEHVI{} is robust to noise and identifies a well-distributed Pareto frontier with no signs of clumping, even under very noisy observations.\footnote{This noise level is 5x greater than the ones considered by previous works that evaluate noisy MOBO \citep{pesmo}.} 

The integral in \eqref{eqn:ideal_nehvi} is analytically intractable, but can easily be approximated using MC integration. Let $\tilde{\bm f}_t \sim p(\bm f | D_n)$ for $t=1, ... N$ be samples from the posterior, and let $\mathcal P_t = \{\tilde{\bm f}_t(\bm x) ~|~ \bm x \in X_n, \tilde{\bm f}_t(\bm x) \succ \tilde{\bm f}_t(\bm x') ~\forall~ \bm x' \in X_n\}$ be the Pareto frontier over the previously evaluated points under the sampled function $\tilde{\bm f}_t$. Then, $\aNEHVI{}(\bm x) \approx \frac{1}{N}\sum_{t=1}^{N} \aEHVI(\bm x | \mathcal P_t)$.
Using MC integration, we can compute the inner expectation in $\aEHVI{}$ simultaneously using samples from the joint posterior $\tilde{\bm f}_t(X_n, \bm x) \sim p(\bm f(X_n, \bm x) | \mathcal D_n)$ over $\bm x$ and $X_n$:
\begin{equation}
\begin{aligned}
\label{eqn:mc_nehvi}
\hataNEHVI{}(\bm x) &= \frac{1}{N}\sum_{t=1}^{N} \HVI(\tilde{\bm f}_t(\bm x) | \mathcal P_t).
\end{aligned}
\end{equation}
See Appendix~\ref{appdx:sec:box_decompositions} for details on computing \eqref{eqn:mc_nehvi} using box decompositions. Note that this ``full-MC'' variant of \NEHVI{} does not require objectives to be modeled independently, and supports multi-task covariance functions across correlated objectives.
\vspace{-2ex}
\subsection{Parallel Noisy Expected Hypervolume Improvement}
\vspace{-1ex}
Generating and evaluating batches of  candidates is imperative to achieving adequate throughput in many real-world scenarios. $\qNEHVI{}$ can naturally be extended to the parallel (asynchronous or batch) setting by evaluating \HVI{} with respect to a batch of $q$ points $\xcand = \{\bm x_i\}_{i=1}^q$
\begingroup
\setlength{\thickmuskip}{3mu}
\begin{equation}
\label{eqn:mc_qnehvi}
\aqNEHVI{}(\xcand)= \int \aqEHVI{}(\xcand| \mathcal {P}_n) p(\bm f | \mathcal D_n) d\bm{f} \approx \hataqNEHVI{}(\xcand) = \frac{1}{N}\sum_{t=1}^{N} \HVI(\tilde{\bm f}_t( \xcand) | \mathcal P_t)
\end{equation}
\endgroup

Since optimizing $q$ candidates jointly is a difficult numerical optimization problem over a $qd$-dimensional domain, we use a sequential greedy approximation in the parallel setting and solve a sequence of $q$ simpler optimization problems with $d$ dimensions, which been shown empirically to improve optimization performance \citep{wilson2018maxbo}.
While selecting candidates according to a ``sequential greedy'' policy does not guarantee that the selected batch of candidates is a maximizer of the $\aqNEHVI{}$, the submodularity of $\aqNEHVI{}$ allows us to \emph{bound the regret} of this approximation to be no more than $\frac{1}{e}\alpha_\text{\qNEHVI}^*$, where $\alpha_\text{\qNEHVI}^* =  \max_{\xcand \in \mathcal X}\alpha_\text{\qNEHVI}(\xcand)$
(see Appendix \ref{sec:seq_greedy_details}).

\vspace{-1ex}
\section{Efficient Evaluation with Cached Box Decompositions}
\vspace{-1ex}

Although $\hataNEHVI{}(\bm x)$ in \eqref{eqn:mc_nehvi} has a concise mathematical form, computing it requires determining the Pareto frontier $\mathcal P_t$ under each sample $\tilde{\bm f}_t$ for $t=1, ..., N$ and then partitioning the region that is not dominated by $\mathcal P_t$ into disjoint hyperrectangles $\{S_{k_t}\}_{k_t=1}^{K_t}$. Optimizing the unbiased MC estimator of $\aNEHVI$ would require re-sampling $\{\tilde{\bm f}_t\}_{t=1}^N$ at each evaluation of $\aNEHVI$. However, computing the Pareto frontier and performing a box decomposition under each of the $N$ samples  during every evaluation of $\aNEHVI$ in the inner optimization loop ($\bm x^* = \argmax_{\bm x} \aNEHVI(\bm x| \mathcal D_n)$) would be prohibitively expensive. This is because box decomposition algorithms have super-polynomial time complexity in the number of objectives \citep{yang_emmerich2019}. We instead propose an efficient alternative computational technique for repeated evaluations of EHVI with uncertain Pareto frontiers.

\textbf{Cached Box Decompositions:}\label{subsec:pareto_caching} For repeated evaluations of the integral in \eqref{eqn:ideal_nehvi}, 
we use a set of fixed samples $\{\tilde{\bm f}_t(X_n)\}_{t=1}^N$, which allows us to compute the Pareto frontiers and box decompositions once, and cache them for the entirety of the acquisition function optimization, thereby making those two computationally intensive operations a one-time cost per BO iteration.\footnote{For greater efficiency, we may also prune $X_n$ to remove points that are dominated with high probability, which we estimate via MC.} We refer to this approach as using \emph{cached box decompositions} (\cbd{}). The method of optimizing over fixed random samples is known as sample average approximation (SAA) \citep{balandat2020botorch}.

\textbf{Conditional Posterior Sampling:}\label{subseq:condpostsampling}
Under the \cbd{} formulation, computing $\hataNEHVI{}(\bm x)$  with joint samples from $\tilde{\bm f}_t(X_n, \bm x) \sim p(\bm f(X_n, \bm x) | \mathcal D_n)$ requires sampling from the conditional distributions
\begin{equation}
\label{eqn:conditional_distribution}
\tilde{\bm f}_t(\bm x) \sim p\big(\bm f(\bm x) | \bm f(X_n) = \tilde{\bm f}_t(X_n), \mathcal D_n\big),
\end{equation}
where $t=1,...,N$ and $\{\tilde{\bm f}_t(X_n)\}_{t=1}^N$ are the realized samples at the previously evaluated points. For multivariate Gaussian posteriors (as is the case with GP surrogates), we can sample from $p(\bm f(X_n) | \mathcal D_n)$ via the reparameterization trick \citep{kingma2013reparam} by evaluating
$\tilde{\bm f}_t(\bm x) = \bm\mu_n + L_n^T\bm\zeta_{n,t},$
where $\bm\zeta_{n,t} \sim \mathcal N(\bm 0, I_{nM})$,  $\bm\mu_n\in\mathbb R^{nM}$ is the posterior mean, and $L_n \in \mathbb R^{nM \times nM}$ is a lower triangular
root decomposition of the posterior covariance matrix, typically a Cholesky decomposition. 
Given $L_n$, we can obtain a root decomposition $L_n'$ of the covariance matrix of the joint posterior $p(\bm f(X_n, \bm x) | \mathcal D_n)$ by performing efficient low-rank updates \citep{Osborne2010BayesianGP}. 
Given $L_n'$ and the posterior mean of $p(\bm f(X_n, \bm x) | \mathcal D_n)$, we can sample from \eqref{eqn:conditional_distribution} via the reparameterization trick by augmenting the existing base samples $\bm\zeta_{n,t}$ with $M$ new base samples for the new point.

\vspace{-2ex}
\subsection{Efficient Sequential Greedy Batch Selection using CBD}
\vspace{-1ex}
The \cbd{} technique addresses the general problem of inefficient repeated evaluations of \EHVI{} with uncertain Pareto frontiers. In this section, we show that sequential greedy batch selection (with both \qEHVI{} and \qNEHVI{}) is an incarnation of \EHVI{} with uncertain Pareto frontiers.

The original formulation of parallel EHVI in \citet{daulton2020ehvi}  uses the inclusion-exclusion principle (\iep{}), which involves computing the volume jointly dominated by each of the $2^q - 1$ nonempty subsets of points in $\xcand$. However, using large batch sizes is \emph{not computationally feasible} under this formulation because time and space complexity are exponential in $q$ and multiplicative in the number of hyperrectangles in the box decomposition \citep{daulton2020ehvi} (see Appendix~\ref{appdx:sec:complexity} for a complexity analysis).
Although \qEHVI{} is optimized using sequential greedy batch selection, the \iep{} is used over all candidates $\bm x_1, ..., \bm x_i$ when selecting candidate $i$. 
Although the \iep{} could similarly be used to compute \qNEHVI{}, we instead leverage \cbd{}, which yields a sequential greedy approximation of the joint (noisy) \EHVI{} that is \emph{mathematically equivalent} to the \iep{} formulation, but \emph{significantly reduces computational overhead}. That is, the \iep{} and \cbd{} approaches produce exactly the same acquisition value for a given set of points $\xcand$, but the \iep{} and the \cbd{} approaches have \emph{exponential} and \emph{polynomial} time complexities in $q$, respectively.

When selecting $\bm x_i$ for $i \in \{2,\dotsc, q\}$, all  $\bm x_j$ for which $j < i$ have already been selected and are therefore held constant. Thus, we can decompose \qNEHVI{} into the \qNEHVI{} from the previously selected candidates $\bm x_1, \dotsc, \bm x_{i-1}$ and \NEHVI{} from $\bm x_i$ given the previously selected candidates

\vspace{-1ex}
\begingroup
\setlength{\thickmuskip}{3mu}
\begin{equation}
    \label{eqn:cbd_qnehvi}
    \hat{\alpha}_{q\textsc{NEHVI}}(\{\bm x_j\}_{j=1}^i) =\frac{1}{N}\sum_{t=1}^N \HVI\big(\{\tilde{\bm f}_t(\bm x_j)\}_{j=1}^{i-1}\}\mid \mathcal P_t\big)
    +\frac{1}{N}\sum_{t=1}^N\HVI\bigl(\tilde{\bm f}_t(\bm x_i) \mid \mathcal P_t \cup \{\tilde{\bm f}_t(\bm x_j)\}_{j=1}^{i-1}\}\bigr)
\end{equation}
\endgroup
Note that the first term on the right hand side is constant, since $\{\bm x_j\}_{j=1}^{i-1}$ and $\{\tilde{\bm f}_t(\bm x_j)\}_{j=1}^{i-1}$ are fixed for all $t=1, ..., N$. The second term is $\hataNEHVI (\bm x_i)$, where the \NEHVI{} is taken with respect to the Pareto frontier across $\bm f(X_n, \bm x_1, ..., \bm x_{i-1})$ and computed using the fixed samples $\{\tilde{\bm f}_t(X_n, \bm x_1, ... \bm x_{i-1})\}_{t=1}^N$. To compute the second term when selecting candidate $\bm x_i$, the $N$ Pareto frontiers and \cbd{}s are updated to include $\{\tilde{\bm f}_t(X_n, \bm x_1, ... \bm x_{i-1})\}_{t=1}^N$. As in the sequential $q=1$ setting, the box decompositions are only computed and cached while selecting each candidate point. See Appendix \ref{appdx:sec:qnehvi_cbd} for a derivation of \eqref{eqn:cbd_qnehvi}. Although we have focused on \qNEHVI{} in the above, the \cbd{} formulation for \qEHVI{} is obtained by simply replacing $\mathcal P_t$ with the Pareto frontier over the observed values $\mathcal Y_n$.

Despite computing $N$ box decompositions when selecting each candidate $\bm x_i$ for $i=2, ..., q$, the \cbd{} approach reduces the time and space complexity from exponential (under the \iep{}) to polynomial in~$q$ (see Appendix~\ref{appdx:sec:complexity} for details on time and space complexity). Figure~\ref{fig:q_cache_iep} shows the total acquisition optimization time (including box decompositions) for various batch sizes and demonstrates that using \cbd{} allows to scale to batch sizes that are \emph{completely infeasible} when using \iep{}.

\begin{figure}
\floatbox[{\capbeside\thisfloatsetup{capbesideposition={right,center},capbesidewidth=0.65\textwidth}}]{figure}[\FBwidth]
{\caption{\label{fig:q_cache_iep} Acquisition optimization wall time under a sequential greedy approximation using L-BFGS-B. \cbd{} enables scaling to much larger batch sizes $q$ than using the  \iep{} and avoids running out-of-memory (OOM) on a GPU. Independent GPs are used for each outcome. The Pareto frontier of of the 2-objective, 6-dimensional DTLZ2 problem \citep{dtlz} is initialized with 20 points.  Wall times were measured on a Tesla V100 SXM2 GPU (16GB RAM) and a 2x Intel Xeon 6138 CPU @ 2GHz (251GB RAM). See Appendix \ref{appdx:sec:cbd_vs_iep} for results with more objectives.}}
{\includegraphics[width=0.35\textwidth]{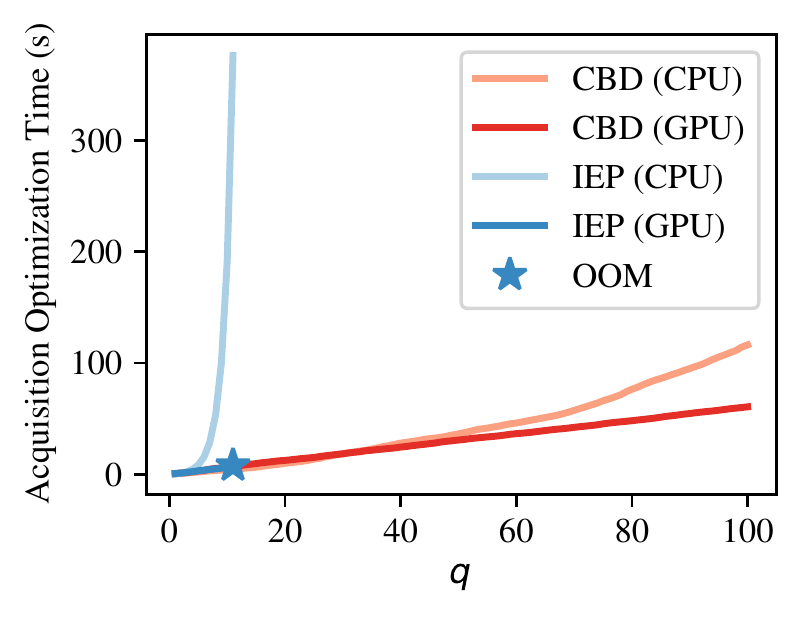}}
\vspace{-5ex}
\end{figure}

\vspace{-1ex}
\section{Optimizing \NEHVI{}}
\vspace{-1ex}
\label{sec:optimizing_nehvi}
\textbf{Differentiability}: Importantly, $\hataNEHVI{}(\bm x)$ is differentiable w.r.t. $\bm x$. Although determining the Pareto frontier and computing the box decompositions are non-differentiable operations, 
these operations do not involve $\bm x$, even when re-sampling from the joint posterior $p(\bm f(X_n, \bm x) | \mathcal D_n)$. Exact sample-path gradients of $\nabla_{\bm x} \hataNEHVI{}(\bm x)$ can easily be computed using auto-differentiation in modern computational frameworks. 
This enables efficient gradient-based optimization of \qNEHVI{}.\footnote{One can also show that the gradient of the full MC estimator $\hataqNEHVI{}$ is an unbiased estimator of the gradient of the true joint noisy expected hypervolume improvement $\aqNEHVI{}$. However, this result is not necessary for our SAA approach.}

\textbf{SAA Convergence Results:} In addition to approximating the outer expectation over $\bm f(X_n)$ with fixed posterior samples, 
we can similarly  
fix the base samples used for the new candidate point $\bm x$. This approach yields a deterministic acquisition function, which enables using (quasi-) higher-order optimization methods to obtain fast convergence rates for acquisition optimization \citep{balandat2020botorch}. Importantly, we prove that the theoretical convergence guarantees on acquisition optimization under the 
SAA approach proposed by \citet{balandat2020botorch} also hold for \NEHVI{}.

\begin{theorem}
\label{thm:saa}
Suppose $\mathcal X$ is compact and $\bm f$ has a multi-output GP prior with continuously differentiable mean and covariance functions. Let $X_n= \{\bm x_i\}_{i=1}^n$ denote the previously evaluated points and $\{\bm \zeta\}_{t=1}^N$ be base samples $\bm \zeta \sim \mathcal N(\bm 0, I_{(n+1)M})$. Let $\hat{\alpha}_{\textsc{NEHVI}}$ denote the deterministic acquisition function  computed using $\{\bm \zeta\}_{t=1}^N$ as $\hat{\alpha}_{\textsc{NEHVI}}^N$ and define $S^* := \argmax_{\bm x \in \mathcal X} \aNEHVI{}(\bm x)$ to be the set of maximizers of $\aNEHVI{}(\bm x)$ over $\mathcal X$. Suppose $\hat{\bm x}_N^* \in \argmax_{\bm x \in \mathcal X} \hat{\alpha}_{\textsc{NEHVI}}^N(\bm x)$. Then (1) $\hat{\alpha}_{\textsc{NEHVI}}^N(\hat{\bm x}_N^*) \rightarrow \aNEHVI{}(\bm x_N^*)$ almost surely, and (2) $\emph{dist}(\hat{\bm x}_N^*, S^*) \rightarrow 0$, 
where $\text{dist}(\hat{\bm x}_N^*, \mathcal S^*) := \inf_{\bm x \in S^*} ||\hat{\bm x}_N^* - \bm x||$ is the Euclidean distance between $\hat{\bm x}_N^*$ and the set $S^*$.
\end{theorem}

Theorem \ref{thm:saa} also holds in the parallel setting, so \qNEHVI{} enjoys the same convergence guarantees as \NEHVI{} on acquisition optimization under the SAA. See Appendix \ref{appdx:sec:theoretical_results} for further details and proof.

\vspace{-1ex}
\section{Approximation of \qNEHVI{} using Approximate GP Sample Paths}
\vspace{-1ex}
Although \cbd{} yields polynomial complexity of \qNEHVI{} with respect to $q$ (rather than exponential complexity with the \iep{}), it still requires computing $N$ box decompositions and repeatedly evaluating the joint posterior over $\bm f(X_n, \{\bm x_j\}_{j=1}^{i-1})$ for selecting each candidate $\bm x_i$ for $i=1, ..., q$. A cheaper alternative is to approximate the integral in \eqref{eqn:mc_qnehvi} using a single approximate GP sample path $\tilde{\bm f_i}$ using RFFs when optimizing candidate $\bm x_i$. A single-sample approximation of $q$NEHVI, which we refer to as \TSHVI{}, can be computed by using $\tilde{\bm f_i}$ as the sampled GP in \eqref{eqn:cbd_qnehvi}. Since the RFF is a deterministic model, it is much less computationally expensive to evaluate than the GP posterior on out-of-sample points, and exact gradients of \TSHVI{} with respect to current candidate $\bm x_i$ can be computed and used for efficient multi-start optimization of \TSHVI{} using second-order gradient methods. \TSHVI{} requires CBD for efficient sequential greedy batch selection and gradient-based optimization, but does not use a sample average approximation for optimizing a new candidate $\bm x_i$; instead, it uses an approximate \emph{sample path}. See \citet{rahimi_rff} for details on RFFs.

\TSHVI{} is related to TSEMO in that both use sequential greedy batch selection using \HVI{} based on RFF samples. However, TSEMO does not directly maximize \HVI{} when selecting candidate $\bm x_i$, where $i=1, ..., q$; rather, it relies on a heuristic approach of running NSGA-II on an RFF sample of each objective to create a discrete population of candidates and then selecting the point from the discrete population that maximizes \HVI{} under the RFF sample. In contrast, \TSHVI{} \emph{directly optimizes} \HVI{} under the RFF using exact sample-path gradients, which leads to improved optimization performance (see Appendix \ref{appdx:sec:extra_experiments}). Furthermore, we find that \TSHVI{} is significantly faster than TSEMO, because rather than using NSGA-II it uses second order gradient methods to optimize \HVI{} (see Appendix~\ref{appdx:sec:extra_experiments}). Gradient-based optimization is only possible because \cbd{} enables scalable, differentiable \HVI{} computation.
While the primary goal of this work is to develop a principled, scalable method for parallel EHVI in noisy environments, we include empirical comparisons with \TSHVI{} throughout the appendix to demonstrate the generalizablility of the \cbd{} approach and practical performance of the \TSHVI{} approximation. \TSHVI{} achieves the fastest batch selection timesof any method tested on a GPU on every problem; in many cases, this is an order of magnitude speed-up over \qNEHVI{}. Moreover, \TSHVI{} has a remarkable ability to scale to large batch sizes when the dimensionality of optimization problem is modest. Further investigation of \TSHVI{} is needed, but we hope that the readers can recognize the ways in which \qNEHVI{} can create broader opportunities for research into hypervolume improvement based acquisition functions.

\vspace{-1ex}
\section{Experiments}
\vspace{-1ex}
\label{sec:Experiments}
We empirically evaluate \qNEHVI{} on a set of synthetic and real-world benchmark problems. We compare it against the following recently proposed methods from the literature: PESMO, MESMO (which we extend to the handle noisy observations using the noisy information gain from \citet{takeno20}), PFES, DGEMO, MOEA/D-EGO, TSEMO, TS-TCH, \qEHVI{} (and $q$\textsc{EHVI-PM-CBD}, which uses the posterior mean as a plug-in estimate for the function values at the in-sample points, along with CBD to scale to large batch sizes), and qNParEGO. We optimize all methods using multi-start L-BFGS-B with exact gradients (except for PFES, which uses gradients approximated via finite differences), including TS-TCH where we optimize approximate function samples using RFFs with 500 basis functions. We model each outcome with an independent GP with a Mat\'{e}rn 5/2 ARD kernel and  infer the GP hyperparameters via maximum a posteriori (MAP) estimation. For all problems, we assume that the noise variances are observed (except ABR, where we infer the noise level). See Appendix \ref{appdx:sec:experiment_details} for more details on the experiments and acquisition function implementations.

\begin{figure*}[ht!]
    \centering
    \includegraphics[width=\textwidth]{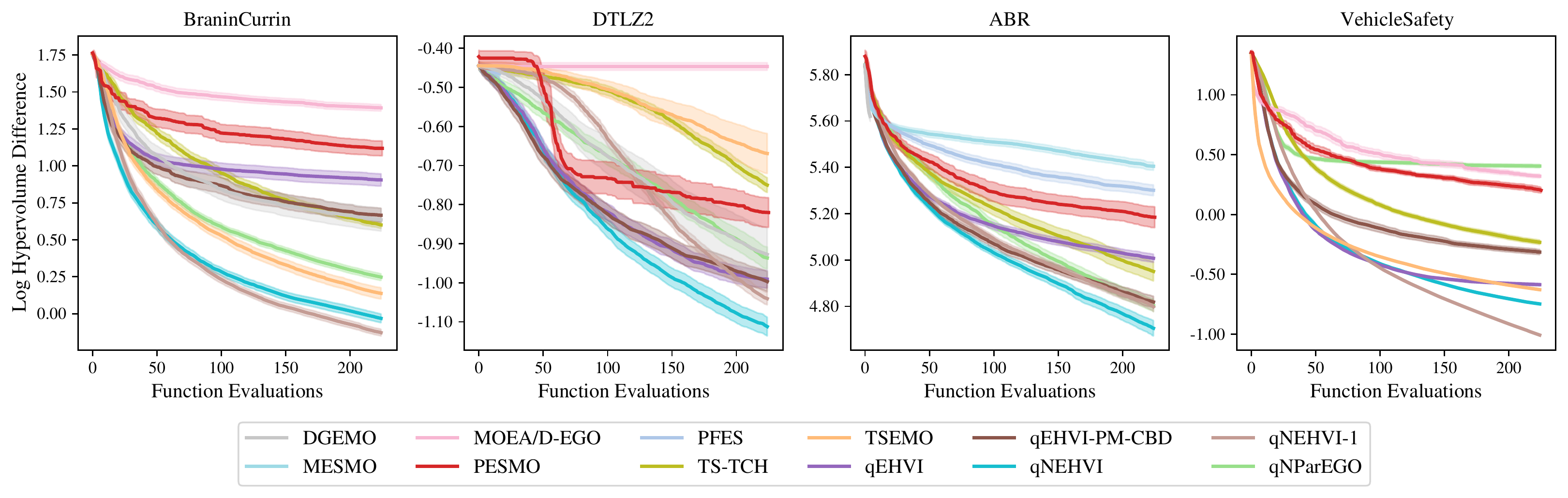}
    \vspace{-5pt}
    \caption{\label{fig:hv_sequential} Sequential optimization performance. The shaded region indicates two standard errors of the mean over 100 replications (only 20 replications were feasible for PESMO due to large runtimes).}
    \vspace{-2ex}
\end{figure*}

We evaluate all methods using the logarithm of the difference in hypervolume between the true Pareto frontier and the approximate Pareto frontier recovered by the algorithm. Since evaluations are noisy, we compute the hypervolume dominated by the noiseless Pareto frontier across the observed points for each method.

\textbf{Synthetic Problems:} We consider a noisy variants of the \emph{BraninCurrin} problem ($M=2, d=2$) and the \emph{DTLZ2} problem ($M=2, d=6$) \citep{dtlz}, in which  observations are corrupted with zero-mean additive Gaussian noise with standard deviation of 5\% of the range of each objective for \emph{BraninCurrin} and 10\% for \emph{DTLZ2}.

\textbf{Adaptive Bitrate (ABR) Control Policy Optimization:}  
ABR controllers are used for real-time communication and media streaming applications. Policies for these controllers must be tuned to deliver a high quality of experience with respect  to multiple objectives \citep{Mao2019RealworldVA}. In industry settings, A/B tests with dozens of policies are tested simultaneously since each policy may take days or weeks to evaluate, producing noisy measurements across multiple objectives. In this experiment, we tune policies to maximize video quality (bitrate) and minimize stall time. The policy has $d=4$ parameters, which are detailed in Appendix \ref{appdx:sec:experiment_details}. We use the Park simulator
\citep{Mao2019ParkAO}
and sample a random set of 100 traces to obtain noisy measurements of the objectives under a given policy. For comparing the performance of different methods, we estimate the true noiseless objective using mean objectives across 300 traces. We infer a homoskedastic noise level jointly with the GP hyperparameters via MAP estimation.

\textbf{Vehicle Design Optimization:} Optimizing the design of the frame an automobile is important to maximizing passenger safety, vehicle durability and fuel efficiency. Evaluating a vehicle design is time-consuming, since either a vehicle must manufactured and crashed, or a nonlinear finite element-based crash analysis must be run to simulate a collision (which can take over 20 hours per run) \citep{youn2004}.  Hence, evaluating many designs in parallel is critical for reducing end-to-end optimization time. Observations are often noisy due to manufacturing imperfections, measurement error, 
or non-deterministic simulations. In this experiment, we tune the $d=5$ widths of various components of a vehicle's frame to minimize proxy metrics for (1) fuel consumption, (2) passenger trauma in a full frontal collison, and (3) vehicle fragility \citep{tanabe2020}. See Appendix~\ref{appdx:sec:experiment_details} for details. For this demonstration, we add zero-mean Gaussian noise with a standard deviation of 1\% of the objective range, which roughly corresponds to the manufacturing noise level used in previous work \citep{youn2004}.

\vspace{-2ex}
\subsection{Summary of Results:} 
\vspace{-1ex}
We find that \qNEHVI{} and \TSHVI{}
outperform all other methods on the noisy benchmarks, both in the sequential and parallel setting.  In the sequential setting (Fig \ref{fig:hv_sequential}), \qNEHVI{} and \TSHVI{} are followed closely by \qEHVI{}-PM, and in some cases, even \qEHVI{}. TS-TCH is firmly in the middle of the pack, while information-theoretic acquisition functions appear to perform the worst. This is consistent across noise levels; for experiments where we add noise to the objectives, we consider noise levels ranging from 1\% to 10\% of the range of each objective (these are magnitudes of the noise often seen in practice). Previous works have only evaluated MOBO algorithms with noise levels of 1\% \citep{pesmo}. In Appendix~\ref{appdx:sec:extra_experiments}, we perform a study showing that \qNEHVI{} consistently performs best with increasing noise levels up to 30\% of the range of each objective.

While parallel evaluation can provide optimization speedups on order of the batch size $q$, these evaluations do affect the overall sample complexity of the algorithm, since less information is available within the synchronous batch setting compared with fully sequential optimization. We find that, by and large, \qNEHVI{} achieves the greatest hyper-volume for increasingly large batch sizes, and scales more elegantly relative to TS-TCH and the ParEGO variants (Fig \ref{fig:q_hv}). \qNEHVI{} also consistently outperforms \qEHVI{}-PM-CBD. In Appendix~\ref{appdx:sec:extra_experiments}, we observe that \qNEHVI{} and \TSHVI{} provides excellent anytime performance all values of $q$ that we tested. We provide results on 4 additional test problems in Appendix~\ref{appdx:sec:additional_results}, and in Appendix~\ref{appdx:5obj}, we demonstrate that leveraging \cbd{} and a single sample path approximation, \TSHVI{} enables scaling to 5-objective problems, which is a first for an \HVI-based method, to our knowledge.

In our experiments, we find that \TSHVI{} is among the top performers on relatively low-dimensional problems. Given the strong performance of \TSHVI{}, we examine its performance as the dimensionality of the search space increases in Appendix~\ref{sec:high_d_tshvi}. We find that \qNEHVI{} is more robust than \TSHVI{} in higher-dimensional search spaces, but further investigation is needed into how the number of the Fourier basis functions affects the performance of \TSHVI{} in high-dimensional search spaces. 

\begin{figure*}[t]
\centering
\includegraphics[width=\textwidth]{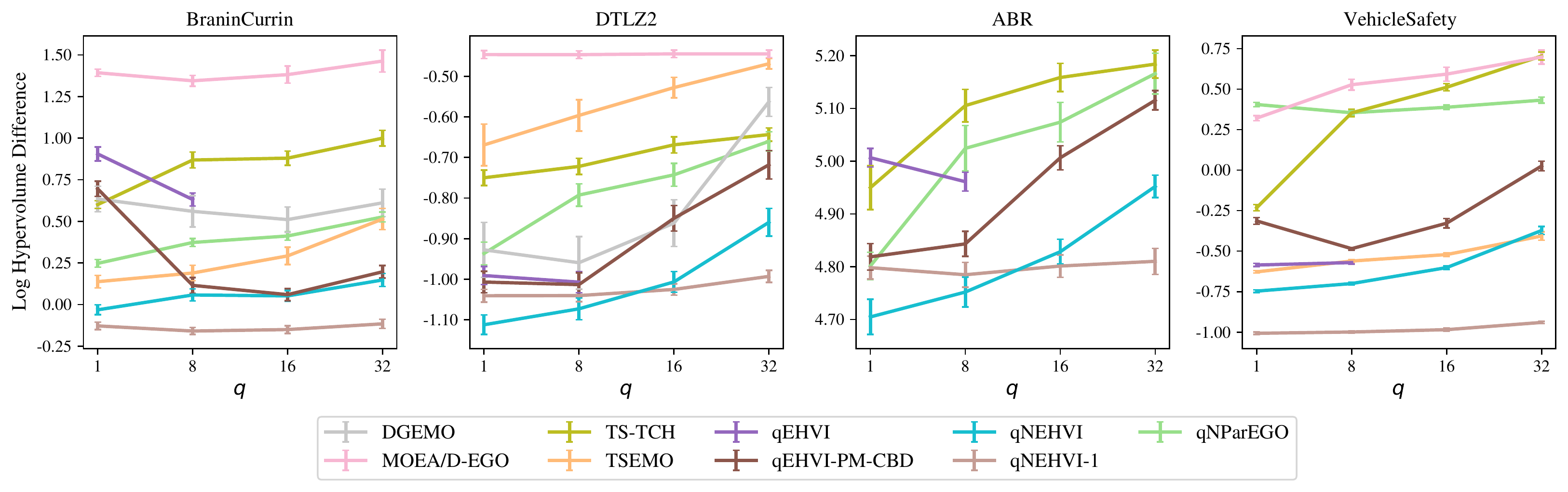}
\caption{\label{fig:q_hv} The quality of the final Pareto frontier identified by each method with increasing batch sizes $q$ given a budget of 224 function evaluations. \qEHVI{} is only included for $q=1$ and $q=8$ because the \iep{} scales exponential with $q$. DGEMO is omitted on the ABR problem because it was prohibitively slow with time-consuming ABR simulations and on the VehicleSafety problem because DGEMO consistently crashed in the graph cutting algorithm.}
\end{figure*}

\textbf{Optimization wall time:} 
Across all experiments, we observe competitive wall times for optimizing \qNEHVI{} and \TSHVI{} (all wall time comparisons are provided in Appendix \ref{appdx:sec:extra_experiments}). On a GPU, optimizing \TSHVI{} \emph{incurs the lowest wall time of any method that we tested on every single problem} and optimizing \qNEHVI{} is faster than optimizing information-theoretic methods on all problems. Using efficient low-rank Cholesky updates, \qNEHVI{} is often faster than the $q$NParEGO implementation in BoTorch on a GPU.
\vspace{-1ex}
\section{Discussion}
\vspace{-1ex}
\label{sec:Discussion}

We proposed \NEHVI{}, a novel acquisition function that provides a principled approach to parallel and noisy multi-objective Bayesian optimization. \NEHVI{} is a one-step Bayes-optimal policy for maximizing the hypervolume dominated by the Pareto frontier in noisy and noise-free settings. \NEHVI{} is made feasible by a new approach to computing joint hypervolumes (\cbd{}), and we demonstrated that \cbd{} enables scalable, parallel candidate generation with both noiseless \qEHVI{} and \qNEHVI{}. We provide theoretical results on optimizing a MC estimator of \qNEHVI{} using sample average approximation and demonstrate significant improvements in optimization performance over state-of-the-art MOBO algorithms.

Yet, our work has some limitations.  While the information-theoretic acquisition functions tested here perform poorly on our benchmarks, they do allow for decoupled evaluations of different objectives in cases where querying one objective may be more resource-intensive than querying other objectives. Optimizing such acquisition functions is a non-trivial task, and it is possible that with improved procedures, such acquisition functions could yield improved performance and provide a principled approach to selecting evaluation sources on a budget. 
Although practically fast enough for most Bayesian optimization tasks, exact hypervolume computation has super-polynomial complexity in the number of objectives. Combining \qNEHVI{} with differentiable approximate methods for computing hypervolume (e.g. \citet{couckuyt12, golovin2020random}) could lead to further speed-ups. 

We hope that the core ideas presented in this work, including the \cbd{} approach, can provide a framework to support the development of new computationally efficient MOBO methods.

\FloatBarrier





\FloatBarrier
\clearpage
\typeout{}
\bibliography{neurips_2021}
\bibliographystyle{icml2021}

\clearpage
\newpage
\FloatBarrier
\onecolumn
\appendix
\begin{center}
\hrule height 4pt
\vskip 0.25in
\vskip -\parskip
    {\LARGE\bf  Appendix to:\\[2ex] \papertitle}
\vskip 0.29in
\vskip -\parskip
\hrule height 1pt
\vskip 0.2in%
\end{center}

\section{Potential Societal Impact}
\label{app:societal_impact}

Bayesian Optimization specifically aims to increase sample efficiency for hard optimization algorithms, and consequently can help achieve better solutions without incurring large societal costs. For instance, as demonstrated in this work, automotive design problems may be solved much faster, reducing the amount of computationally costly simulations and thus the energy footprint during development. At the same time, improved solutions mean that high crash safety can be achieved with lighter cars, resulting in fewer resources required for their production and, importantly, improving fuel economy of the whole vehicle fleet. Increased robustness to noisy observations further helps reduce the resources spent on evaluating regions of the search space that are not promising.
Improvements to the optimization performance and practicality of multi-objective Bayesian optimization have the potential to allow decision makers to better understand and make more informed decisions across multiple trade-offs. We expect these directions to be particularly important as Bayesian optimization is increasingly used for applications such as recommender systems~\cite{letham2019mtbo}, where auxiliary goals such as fairness must be accounted for. Of course, at the end of the day, exactly what objectives decision makers choose to optimize, and how they balance those trade-offs (and whether that is done in equitable fashion) is up to the individuals themselves.

\section{Computing Hypervolume Improvement with Box Decompositions}
\label{appdx:sec:box_decompositions}
\begin{definition}\label{def:delta} For a set of objective vectors $\{\bm f(\bm x_i)\}_{i=1}^q$, a reference point $\bm r \in \mathbb R^M$, and a Pareto frontier $\mathcal P$, let $\Delta(\{\bm f(\bm x_i)\}_{i=1}^q, \mathcal P, \bm r) \subset \mathbb{R}^M$ denote the set of points (1) that are dominated by $\{\bm f(\bm x_i)\}_{i=1}^q$, (2) that dominate $\bm r$, and (3) that are not dominated by $\mathcal P$.
\end{definition}
Let $\{S_1, ..., S_K\}$ be a set of$K$ disjoint axis-aligned rectangles where each $S_k$ is defined by a pair of lower and upper vertices $\bm l_k \in \mathbb{R}^M$ and $\bm u_k \in \mathbb{R}^M \cup \{\bm\infty\}$.  Figure \ref{fig:hvi_box} shows an example decomposition. Such a partitioning allows for efficient piece-wise computation of the hypervolume improvement from a new point $\bm f(\bm x_i)$ by computing the volume of the intersection of the region dominated exclusively by the new point with $\Delta(\{f(\bm x_i), \mathcal P, \bm r)$ (and not dominated by the $P$) with each hyperrectangle $S_k$. Although $\Delta(\bm f(\bm x_i), \mathcal P, \bm r)$ is a non-rectangular polytope, the intersection of $\Delta(\bm f(\bm x_i), \mathcal P, \bm r)$ with each rectangle $S_k$ is a rectangular polytope and the vertices bounding the hyperrectangle corresponding to $\Delta(\bm f(\bm x_i), \mathcal P, \bm r) \cap S_k$ can be easily computed: the lower bound vertex is $\bm l_k$ and the upper bound vertex is the component-wise minimum of $\bm u_k$ and the new point $\bm f(\bm x)$: $\bm z_k := \min \big[\bm u_k,\bm f(\bm x)\big]$. The hypervolume improvement can be computed by summing over the volume of $\Delta(\bm f(\bm x_i), \mathcal P, \bm r) \cap S_k$ over all $S_k$
\begin{equation}
\label{eqn:HVI}
    \HVI{}\big(\bm{f}(\bm{x}), \mathcal P\big) =\sum_{k=1}^{K}\HVI{}_k\big(\bm{f}(\bm{x}), \bm l_k, \bm u_k\big)=\sum_{k=1}^{K}\prod_{m=1}^M \big[z_k^{(m)} - l_k^{(m)}\big]_{+},
\end{equation}
where $[\cdot]_{+}$ denotes the $\max(\cdot, 0)$ operation.
\begin{figure}[tb]
    \centering
    \includegraphics[width=0.3\textwidth]{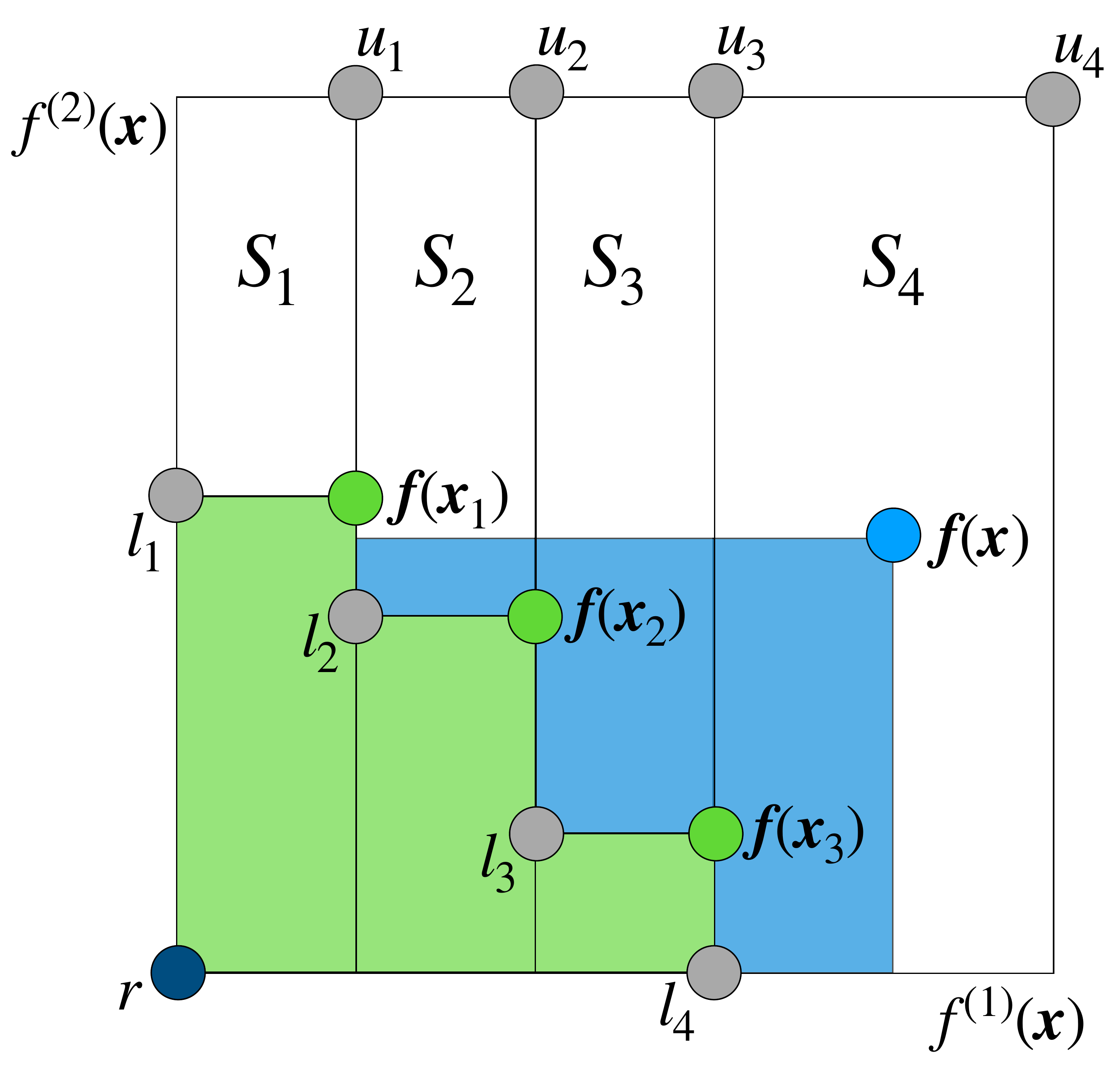}
    \caption{\label{fig:hvi_box} The hypervolume improvement from a new point $\bm f(\bm x)$ is shown in blue. The current Pareto frontier $\mathcal P$ is given by the green points, the green area is the hypervolume of the Pareto frontier $\mathcal P$ given reference point $\bm r$. The white rectangles $S_1, ..., S_k$ are a disjoint, box decomposition of the non-dominated space that can be used to efficiently compute the hypervolume improvement.}
\end{figure}

\section{\qNEHVI{} under Different Computational Approaches}
\subsection{Derivation of \iep{} formulation of \qNEHVI{}}
\label{appdx:sec:iep_formulation}
From \eqref{eqn:mc_qnehvi}, the expected noisy joint hypervolume improvement is given by
\begin{align*}
\hataqNEHVI{}(\xcand) &= \frac{1}{N}\sum_{t=1}^{N} \HVI(\tilde{\bm f}_t( \xcand) | \mathcal P_t)
\end{align*}
Recall that the joint \HVI{} formulation under the \iep{} derived by \citet{daulton2020ehvi} is given by
\begin{align}
    \HVI(\bm f( \xcand) | \mathcal P) =\sum_{k=1}^{K}\sum_{j=1}^q \sum_{X_j \in \mathcal X_j} (-1)^{j+1} \prod_{m=1}^M\big[z_{k, X_j}^{(m)} - l_k^{(m)}\big]_{+}
\end{align}
where $\mathcal X_j := \{X_j \subseteq \xcand : \vert X_j \vert = j\}$ and $z_{k,t, X_j}^{(m)} := \min[u_{k, t}^{(m)}, f^{(m)}(\bm x_{i_1}), ..., f^{(m)}(\bm x_{i_j})]$ for $X_j = \{\bm x_{i_1}, ..., \bm x_{i_j}\}$. In \qNEHVI{}, the lower and upper bounds and the number of rectangles in each box decomposition depend $\mathcal P_t$. Hence,
\begin{align*}
\hataqNEHVI{}(\xcand) &= \frac{1}{N}\sum_{t=1}^{N}\sum_{k=1}^{K_t}\sum_{j=1}^q \sum_{X_j \in \mathcal X_j} (-1)^{j+1} \prod_{m=1}^M\big[z_{k, t, X_j}^{(m)} - l_{k,t}^{(m)}\big]_{+}
\end{align*}
where $z_{k,t, X_j}^{(m)} := \min[u_{k, t}^{(m)}, \tilde{f}_t^{(m)}(\bm x_{i_1}), ..., \tilde{f}_t^{(m)}(\bm x_{i_j})]$ for $X_j = \{\bm x_{i_1}, ..., \bm x_{i_j}\}$.
\subsection{Derivation of \cbd{} formulation of \qNEHVI{}}
\label{appdx:sec:qnehvi_cbd}
Using Definition \ref{def:hvi}, we rewrite \eqref{eqn:mc_qnehvi} as 
\begin{align*}
\hataqNEHVI{}(\xcand) &=\frac{1}{N}\sum_{t=1}^{N} \HVI(\tilde{\bm f}_t( \xcand) | \mathcal P_t)\\
&= \frac{1}{N}\sum_{t=1}^{N}\bigg[ \HV{}(\tilde{\bm f}(\xcand) \cup P_t) - 
\HV{}(P_t)\bigg]
\end{align*}
Adding and subtracting $\HV{}(\tilde{\bm f}(\{\bm x_1), ..., \tilde{\bm f}(x_{q-1})\}) \cup P_t)$ yields 
\begin{align*}
\hataqNEHVI{}(\xcand) = \frac{1}{N}\sum_{t=1}^{N}\bigg[ &\HV{}(\tilde{\bm f}(\xcand) \cup P_t) - \HV{}(\tilde{\bm f}(\{\bm x_1), ..., \tilde{\bm f}(x_{q-1})\}) \cup P_t) \\&+ \HV{}(\{\tilde{\bm f}(\bm x_1), ..., \tilde{\bm f}(x_{q-1})\} \cup P_t)
-\HV{}(P_t)\bigg].
\end{align*}
Applying Definition \ref{def:hvi} again leads to \eqref{eqn:cbd_qnehvi}:
\begin{align*}
\hataqNEHVI{}(\xcand) &= \frac{1}{N}\sum_{t=1}^{N}\HVI{}(\tilde{\bm f}(\bm x_q)|\{\tilde{\bm f}(\bm x_1), ..., \tilde{\bm f}(\bm x_{q-1})\} \cup P_t ) \\ &\qquad + \frac{1}{N}\sum_{t=1}^{N}\HVI{}(\{\tilde{\bm f}(\bm x_1), ..., \tilde{\bm f}(\bm x_{q-1})\})| P_t).
\end{align*}

Note that using the method of common random numbers, the \cbd{} formulation is mathematically equivalent to \iep{} formulation, but the computing \qNEHVI{} with the \cbd{} trick is much more efficient.

\section{Complexity Analysis}
\label{appdx:sec:complexity}
\subsection{Complexity of Computing \qNEHVI{}}
In this section we study the complexity of computing the acquisition function. For brevity, we omit the cost of posterior sampling, which is the same for the \cbd{} and \iep{} approaches.\footnote{Sampling from $p(\bm f(X_n)|\mathcal D_n)$ incurs a one-time cost of $O(Mn^3)$ if each of the $M$ outcomes is modeled by an independent GP, as it involves computing a Cholesky decomposition of the $n \times n$ posterior covariance (at the $n$ observed points) for each. Using low-rank updates of the Cholesky factor to sample from $p(\bm f(X_n, \bm x_0, ..., \bm x_i)| \mathcal D_n)$ has a time complexity of $O(M(n+i-1)^2)$ for $1 \leq i \leq q$ since each triangular solve has quadratic complexity. Sampling is more costly when using a multi-task GP model, as it requires a root decomposition of the $Mn \times Mn$ posterior covariance across data points and tasks.}

The \cbd{} approach requires recomputing box decompositions when generating each new candidate. In the worst case, each new candidate is Pareto optimal under the fixed posterior samples, which leads to a time complexity of 
$O\big(N(n+i)^M\big)$
for computing the box decompositions in iteration $i$ \citep{yang_emmerich2019}. Note that there are 
$O\big((n+i)^M\big)$ 
rectangles in each box decomposition. Given box decompositions and posterior samples at the new point, the complexity of computing the acquisition function on a single-threaded machine is 
$O\big(MN(n+i)^M\big)$. 
Hence, the total time complexity for generating $q$ candidates (ignoring potentially additional time complexity for automated gradient computations) is
\begin{equation}
\label{eqn:cbd_complexity}
O\bigg(N\sum_{i=1}^q(n+i)^M\bigg) + O\bigg(N_\text{opt}MN\sum_{i=1}^q(n+i)^M\bigg) = O\left(N_\text{opt}NM (n+q)^M q\right),
\end{equation}
\begin{equation}
\label{eqn:iep_complexity}
O\big(Nn^M \big)+ O\bigg(N_\text{opt}MNn^M\sum_{i=1}^q 2^{i-1}\bigg) = O\left(N_\text{opt}NMn^M 2^q q\right).
\end{equation}
The second term on the left hand side in both \eqref{eqn:cbd_complexity} and \eqref{eqn:iep_complexity} is the acquisition optimization complexity, which boils down to $O(N_\text{opt})$ given infinite computing cores because the acquisition computation is completely parallelizable. However, as shown in Figure \ref{fig:q_cache_iep}, even for relatively small values of $q$, CPU cores become saturated and GPU memory limits are reached.

Everything else fixed, the asymptotic relative time complexity of using \cbd{} over \iep{} is therefore $q^{-M}2^q \rightarrow \infty$ as $q\rightarrow \infty$.

Similarly, the space complexity under the \cbd{} formulation,  $O\big(MN(n+q)^M\big)$, is also polynomial in $q$, whereas the space complexity is exponential in $q$  under the \iep{} formulation: $O\left(MNn^M q2^q \right)$.

Everything else fixed, the asymptotic relative complexity (both in terms of time and space) of using \cbd{} over \iep{} is therefore $q^{M}2^{-q} \rightarrow 0$ as $q\rightarrow \infty$.

\subsection{Efficient Batched Computation}
\label{appdx:sec:eff_computation}
As noted above, using either the \iep{} or \cbd{} approach, the acquisition computation given the box decompositions is highly parallelizable. However, since the number of hyperrectangles $K_t$ in the box decompoosition can be different under each posterior sample $\tilde{\bm f}_t$, stacking the box decompositions does not result in a rectangular matrix; the matrix is ragged. In order to leverage modern batched tensor computational paradigms, we pad the box decompositions with empty hyperrectangles (e.g. $\bm l = \bm 0, \bm u = \bm 0$) such that the box decomposition under every posterior sample contains exactly $K = \max_t K_t$ hyperrectangles, which allows us to define a $t \times K$ dimensional matrix of box decompositions for use in batched tensor computation.

In the 2-objective case, instead of padding the box decomposition, the Pareto frontier under each posterior sample can be padded instead by repeating a point on the Pareto Frontier such that the padded Pareto frontier under every posterior sample has exactly $\max_t |\mathcal P_t|$ points. This enables computing the box decompositions analytically for all posterior samples in parallel using efficient batched computation. The resulting box decompositions all have $K = \max_t |\mathcal P_t| + 1$ hyperrectangles (some of which may be empty).

\section{Theoretical Results}
\label{appdx:sec:theoretical_results}
Let $\xprev \in \mathbb R^{nd}$ denote the stacked set of previously evaluated points in $X_n$: $\xprev := [\bm x_1^T, ..., \bm x_n^T]^T$. Similarly, let $\xcandvec \in \mathbb R^{qd}$ denote the stacked set of candidates in $\xcand$: $\xcandvec := [\bm x_{n+1}^T, ..., \bm x_{n+q}^T]^T$. Let
$\tilde{\bm f}_t(\xprev, \xcandvec) := [\tilde{\bm f}_t(\bm x_{1})^T,...,\tilde{\bm f}_t(\bm x_{n+q})^T]^T$ denote the $t^\text{th}$ sample of the corresponding objectives, which we write using the parameterization trick as
\begin{align*}
    \bm f_t(\xprev, \xcandvec) = \mu(\xprev, \xcandvec) + L(\xprev, \xcandvec) \bm \zeta_t,
\end{align*}
where $\mu(\xprev, \xcandvec): \mathbb{R}^{(n+q)d} \rightarrow \mathbb{R}^{(n+q)M}$ is the multi-output GP's posterior mean and $L(\xcandvec, \xprev) \in \mathbb{R}^{(n+q)M \times (n+q)M}$ is a root decomposition (often a Cholesky decomposition) of the multi-output GP's posterior covariance $\Sigma(\xcandvec, \xprev) \in \mathbb{R}^{(n+q)M \times (n+q)M}$, and $\bm \zeta_t \in \mathbb{R}^{(n+q)M}$ with $\bm \zeta_t \sim \mathcal{N}(0, I_{(n+q)M})$.\footnote{Theorem \ref{thm:saa} can be extended to handle non-$iid$ base samples from a family of quasi-Monte Carlo methods as in \citet{balandat2020botorch}.}
\begin{proof}[Proof of Theorem \ref{thm:saa}]
Since the sequential \NEHVI{} is equivalent to the \qNEHVI{} with $q=1$, we prove Theorem \ref{thm:saa} for the general $q>1$ case. Recall from Section \ref{appdx:sec:qnehvi_cbd}, that using the method of common random numbers to fix the base samples, the \iep{} and \cbd{} formulations are equivalent. Therefore, we proceed only with the \iep{} formulation for this proof. 

We closely follow the proof of Theorem 2 in \citet{daulton2020ehvi}. We consider the setting from~\citet[Section D.5]{balandat2020botorch}.
Let $f_t^{(m)}(\bm x_i, \bm \zeta_t)= S_{\{i, m\}}(\mu(\xcandvec, \xprev) + L(\xcandvec, \xprev)\bm\zeta_t)$ denote the posterior distribution over the $m^\text{th}$ outcome at $\bm x_i$ as a random variable, where $S_{\{i, m\}}$ denotes the selection matrix ($\|S_{\{i, m\}}\|_{\infty} \leq 1$ for all $i=1, ..., n+q$ and $m=1, ..., M$), to extract the element corresponding to outcome $m$ for the point $\bm x_i$. The \HVI{} under a single posterior sample is given by 
\begin{align*}
    A(\xcandvec, \bm\zeta_t; \xprev) = \sum_{k=1}^{K_t}\sum_{j=1}^q \sum_{X_j \in \mathcal X_j} (-1)^{j+1} \prod_{m=1}^M\big[z_{k, X_j}^{(m)}(\bm\zeta_t) - l_k^{(m)}\big]_{+}
\end{align*}
where $\mathcal X_j := \{X_j=\{\bm x_{i_1}, ... \bm x_{i_j}\} \subseteq \xcand : \vert X_j \vert = j, n+1 \leq i_1\leq i_j \leq n+q\}$
and $z_{k, X_j}^{(m)}( \bm\zeta_t) = \min \big[u_k^{(m)}, f^{(m)}(\bm x_{i_1},\bm \zeta_t), \ldots, f^{(m)}(\bm x_{i_j}, \bm\zeta_t)\big]$. Note that the box decomposition of the non-dominated space $\{S_1, ..., S_{K_t}\}$ and the number of rectangles in the box decomposition depend on $\bm \zeta_t$. Importantly, the number of hyperrectangles $K_t$ in the decomposition is a finite and bounded by $O(|\mathcal P_t|^{\lfloor \frac{M}{2} \rfloor+1})$ \citep{LACOUR2017347, yang_emmerich2019}, where $|\mathcal P_t| \leq n$.

To satisfy the conditions of  \citep[Theorem 3]{balandat2020botorch}, we need to show that there exists an integrable function $\ell:\mathbb{R}^{q \times M}\mapsto \mathbb{R}$ such that for almost every $\bm \zeta_t$ and all $\xcandvec, \ycandvec \subseteq \mathcal X$,  
\begin{align}
    |A(\xcandvec, \bm \zeta_t; \xprev) - A(\ycandvec, \bm \zeta_t; \xprev)| \leq \ell(\bm \zeta_t) \|\xcandvec - \ycandvec\|.
\end{align}
We note that $\xprev$ is fixed and omit $\xprev$ for brevity, except where necessary. 

Let $$\tilde{a}_{k,m,j,X_j}(\xcandvec, \bm\zeta_t) := \Bigl[\min \big[u_{k,t}^{(m)}, f^{(m)}(\bm x_{i_1}, \bm \zeta_t), \ldots, f^{(m)}(\bm x_{i_j},  \bm \zeta_t)\big] - l_{k,t}^{(m)}\Bigr]_{+}.$$
Because of linearity, it suffices to show that this condition holds for
\begin{align}
    \tilde{A}(\xcandvec, \bm\zeta_t) &:=  \prod_{m=1}^M \tilde{a}_{k,m,j,X_j}(\xcandvec, \bm\zeta_t)
    =\prod_{m=1}^M\Bigl[\min \big[u_{k,t}^{(m)}, f^{(m)}(\bm x_{i_1}, \bm\zeta_t), \ldots, f^{(m)}(\bm x_{i_j}, \bm\zeta_t)\big] - l_{k,t}^{(m)}\Bigr]_{+}
\end{align}
for all $k, j$, and $X_j$.
Note that we can bound $\tilde{a}_{k,m,j,X_j}(\xcandvec, \bm \zeta_t)$ by
\begin{equation}
\label{eqn:saa_bound1}
\begin{aligned}
    \tilde{a}_{k,m,j,X_j}(\xcandvec, \bm \zeta_t)
    &\leq \Bigl| \min \big[u_{k,t}^{(m)}, f^{(m)}(\bm x_{i_1}, \bm \zeta_t), \ldots, f^{(m)}(\bm x_{i_j}, \bm \zeta_t)\big] - l_{k,t}^{(m)}\Bigr| \\
    &\leq | l_{k,t}^{(m)} | + \Bigl| \min \big[ u_{k,t}^{(m)}, f^{(m)}(\bm x_{i_1}, \bm \zeta_t), \ldots, f^{(m)}(\bm x_{i_j}, \bm \zeta_t)\big]  \Bigr|.
\end{aligned}
\end{equation}
Consider the case where $u_{k,t}^{(m)} =\infty$. Then $$\min[u_{k,t}^{(m)}, f(\bm x_{i_1}, \bm \zeta_t)^{(m)}, ..., f^{(m)}(\bm x_{i_j}, \bm \zeta_t)] = \min[f^{(m)}(\bm x_{i_1}, \bm \zeta_t), ..., f^{(m)}(\bm x_{i_j}, \bm \zeta_t)].$$ Now suppose $u_{k,t}^{(m)} < \infty$. Then $$\min[u_{k,t}^{(m)}, f^{(m)}(\bm x_{i_1}, \bm \zeta_t), ... f^{(m)}(\bm x_{i_j}, \bm \zeta_t)] <\bigl| \min[f^{(m)}(\bm x_{i_1}, \bm \zeta_t), ... f^{(m)}(\bm x_{i_j}, \bm \zeta_t)]\bigr| +\bigl| u_{k,t}^{(m)}\bigr|.$$ Let
$$
w_{k,t}^{(m)} =
\begin{cases}
    u_{k,t}^{(m)},& \text{if } u_{k,t}^{(m)} < \infty\\
    0,              & \text{otherwise}.
\end{cases}$$
Note that $l_{k,t}^{(m)}$ is finite and bounded from above and below by $r^{(m)} \leq l_{k,t}^{(m)} < u_{k,t}^{(m)}$ for all $k, t, m$, where $r^{(m)}$ is the $m^\text{th}$ dimension of the reference point.

Hence, we can express the bound in \eqref{eqn:saa_bound1} as
\begin{equation}
\label{eqn:saa_bound2}
\begin{aligned}
    \tilde{a}_{k,m,j,X_j}(\xcandvec, \bm \zeta_t)
    &\leq | l_{k,t}^{(m)} | + | w_{k,t}^{(m)} | + \bigl| \min \big[ f^{(m)}(\bm x_{i_1}, \bm \zeta_t), \ldots, f^{(m)}(\bm x_{i_j}, \bm \zeta_t)\big]  \bigr| \\
    &\leq | l_{k,t}^{(m)} |+ | w_{k,t}^{(m)} | + \sum_{i_1, \dotsc, i_j} \bigl| f^{(m)}(\bm x_{i_j}, \bm \zeta_t) \bigr|.
\end{aligned}
\end{equation}
Note that we can bound $\sum_{i_1, \dotsc, i_j} \bigl| f^{(m)}(\bm x_{i_j}, \bm \zeta_t) \bigr|$ by $$\sum_{i_1, \dotsc, i_j} \bigl| f^{(m)}(\bm x_{i_j}, \bm \zeta_t) \bigr|\leq |X_j| \bigl( \| \mu^{(m)}(\xcandvec, \xprev)\| + \|L^{(m)}(\xcandvec, \xprev)\| \|\bm \zeta_t\| \bigr).$$
Substituting this into \eqref{eqn:saa_bound2} yields
\begin{equation}
\label{eqn:saa_bound3}
\begin{aligned}
    |\tilde{a}_{k,m,j,X_j}(\xcandvec,\bm \zeta_t)| \leq | l_{k,t}^{(m)} | + |w_{k,t}^{(m)} | + |X_j| \bigl( \| \mu^{(m)}(\xcandvec, \xprev)\| + \|L^{(m)}(\xcandvec, \xprev)\| \|\bm \zeta_t\| \bigr)
\end{aligned}
\end{equation}
for all $k, m, j, X_j$. 

Because of our assumptions of that $\mathcal X$ is compact and that the mean and covariance functions are continuously differentiable, $\mu(\xcandvec, \xprev), L(\xcandvec, \xprev), \nabla_{\xcandvec}\mu(\xcandvec, \xprev),$ and $\nabla_{\xcandvec}L(\xcandvec, \xprev)$ are uniformly bounded. Hence, there exist $C_1, C_2 < \infty$ such that
\begin{align*}
    |\tilde{a}_{k,m,j,X_j}(\xcandvec, \bm \zeta_t)| &\leq C_1 + C_2 \|\bm \zeta_t\| 
\end{align*}
for all $k, m, j, X_j$.

Consider the $M=2$ case.
Omitting the indices $k, t, j, X_j$ for brevity, we have
\begin{equation}
\label{appdx:eq:Convergence:SplitProduct}
\begin{aligned}
    \bigl|\tilde{A}(\xcandvec, \bm \zeta_t) -& \tilde{A}(\ycandvec, \bm \zeta_t)\bigr|\\ 
    &= \bigl|\tilde{a}_{1}(\xcandvec, \bm \zeta_t) \tilde{a}_{2}(\xcandvec, \bm \zeta_t) - \tilde{a}_{1}(\ycandvec, \bm \zeta_t) \tilde{a}_{2}(\ycandvec, \bm \zeta_t)\bigr| \\
    &= \bigl|\tilde{a}_{1}(\xcandvec, \bm \zeta_t) \bigl(\tilde{a}_{2}(\xcandvec, \bm \zeta_t) - \tilde{a}_{2}(\ycandvec, \bm \zeta_t)\bigr) + \tilde{a}_{2}(\ycandvec, \bm \zeta_t) \bigl(\tilde{a}_{1}(\xcandvec, \bm \zeta_t) - \tilde{a}_{1}(\ycandvec, \bm \zeta_t)\bigr) \bigr| \\
    &\leq |\tilde{a}_{1}(\xcandvec, \bm \zeta_t)| \bigl|\tilde{a}_{2}(\xcandvec, \bm \zeta_t) - \tilde{a}_{2}(\ycandvec, \bm \zeta_t)\bigr| + |\tilde{a}_{2}(\ycandvec, \bm \zeta_t)| \bigl|\tilde{a}_{1}(\xcandvec, \bm \zeta_t) - \tilde{a}_{1}(\ycandvec, \bm \zeta_t)\bigr|.
\end{aligned}
\end{equation}

Using \eqref{eqn:saa_bound3}, we can bound $|\tilde{a}_{k,m,j,X_j}(\xcandvec, \bm \zeta_t) - \tilde{a}_{kmjX_j}(\ycandvec, \bm \zeta_t)|$ by
\begin{align*}
    |\tilde{a}_{k,t,m,j,X_j}&(\xcandvec, \bm \zeta_t) - \tilde{a}_{k,t,m,j,X_j}(\ycandvec, \bm \zeta_t)| \\
    &\leq \sum_{i_1, \dotsc, i_j} \bigl| S_{\{i_j, m\}}(\mu(\xcandvec, \xprev) + L(\xcandvec, \xprev)\bm \zeta_t) - S_{\{i_j, m\}}(\mu(\ycandvec, \xprev) + L(\ycandvec, \xprev)\bm \zeta_t) \bigr| \\
    &\leq |X_j| \Bigl( \|\mu(\xcandvec, \xprev) - \mu(\ycandvec, \xprev) \| + \|L(\xcandvec, \xprev) - L(\ycandvec, \xprev)\| \|\bm \zeta_t\| \Bigr).
\end{align*}
Since $\mu$ and $L$ have uniformly bounded gradients with respect to $\xcandvec$ and $\ycandvec$, they are Lipschitz. Therefore, there exist $C_3, C_4 < \infty$ such that 
\begin{equation}
    \label{eqn:lipschitz_bound}
    |\tilde{a}_{k,t,m,j,X_j}(\xcandvec, \bm \zeta_t) - \tilde{a}_{k,t,m,j,X_j}(\ycandvec, \bm \zeta_t)| \leq (C_3 + C_4 \|\bm \zeta_t\|) \|\xcandvec - \ycandvec \|
\end{equation}
for all $\xcandvec, \ycandvec, k,t, m, j, X_j$.

Substituting \eqref{eqn:lipschitz_bound} into \eqref{appdx:eq:Convergence:SplitProduct}, we have
\begin{align*}
    \bigl|\tilde{A}(\xcandvec, \bm \zeta_t) - \tilde{A}(\ycandvec, \bm \zeta_t)\bigr| 
    &\leq 2 \Bigl(C_1C_3 + (C_1C_4 + C_2C_3) \|\bm \zeta_t\| + C_2C_4 \|\bm \zeta_t\|^2 \Bigr) \|\xcandvec - \ycandvec \|
\end{align*}

The $M>2$ is very similar to the $M=2$ case in ~\eqref{appdx:eq:Convergence:SplitProduct} albeit with more complex expansions. Similarly, There exist $C < \infty$ such that 
\begin{align*}
    \bigl|\tilde{A}(\xcandvec, \bm 
    \zeta_t) - \tilde{A}(\ycandvec, \bm 
    \zeta_t)\bigr| 
    &\leq C \sum_{m=1}^M \|\bm 
    \zeta_t\|^m \|\xcandvec - \ycandvec \|
\end{align*}
Let us define $\ell(\bm \zeta_t) := C \sum_{m=1}^M \|\bm \zeta_t\|^m$. Note that $\ell(\bm \zeta_t)$ is integrable because  all absolute moments exist for the Gaussian distribution.
Since this satisfies the criteria for Theorem 3 in \citet{balandat2020botorch}, the theorem holds for \qNEHVI{}.
\end{proof}

\subsection{Unbiased Gradient estimates from the MC formulation}
\label{appdx:subsec:theoretical_results:UnbiasedGradient}

As noted in Section~\ref{sec:optimizing_nehvi}, we can show the following (note that this result is not actually required for Theorem~\ref{thm:saa}): 

\begin{proposition}
\label{prop:appdx:theoretical_results:UnbiasedGradient}
    Suppose that the GP mean and covariance function are continuously differentiable. Suppose further that the candidate set $\xcand$ has no duplicates, and that the sample-level gradients $\nabla_{\bm x}\HVI(\tilde{f}_t(\bm x))$ are obtained using the reparameterization trick as in \citet{balandat2020botorch}. Then
    \begin{align}
    \mathbb{E}\bigl[\nabla_{\xcandvec}\hataqNEHVI^N(\xcandvec)\bigr] = \nabla_{\xcandvec}\aqNEHVI(\xcandvec),
    \end{align}
    that is, the averaged sample-level gradient is an unbiased estimate of the gradient of the true acquisition function.
\end{proposition}

The proof of Proposition~\ref{prop:appdx:theoretical_results:UnbiasedGradient} closely follows the proof of Proposition~1 in \citet{daulton2020ehvi}.

\section{Error Bound on Sequential Greedy Approximation for \NEHVI{}}
\label{sec:seq_greedy_details}
If the acquisition function $\mathcal L(\xcand)$ is a normalized (meaning $\mathcal L(\emptyset) = 0$), monotone, submodular (meaning that \emph{the increase} in $\mathcal L(\xcand)$ is non-increasing as elements are added to $\xcand$ set function), then the sequential greedy approximation $\hat{\mathcal L}$ of $\mathcal L$ enjoys regret of no more than $\frac{1}{e}\mathcal L^*$, where $\mathcal L^* = \max_{\xcand \subseteq \mathcal X} L(\xcand)$ is the optima of $\mathcal L$ \citep{Fisher1978}. We have $\aqNEHVI(\xcand) = \mathcal L(\xcand) = \mathbb E_{\mathcal P} \big[\aqEHVI{}(\xcand | \mathcal P)\big]$. For a fixed, known $\mathcal P$, \citet{daulton2020ehvi} showed that  $\aqEHVI{}$ is submodular set function. In $\aqNEHVI{}$, $\mathcal P$ is a stochastic, so $\aqEHVI{}(\xcand | \mathcal P)$ is a stochastic submodular set function. Because the expectation of a stochastic submodular function is submodular \citep{asadpour}, $\aqNEHVI{}$ is also submodular. Hence, the sequential greedy approximation of $\aqNEHVI{}$ enjoys regret of no more than $\frac{1}{e} \aqNEHVI{}^*$. Using the result from \citet{wilson2018maxbo}, the MC-based approximation $\hataqNEHVI{}(\xcand) = \sum_{t=1}^{N}\HVI{}\big[\bm f_t(\xcand)| P_t\big]$
also enjoys the same regret bound because \HVI{} is a normalized submodular set function.\footnote{Submodularity technically requires a finite search space $\mathcal X$, whereas in BO $\mathcal X$ is typically an infinite set. Nevertheless in similar scenarios, submodularity has been extended to infinite sets (e.g. \citet{wilson2018maxbo}).}

\section{Experiment Setup}
\label{appdx:sec:experiment_details}

\subsection{Implementation / Code used in the experiments}

Our implementations of \qNEHVI, MESMO, PFES are available in the supplementary files and will be open-sourced under MIT license upon publication. For PESMO, we use the open-source implementation in Spearmint (\url{https://github.com/HIPS/Spearmint/tree/PESM}), which is licensed by Harvard. For DGEMO, MOEA/D-EGO, and TSEMO we use the open-source implementations available at \url{https://github.com/yunshengtian/DGEMO/tree/master} under the MIT license. For TS-TCH, \qEHVI{}, and qNParEGO we use the open-source implementations in BoTorch, which are available at \url{https://github.com/pytorch/botorch)} under the MIT license.

For the ABR problem, we use the Park simulator, which is available in open-source at \url{https://github.com/park-project/park} under the MIT license.

\subsection{Algorithm Details}
All methods are initialized with $2(d+1)$ points from a scrambled Sobol sequence. All MC acquisition functions uses $N=128$ quasi-MC samples \citep{balandat2020botorch}. All parallel algorithms using sequential greedy optimization for selecting a batch of candidates points and the base samples are redrawn when selecting candidate $\bm x_i, i=1, ...,q$. 

For EHVI-based methods, we leverage the two-step trick proposed by \citep{yang_emmerich2019} to perform efficient box decompositions; (i) we find the set of local lower bounds for the maximization problem using Algorithm 5 from \citet{Klamroth_2015}\footnote{More efficient methods for this step exist (e.g. \citet{DACHERT2017}), but \citet{Klamroth_2015} can easily leverage vectorized operations and we find it to be efficient in our experiments.}, and then (ii) using the local lower bounds as a Pareto frontier for the artificial minimization problem, we compute a box decomposition of the dominated space using Algorithm 1 from \citet{LACOUR2017347}.

\qEHVI{} uses the \iep{} for computing joint EHVI over a set of candidates and computes EHVI with respect to the observed Pareto frontier. $q$EHVI-PM-CBD uses the Pareto frontier over the posterior means at the previously evaluated points, providing some amount of regularization with respect to the observed values. In addition, $q$EHVI-PM-CBD uses \cbd{} rather than the \iep{}, which enables scaling to large batch sizes. \TSHVI{} uses 500 fourier basis functions.

For PFES and MESMO, we use 10 sampled (approximate) functions using RFFs (with 500 basis functions) and optimize each function using 5000 iterations of NSGA-II \citep{deb02nsgaii} with a population size of 50. For PFES, we partition the dominated space under each sampled Pareto frontier using the algorithm proposed \citet{LACOUR2017347}, which is more efficient and yields fewer hyperrectangles than the Quick Hypervolume algorithm used by the PFES authors \citep{pfes}. For $q$NParEGO, we use a similar pruning strategy to that in \qNEHVI{} to only integrate over the function values of in-sample points that have positive probability of being best with respect to the sampled scalarization. We use the off-the-shelf implementation of  $q$NParEGO in BoTorch \citep{balandat2020botorch}, which does not use low-rank Cholesky updates; however, we do note that  $q$\textsc{NParEGO} would likely achieve lower wall times using more efficient linear algebra tricks.

For DGEMO, TSEMO, and MOEA/D-EGO, we use the default settings provided in \url{https://github.com/yunshengtian/DGEMO/tree/master}.

\subsection{Problem Details}
All benchmark problems are treated as maximization problems; the objectives for minimization problems are multiplied by -1 to obtain an equivalent maximization problem.
\paragraph{BraninCurrin ($\bm{M=2, ~d=2}$)}
The BraninCurrin problem involves optimizing two competing functions used in BO benchmarks: Branin and Currin. The goal is minimize both:
\begin{align*}
f^{(1)}(x_1', x_2') &=  (x_2 - \frac{5.1}{4 \pi^ 2} x_1^2 + \frac{5}{\pi} x_1 - r)^2 + 10 (1-\frac{1}{8\pi}) \cos(x_1) + 10\\
f^{(2)}(x_1, x_2) &= \bigg[1 - \exp\bigg(-\frac{1} {(2x_2)}\bigg)\bigg] \frac{2300 x_1^3 + 1900x_1^2 + 2092 x_1 + 60}{100 x_1^3 + 500x_1^2 + 4x_1 + 20}
\end{align*}
where $x_1, x_2 \in [0,1]$, $x_1' = 15x_1 - 5$, and $x_2' = 15x_2$.

\paragraph{DTLZ2 ($\bm{M=2, ~d=6}$)} 
DTLZ2 \citep{dtlz} is a standard problem from the multi-objective optimization literature. The two objectives are
\begin{align*}
    f_1(\bm x) &= (1+ g(\bm x_M))\cos\big(\frac{\pi}{2}x_1\big)\cdots\cos\big(\frac{\pi}{2}x_{M-2}\big) \cos\big(\frac{\pi}{2}x_{M-1}\big)\\
    f_2(\bm x) &= (1+ g(\bm x_M))\cos\big(\frac{\pi}{2}x_1\big)\cdots\cos\big(\frac{\pi}{2}x_{M-2}\big) \sin\big(\frac{\pi}{2}x_{M-1}\big),
\end{align*}
where $g(\bm x) = \sum_{x_i \in \bm x_M} (x_i - 0.5)^2, \bm x \in [0,1]^d,$ and $\bm x_M$ is the $d - M +1$ elements of $\bm x$.
\paragraph{ZDT1 ($\bm{M=2, ~d=4}$)}
ZDT1 is a benchmark problem from the multi-objective optimization literature \citep{zdt}. The goal is minimize the following two objectives
\begin{align*}
    f^{(1)}(\bm x) &= x_1\\
    f^{(2)}(\bm x) &= g(\bm x)\Bigg(1-\sqrt{\frac{f^{(1)}(\bm x)}{g(\bm x)}}\Bigg)
\end{align*}
where $g(\bm x) = 1 + \frac{9}{d-1}\sum_{i=2}^d x_i$ and $\bm x = [x_1, ..., x_d] \in [0,1]^d$.
\paragraph{VehicleSafety ($\bm{M=3, ~d=5}$)}
The 3 objectives are based on a response surface model that is fit to data collected from a simulator and are given by \citep{tanabe2020}:
\begin{align*}
    f_1(\bm x) &=1640.2823 + 2.3573285x_1 + 2.3220035x_2 + 4.5688768x_3 + 7.7213633x_4 + 4.4559504x_5\\
    f_2(\bm x) &= 6.5856+ 1.15x_1 - 1.0427x_2 + 0.9738x_3+ 0.8364x_4 - 0.3695x_1x_4 + 0.0861x_1x_5\\
    &+ 0.3628x_2x_4 +  0.1106x_1^2 - 0.3437x_3^2 + 0.1764x_4^2\\
    f_3(\bm x) &= -0.0551 + 0.0181x_1 + 0.1024x_2 + 0.0421x_3 - 0.0073x_1x_2 + 0.024x_2x_3-0.0118x_2x_4\\ 
    &- 0.0204x_3x_4 - 0.008x_3x_5 - 0.0241x_2^2 + 0.0109x_4^2
\end{align*}
where $\bm x \in [1,3]^5$. We seek to (1) minimize mass (a proxy for fuel efficiency), (2) minimize acceleration (a proxy for passenger trauma) in a full-frontal collision, and (3) minimize the distance that the toe-board intrudes into the cabin (a proxy for vehicle fragility) \citep{tanabe2020}.

\paragraph{AutoML ($\bm{M=2, ~d=8}$)}. This experiment considers optimizing predictive performance and latency of a deep neural networks (DNN).
Practitioners and researchers across many domains use DNNs for recommendation and recognition tasks in low-latency (e.g. on-device) environments \citep{schuster2010}, where any increase in prediction time degrades the product experience \citep{namkoong2020distilled}. Simultaneously, researchers are considering increasingly larger architectures that improve predictive performance \citep{Real_Aggarwal_Huang_Le_2019}.  Therefore, a firm may be interesting understanding the set of optimal trade-offs between prediction latency and predictive performance. For a demonstration, we consider optimizing ($d=8$) hyperparameters of DNN (detailed in Table \ref{table:fcnet_search_space}) to minimize out-of-sample prediction error and minimize latency on the MNIST data set \citep{lecun2010mnist}. Using a small randomized test set leads to noisy evaluations of predictive performance and latency measurements are often noisy due to unrelated fluctuations in the testing environment. As in previous works, we minimize a logit transformation of the prediction error and minimize a logarithm of the ratio between the latency of a proposed DNN and the latency of the fastest DNN \citep{pesmo,garrido2019predictive,garridomerchn2020parallel}. For each evaluation, we randomly partition the 60,000 examples from the MNIST training set into a set of 50,000 examples for training and 10,000 examples for evaluation. We train each network for 8 epochs using SGD with momentum with mini-batches of 512 examples. The learning rate is decayed after every 30 mini-batch updates using the specified decay multiplier. We use randomized rounding on the integer parameters before evaluation. For evaluating the performance of different BO methods, we estimate the noiseless objectives using the mean objectives across 3 replications. DNNs are implemented in PyTorch using ReLU activations and a softmax output layer. Latency measurements are taken on a CPU (2x Intel Xeon E5-2680 v4 @ 2.40GHz).
\begin{table*}
\centering
\begin{small}
\begin{sc}
\begin{tabular}{lc}
\toprule
Parameter & Search Space\\
\midrule
Learning rate ($\log_{10}$ scale) & [-5.0, -1.0]\\
Learning rate decay multiplier & [0.01, 1.0]\\
Dropout rate & [0.0, 0.7]\\
$\text{L}_1$ regularization & [$10^{-5}$, 0.1]\\
$\text{L}_2$ regularization & [$10^{-5}$, 0.1]\\
Hidden Layer 1 Size & [20, 500]\\
Hidden Layer 2 Size & [20, 500]\\
Hidden Layer 3 Size & [20, 500]\\
\bottomrule
\end{tabular}
\end{sc}
\end{small}
\caption{\label{table:fcnet_search_space} The search space for the AutoML benchmark.}
\end{table*}

\paragraph{CarSideImpact ($\bm{M=4, ~d=7}$)}
A side-impact test is common practice under European Enhanced Vehicle-Safety Committee to uphold vehicle safety standards \citep{deb2009}.
In constrast with the previous VehicleSafety problem where we considered a full-frontal collision, we now consider the problem of tuning parameters controlling the structural design the of an automobile in the case of a  \emph{side-impact} collision. This problem has been widely used in various works and has previously used stochastic parameters to account for manufacturing error \citep{deb2009}. We use the recent 4-objective version proposed by \citet{tanabe2020} where the goal to minimize the weight of the vehicle, passenger trauma (pubic force), and vehicle damage (the average velocity of the V-pillar). The fourth objective is a combination of 10 other measures of the vehicle durability and passenger safety (see \citep{deb2009} for details). The mathematical formulas for a response surface model fit to data collected from a simulator are given below:

\begin{align*}
    f^{(1)}(\bm x) &= 1.98 + 4.9x_1 + 6.67x_2 + 6.98x_3 + 4.01x_4 + 1.78x_5 + 10^{-5}x_6 + 2.73x_7 \\
    f^{(2)}(\bm x) &=4.72 - 0.5x_4 - 0.19x_2x_3\\
    f^{(3)}(\bm x) &=0.5(V_\text{MBP}(\bm x) + V_\text{FD}(\bm x))\\
    f^{(4)}(\bm x) &=-\sum_{i=1}^{10} \max[g_i(\bm x), 0]
\end{align*}
where
\begin{align*}
    g_1(\bm x) &= 1 - 1.16 + 0.3717x_2x_4 + 0.0092928x_3\\
    g_2(\bm x) &= 0.32 - 0.261 + 0.0159x_1x_2 + 0.06486x_1 + 0.019x_2x_7 - 0.0144x_3x_5 - 0.0154464x_6\\
    g_3(\bm x) &=0.32 - 0.214 - 0.00817x_5 + 0.045195x_1 + 0.0135168x_1 - 0.03099x_2x_6\\
    &+ 0.018x_2x_7 - 0.007176x_3 - 0.023232x_3 + 0.00364x_5x_6 + 0.018x_2^2\\
    g_4(\bm x) &= 0.32 - 0.74 + 0.61x_2 + 0.031296x_3 + 0.031872x_7 - 0.227x_2^2\\
    g_5(\bm x) &= 32 - 28.98 - 3.818x_3 + 4.2x_1x_2 - 1.27296x_6 + 2.68065x_7\\
    g_6(\bm x) &=32 - 33.86 - 2.95x_3 + 5.057x_1x_2 + 3.795x_2 + 3.4431x_7 - 1.45728\\
    g_7(\bm x) &= 32 - 46.36 + 9.9x_2 + 4.4505x_1\\
    g_8(\bm x) &= 4 - f_2(\bm x)\\
    g_9(\bm x) &= 9.9 - V_\text{MBP}(\bm x)\\
    g_{10}(\bm x) &= 15.7 - V_\text{FD}(\bm x)\\
    V_\text{MBP}(\bm x) &=10.58 - 0.674x_1x_2 - 0.67275x_2\\
    V_\text{FD}(\bm x) &=16.45 - 0.489x_3x_7 - 0.843x_5x_6
\end{align*}.
The search space is:
\begin{align*}
    x_1 &\in [0.5, 1.5]\\
    x_2 &\in [0.45, 1.35]\\
    x_3, x_4 &\in [0.5, 1.5]\\
    x_5 &\in [0.875, 2.625]\\
    x_6, x_7 &\in [0.4, 1.2].
\end{align*}
As in the VehicleSafety problem, we add zero-mean Gaussian noise to each objective with a standard deviation of 1\% the range of each objective.

\paragraph{Constrained BraninCurrin ($\bm{M=2, ~V=2, ~d=2}$)}
The constrained BraninCurrin problem uses the same objectives as BraninCurrin, but adds the following disk constraint from \citep{gelbart2014unknowncon}:
 $$c(x_1', x_2') = 50 - (x_1' - 2.5)^2 - (x_2' - 7.5)^2
)  \geq 0$$

We add zero-mean Gaussian noise to objectives and the constraint slack observations with a standard deviation of 5\% of the range of each outcome.

\paragraph{SphereEllipsoidal ($\bm{M=2, ~d=5}$)}
The SphereEllipsoidal problem is defined over $\bm x \in [-5, 5]^d$ and the objectives are given by \citep{brockhoff2019using}:
 \begin{align*}
     f^{(1)}(\bm x) &= \sum_{i=1}^d(x_i - x^{(1)}_{\text{opt}, i})^2 + f^{(1)}_\text{opt}\\
     f^{(2)}(\bm x) &= \sum_{i=1}^d 10^{6 \frac{i - 1}{d - 1}}z_i^2 + f^{(2)}_\text{opt}
 \end{align*}
 where \begin{align*}
     z_i &= T_\text{osz}(\delta_i)\\
     \delta_i &= x_i - x^{(2)}_{\text{opt}, i}\\
     T_\text{osz}(\delta_i) &= \text{sign}(\delta_i) e^{\hat{\delta_i} + 0.049\big(\sin\big[c_1(\delta_i) \hat{\delta_i}\big] + \sin\big[c_2(\delta_i)\hat{\delta_i}\big]\big)}
 \end{align*}
 and
 \[
    \hat{\delta_i} =
\begin{cases}
    \log(|\delta_i|),& \text{if } \delta_i\neq 0\\
    0,              & \text{otherwise}
\end{cases}
\]
 \[
    c_1(\delta_i) =
\begin{cases}
    10,& \text{if } \delta_i\geq 0\\
    5.5,              & \text{otherwise}
\end{cases}
\]
 \[
    c_2(\delta_i) =
\begin{cases}
    7.9,& \text{if } \delta_i\geq 0\\
    3.1,              & \text{otherwise}.
\end{cases}
\]
We set 
\begin{align*}
    \bm x_\text{opt}^{(1)} &= [-0.0299,  2.1458, -3.2922, -2.9438, -1.5406]\\
    \bm x_\text{opt}^{(2)} &= [ 2.0611, -1.7655, -0.7754,  1.8775, -3.7657]\\
    f_\text{opt}^{(1)} &= 203.71\\
    f_\text{opt}^{(2)} &= 135.6
\end{align*}.
We add zero-mean Gaussian noise to objectives and the constraint slack observations with a standard deviation of 5\% of the range of each outcome.

\subsection{Evaluation Details}
\label{appdx:sec:eval_details}
To compute the log hypervolume difference metric, we use NSGA-II to estimate the true Pareto frontier (except for the ABR and AutoML problems, where evaluations are time-consuming and we instead take the true Pareto frontier to be the Pareto frontier across the estimated objectives across all methods and replications). Using this Pareto frontier, we compute the hypervolume dominated by the true Pareto frontier in order to calculate the log hypervolume difference. For ZDT1, the hypervolume dominated by the true Pareto frontier can be computed analytically. For Constrained BraninCurrin, we evaluate the logarithm of the difference between the hypervolume dominated by the true feasible Pareto frontier and the feasible in-sample Pareto frontier for each method.

For all problems, we selected the reference point based on the component-wise noiseless nadir point $\bm f_\text{nadir}(\bm x) = \min_{\bm x \in 
\mathcal X} \bm f(\bm x)$ and the range of the Pareto frontier for each noiseless objective using the common heuristic \citep{WANG201725}:
$\bm r = \bm f_\text{nadir}(\bm x) - \beta * (\bm f_\text{ideal}(\bm x) - \bm f_\text{nadir}(\bm x)),$ where $\beta = 0.1$ and $\bm f_\text{ideal}(\bm x) = \max_{\bm x \in 
\mathcal X} \bm f(\bm x)$.

\begin{table*}
\centering
\begin{small}
\begin{sc}
\begin{tabular}{lc}
\toprule
Problem & Reference Point\\
\midrule
BraninCurrin & [-18.00, -6.00]\\
ZDT1 & [-1.10, -1.10]\\
DTLZ2 & $[-1.10]^M$\\
VehicleSafety & [-1698.55, -11.21, -0.29]\\
ABR & [150.00, -3500.00]\\
AutoML & [-2.45, 0.60] \\
CarSideImpact & [-45.49, -4.51, -13.34, -10.39]\\
ConstrainedBraninCurrin & [-80.00, -12.00]\\
SphereEllipsoidal & [-261.00, $-6.77\cdot10^6$]\\
\bottomrule
\end{tabular}
\end{sc}
\end{small}
\caption{\label{table:ref_points} The reference points for each benchmark problem.}
\end{table*}
\section{Experiments}
\label{appdx:sec:extra_experiments}
\subsection{Wall Time Results}
Tables \ref{table:walltime_cpu} and \ref{table:walltime_cpu2} report the acquisition optimization wall times for each method. On all benchmark problems except CarSideImpact, \qNEHVI{} is faster to optimize than MESMO and PFES on a GPU. The wall times for optimizing \qNEHVI{} are competitive with those for $q$NParEGO on most benchmark problems and batch sizes; on many problems, \qNEHVI{} is often faster than $q$NParEGO. On the problems VehicleSafety and CarSideImpact problems which have 3 and 4 objectives respectively, we observed tractable wall times, even when generating $q=32$ candidates. The wall time for 3 and 4 objective problems is larger primarily because the box decompositions are more time-consuming to compute and result in more hyperrectangles as the number of objectives increases. Although, \qEHVI{}(-PM) is faster for $q=1$ and $q=8$ on many problems, it is unable to scale to large batch sizes and ran out of memory for $q=8$ on CarSideImpact due to the box decomposition having a large number of hyperrectangles.
\label{appdx:sec:wall_time_results}
\FloatBarrier
\begin{table*}
\centering
\begin{small}
\begin{sc}
\resizebox{\textwidth}{!}{
\begin{tabular}{lcccccc}
\toprule
\textbf{CPU} & BraninCurrin & ZDT1 & ABR & VehicleSafety\\
\midrule
MESMO (\textit{q}=1) & $21.24 ~(\pm 0.02)$ & $19.76 ~(\pm 0.03)$ & $23.24 ~(\pm 0.04)$ & $28.39 ~(\pm 0.07)$\\
PFES (\textit{q}=1) & $22.86 ~(\pm 0.05)$ & $39.82 ~(\pm 0.14)$ & $43.03 ~(\pm 0.12)$ & $53.16 ~(\pm 0.17)$\\
TS-TCH (\textit{q}=1) & $\bm{0.51 ~(\pm 0.0)}$ & $\bm{0.48 ~(\pm 0.0)}$ & $\bm{0.75 ~(\pm 0.0)}$ & $\bm{0.67 ~(\pm 0.0)}$\\
\qEHVI-PM{}-CBD (\textit{q}=1) & $2.34 ~(\pm 0.02)$ & $3.7 ~(\pm 0.02)$ & $3.56 ~(\pm 0.03)$ & $7.82 ~(\pm 0.05)$\\
\qEHVI{} (\textit{q}=1) & $0.58 ~(\pm 0.0)$ & $0.66 ~(\pm 0.01)$ & $2.98 ~(\pm 0.02)$ & $5.07 ~(\pm 0.03)$\\
\qNEHVI{} (\textit{q}=1) & $40.55 ~(\pm 0.61)$ & $35.66 ~(\pm 0.47)$ & $62.29 ~(\pm 0.97)$ & $120.43 ~(\pm 1.25)$\\
$q$NParEGO (\textit{q}=1) & $3.19 ~(\pm 0.05)$ & $1.65 ~(\pm 0.02)$ & $6.94 ~(\pm 0.06)$ & $1.05 ~(\pm 0.01)$\\
$q$ParEGO (\textit{q}=1) & $0.58 ~(\pm 0.01)$ & $0.7 ~(\pm 0.01)$ & $2.5 ~(\pm 0.03)$ & $0.75 ~(\pm 0.01)$\\
DGEMO (\textit{q}=1) &$65.28 (\pm 0.26)$ & $76.99 (\pm 0.35)$ & NA& NA\\
DGEMO (\textit{q}=8) &$65.44 (\pm 0.63)$ & $86.97 (\pm 0.85)$ & NA& NA\\
DGEMO (\textit{q}=16) & $66.44 (\pm 0.93)$&$86.06 (\pm 1.21)$  & NA&NA \\
DGEMO (\textit{q}=32) &$66.66 (\pm 1.47)$ &$84.66 (\pm 1.66)$  & NA& NA\\
TSEMO (\textit{q}=1) & $3.02 (\pm 0.01)$& $2.98 (\pm 0.01)$ & NA& $3.61 (\pm 0.01)$\\
TSEMO (\textit{q}=8) &$3.53 (\pm 0.01)$ & $3.48 (\pm 0.01)$ & NA& $7.45 (\pm 0.1)$\\
TSEMO (\textit{q}=16) &$3.77 (\pm 0.02)$ & $3.74 (\pm 0.02)$ &NA & $11.06 (\pm 0.28)$\\
TSEMO (\textit{q}=32) &$4.29 (\pm 0.03)$ &$4.22 (\pm 0.02)$  &NA & $16.3 (\pm 0.68)$\\
MOEA/D-EGO (\textit{q}=1) &$57.79 (\pm 0.17)$ & $58.1 (\pm 0.17)$ & NA& $71.0 (\pm 0.21)$\\
MOEA/D-EGO (\textit{q}=8) & $63.56 (\pm 0.18)$& $63.57 (\pm 0.17)$ & NA&$77.56 (\pm 0.22)$ \\
MOEA/D-EGO (\textit{q}=16) & $64.0 (\pm 0.18)$&$63.99 (\pm 0.19)$  & NA& $78.03 (\pm 0.25)$\\
MOEA/D-EGO (\textit{q}=32) & $64.09 (\pm 0.26)$& $63.9 (\pm 0.24)$ &NA & $77.77 (\pm 0.35)$\\
\midrule
\textbf{GPU} & BraninCurrin & ZDT1 & abr & VehicleSafety \\
\midrule
MESMO (\textit{q}=1) & $19.9 ~(\pm 0.04)$ & $19.92 ~(\pm 0.04)$ & $21.54 ~(\pm 0.08)$ & $24.57 ~(\pm 0.09)$\\
PFES (\textit{q}=1) & $21.68 ~(\pm 0.07)$ & $45.9 ~(\pm 0.17)$ & $43.3 ~(\pm 0.13)$ & $47.25 ~(\pm 0.16)$\\
TS-TCH{} (\textit{q}=1) & $0.88 ~(\pm 0.01)$ & $0.94 ~(\pm 0.01)$ & $1.08 ~(\pm 0.01)$ & $1.04 ~(\pm 0.01)$\\
TS-TCH (\textit{q}=8) & $1.85 ~(\pm 0.03)$ & $2.01 ~(\pm 0.03)$ & $2.99 ~(\pm 0.04)$ & $2.32 ~(\pm 0.05)$\\
TS-TCH (\textit{q}=16) & $3.08 ~(\pm 0.08)$ & $3.29 ~(\pm 0.1)$ & $4.28 ~(\pm 0.08)$ & $3.54 ~(\pm 0.09)$\\
TS-TCH (\textit{q}=32) & $5.25 ~(\pm 0.15)$ & $5.41 ~(\pm 0.16)$ & $7.23 ~(\pm 0.2)$ & $6.41 ~(\pm 0.23)$\\
\qEHVI-PM{}-CBD (\textit{q}=1) &$2.17 (\pm0.01)$&$2.12 (\pm0.02)$ &  $3.59 (\pm0.02)$& $51.11 (\pm0.28)$\\
\qEHVI-PM{}-CBD (\textit{q}=8) &$39.56 (\pm0.79)$&$30.83 (\pm1.35)$ & $36.2 (\pm0.73)$ & $716.03 (\pm13.44)$\\
\qEHVI-PM{}-CBD (\textit{q}=16) &$82.91 (\pm2.42)$& $67.3 (\pm3.88)$&  $70.02 (\pm1.64)$& $1410.79 (\pm41.72)$\\
\qEHVI-PM{}-CBD (\textit{q}=32) &$147.81 (\pm6.85)$& $105.74 (\pm8.55)$& $251.97 (\pm12.69)$ & $2570.95 (\pm116.61)$\\
\qEHVI{} (\textit{q}=1) & $0.72 ~(\pm 0.01)$ & $0.99 ~(\pm 0.02)$ & $3.67 ~(\pm 0.02)$ & $3.96 ~(\pm 0.05)$\\
\qEHVI{} (\textit{q}=8) & $18.12 ~(\pm 1.03)$ & $18.05 ~(\pm 0.86)$ & $40.55 ~(\pm 0.58)$ & $71.49 ~(\pm 2.04)$\\
\qNEHVI{} (\textit{q}=1) & $6.15 ~(\pm 0.06)$ & $5.75 ~(\pm 0.04)$ & $7.72 ~(\pm 0.09)$ & $20.81 ~(\pm 0.11)$\\
\qNEHVI{} (\textit{q}=8) & $48.19 ~(\pm 1.2)$ & $46.74 ~(\pm 0.83)$ & $49.7 ~(\pm 0.79)$ & $168.63 ~(\pm 2.49)$\\
\qNEHVI{} (\textit{q}=16) & $102.87 ~(\pm 4.02)$ & $95.6 ~(\pm 2.62)$ & $93.14 ~(\pm 1.72)$ & $289.02 ~(\pm 5.82)$\\
\qNEHVI{} (\textit{q}=32) & $177.56 ~(\pm 7.81)$ & $190.59 ~(\pm 6.07)$ & $181.97 ~(\pm 4.77)$ & $546.83 ~(\pm 16.09)$\\
\TSHVI{} (\textit{q}=1) & $\bm{0.32 (\pm0.0)}$& $\bm{0.24 (\pm0.0)}$&$\bm{0.56 (\pm0.0)}$ &$\bm{0.92 (\pm0.0)}$\\
\TSHVI{} (\textit{q}=8) & $2.43 (\pm0.02)$& $2.11 (\pm0.03)$& $4.55 (\pm0.06)$&$7.1 (\pm0.1)$\\
\TSHVI{} (\textit{q}=16) &$4.97 (\pm0.07)$ & $3.73 (\pm0.06)$& $8.73 (\pm0.14)$&$14.77 (\pm0.31)$\\
\TSHVI{} (\textit{q}=32) & $9.03 (\pm0.18)$& $8.18 (\pm0.24)$& $17.15 (\pm0.41)$&$34.99 (\pm1.46)$\\
$q$NParEGO (\textit{q}=1) & $2.39 ~(\pm 0.04)$ & $1.84 ~(\pm 0.04)$ & $6.47 ~(\pm 0.05)$ & $0.9 ~(\pm 0.02)$\\
$q$NParEGO (\textit{q}=8) & $47.05 ~(\pm 1.74)$ & $52.99 ~(\pm 1.94)$ & $74.72 ~(\pm 1.9)$ & $45.56 ~(\pm 1.17)$\\
$q$NParEGO (\textit{q}=16) & $118.73 ~(\pm 5.53)$ & $116.68 ~(\pm 5.51)$ & $116.79 ~(\pm 3.19)$ & $91.3 ~(\pm 3.83)$\\
$q$NParEGO (\textit{q}=32) & $306.17 ~(\pm 17.81)$ & $279.01 ~(\pm 17.72)$ & $240.56 ~(\pm 6.44)$ & $188.42 ~(\pm 13.42)$\\
$q$ParEGO (\textit{q}=1) & $0.81 ~(\pm 0.02)$ & $1.05 ~(\pm 0.03)$ & $4.39 ~(\pm 0.05)$ & $0.79 ~(\pm 0.02)$\\
$q$ParEGO (\textit{q}=8) & $13.01 ~(\pm 0.53)$ & $16.4 ~(\pm 0.72)$ & $31.02 ~(\pm 0.81)$ & $12.64 ~(\pm 0.84)$\\
$q$ParEGO (\textit{q}=16) & $34.34 ~(\pm 2.12)$ & $43.66 ~(\pm 3.12)$ & $66.85 ~(\pm 3.13)$ & $36.68 ~(\pm 4.48)$\\
$q$ParEGO (\textit{q}=32) & $139.73 ~(\pm 25.22)$ & $108.25 ~(\pm 6.94)$ & $122.37 ~(\pm 6.12)$ & $107.34 ~(\pm 14.76)$\\
\bottomrule
\end{tabular}
}
\end{sc}
\end{small}
\caption{\label{table:walltime_cpu} Acquisition function optimization wall time (including box decompositions) in seconds on a CPU (2x Intel Xeon E5-2680 v4 @ 2.40GHz) and a Tesla V100 SXM2 GPU (16GB RAM). The mean and two standard errors are reported. DGEMO, TSEMO, and MOEA/D-EGO are omitted for ABR because they have package requirements that are not easily compatible with our distributed evaluation pipeline and ABR evaluations are prohibitively slow without distributed evaluations. DGEMO is omitted for VehicleSafety because the open-source implementation raises an consistently raises an exception in the graph cutting algorithm with this problem.}
\end{table*}
\FloatBarrier
\subsection{Scaling to large batch sizes with \cbd{}}
In Figure \ref{fig:q_cache_iep_m3_4}, we provide results demonstrating the \cbd{} approach enables scaling to large batch sizes, even with 3 or 4 objectives, whereas the \iep{} wall times grow exponentially with the batch size and the \iep{} overflows GPU memory even with modest batch sizes.
\label{appdx:sec:cbd_vs_iep}
\FloatBarrier
\begin{figure*}[h]
    \centering
    \includegraphics[width=\textwidth]{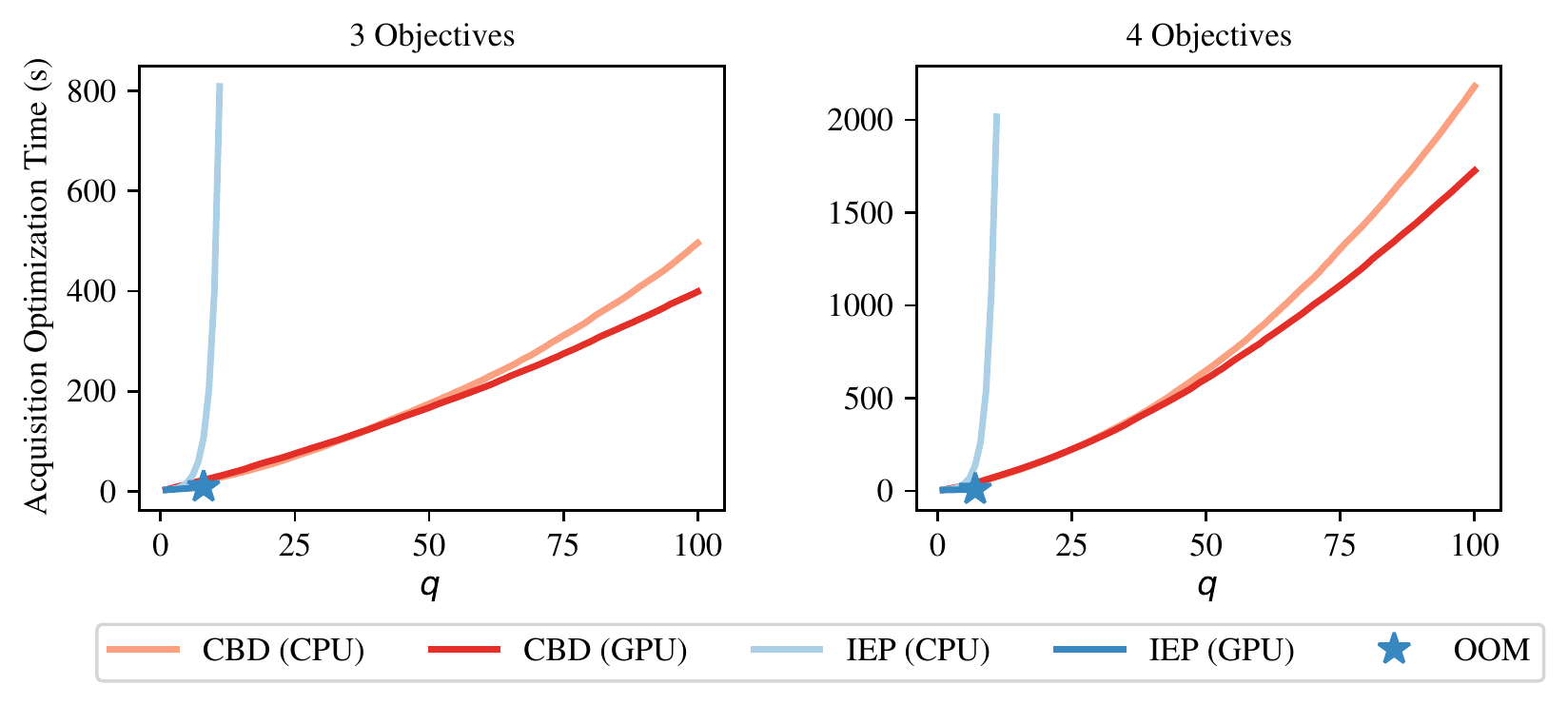}
    \vspace{-15pt}
    \caption{\label{fig:q_cache_iep_m3_4}Acquisition optimization wall time under a sequential greedy approximation using L-BFGS-B for three and four objectives. \cbd{} enables scaling to much larger batch sizes $q$ than using the inclusion-exclusion principle (\iep{}) and avoids running out-of-memory (OOM) on a GPU. Independent GPs are used for each outcome and are initialized with 20 points from the Pareto frontier of the 6-dimensional DTLZ2 problem \citep{dtlz} with 3 objectives (left) and 4 objectives (right). Wall times were measured on a Tesla V100 SXM2 GPU (16GB GPU RAM) and a Intel Xeon Gold 6138 CPU @ 2GHz CPU (251GB RAM).}
    \vspace{15pt}
\end{figure*}
\FloatBarrier
\subsection{Additional Empirical Results}
\label{appdx:sec:additional_results}
The additional optimization performance results in the appendix demonstrate that \qNEHVI{}-based algorithms are consistently the top performer. The only case where \qNEHVI{}(-1) is outperformed is in the sequential setting on the CarSideImpact problem in Figure \ref{fig:hv_sequential_additional_reuslts}(a), where \qEHVI{} performs best However, as show in Figure \ref{fig:hv_sequential_additional_reuslts}(b) and Figure \ref{fig:hv_sequential_additional_reuslts}(c), \qNEHVI{} enables scaling to large batch sizes, whereas \qEHVI{} runs out of memory on a GPU for $q=8$. Therefore, in a practical setting where vehicles are manufactured and test in parallel, \qNEHVI{} would be the best choice.

Figure \ref{fig:anytime} shows that \qNEHVI{} achieves solid performance anytime throughout the learning curve in the sequential setting, and Figure \ref{fig:hv_over_q_additional} shows that \qNEHVI{}-based variants consistently achieves the best performance for various $q$ with a fixed budget of 224 function evaluations.  Although, \TSHVI{} does not consistently perform better than \qNEHVI{}, \TSHVI{} achieves very little degradation of sample complexity as the batch size $q$ increases.

\subsection{Better performance from \qEHVI{} with a larger batch size}
Interestingly, on many test problems \qEHVI{} performs better with $q=8$ than $q=1$. In the case of BraninCurrin and ConstrainedBraninCurrin, $q$ParEGO also improves as $q$ increases. Since this phenomenon is not observed with the noisy acquisition functions (\qNEHVI{}, $q$NParEGO), we hypothesize that it may be the case that sequential data collection results in a difficult to optimize acquisition surface and that integrating over the in-sample points leads to a smoother acquisition surface that results in improved sequential optimization. The acquisition functions that do not account for noise may be misled by the noise in sequential setting, but using a larger batch size (within limits) may help avoid the issue of not properly accounting for noise.
\FloatBarrier
\subsection{Performance over Higher Dimensional Spaces}
\label{sec:high_d_tshvi}
\TSHVI{} relies on approximate GP sample paths (RFFs). Although we find that \TSHVI{} performs very well on many low-dimensional problems (see Figure~\ref{fig:q_hv}), Figures \ref{fig:nehvi1_dimension} and \ref{fig:various_dims} show that \TSHVI{} does not perform as well a \qNEHVI{} on higher dimensional problems. We hypothesize that the RFF approximation degrades in higher dimensional search spaces leading to poor optimization performance relative \qNEHVI{}, which uses exact GP samples. It is likely that using 500 Fourier basis functions leads to large approximation error on high dimensional search spaces. Further study is needed to examine whether robust performance can be achieved by increasing the number of basis functions. As shown in Figure~\ref{fig:various_dims}, \qNEHVI{} consistently outperforms all tested methods regardless of the dimension of the design space.
\begin{figure}[h]
    \centering
    \includegraphics[width=\textwidth]{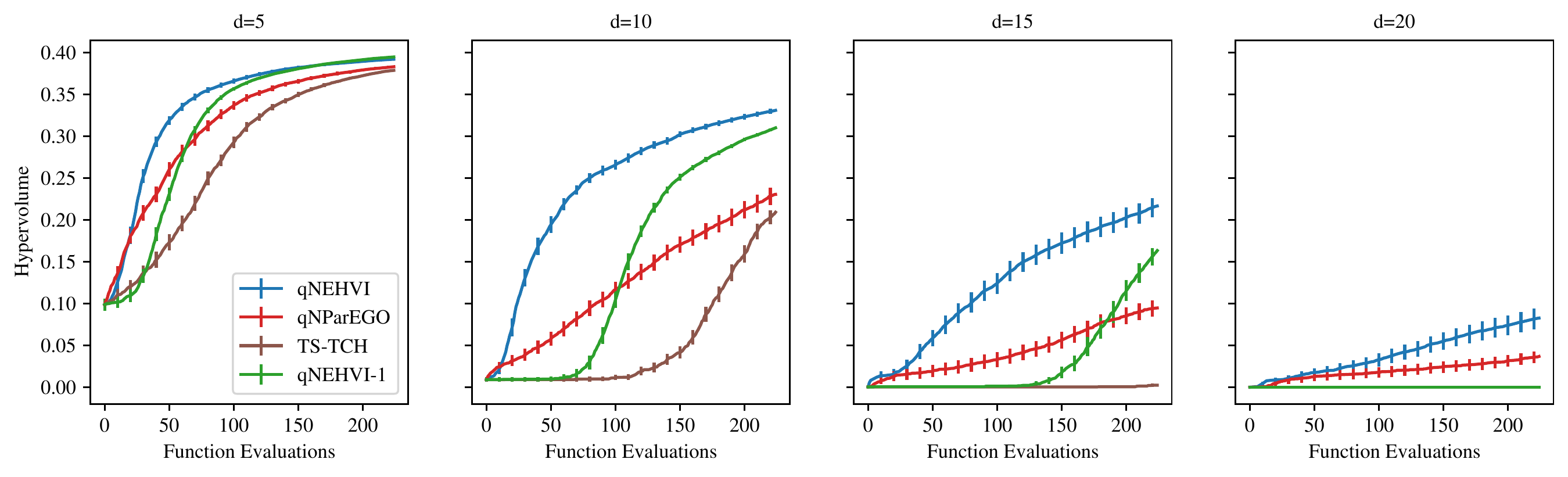}
    \caption{\label{fig:various_dims}Sequential optimization performance  2-objective DTLZ2 with $\sigma=5\%$ problems as the dimension of the search space increases from $d=5$ to $d=20$.}
\end{figure}
\begin{figure}[h]
    \centering
    \includegraphics[width=\textwidth]{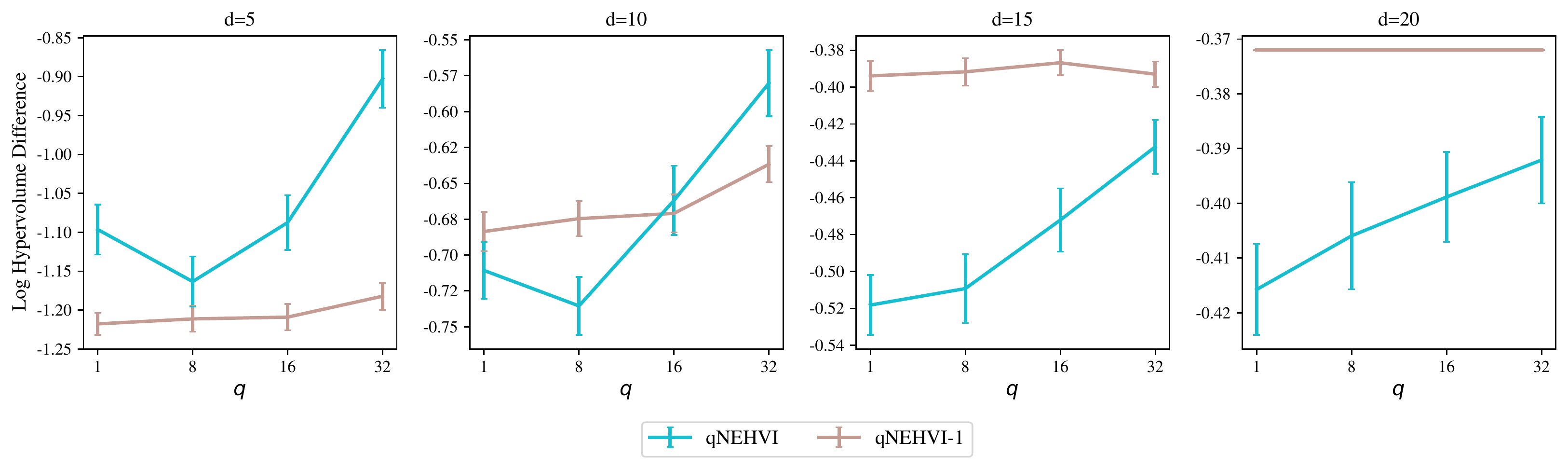}
    \caption{\label{fig:nehvi1_dimension}A comparison of the final optimization performance of \TSHVI{}, a single sample path approximation of \qNEHVI{}, and \qNEHVI{} on 2-objective DTLZ2 problems with input dimensions between 5 and 20 under different batch sizes $q$. \TSHVI{} is very effective on lower dimensional problems, but does not perform as well as \qNEHVI{} on higher dimensional problems.}
\end{figure}

\FloatBarrier
\begin{table*}
\centering
\begin{small}
\begin{sc}
\resizebox{\textwidth}{!}{
\begin{tabular}{lcccc}
\toprule
\textbf{CPU} & DTLZ2 & AutoML & CarSideImpact& Constrained BraninCurrin\\
\midrule
MESMO (\textit{q}=1)&$27.79 (\pm0.07)$&$37.86 ~(\pm 0.08)$&$33.31 ~(\pm 0.06)$& NA\\
PFES (\textit{q}=1)&$69.85 (\pm0.21)$&$101.24 ~(\pm 0.29)$&$102.55 ~(\pm 0.34)$& NA\\
TS-TCH (\textit{q}=1)&$\bm{0.72 (\pm0.0)}$&$\bm{0.93 ~(\pm 0.0)}$&$\bm{1.27 ~(\pm 0.01)}$& NA\\
\qEHVI-PM{}-CBD (\textit{q}=1) &$3.98 (\pm0.02)$ &$5.76 ~(\pm 0.05)$&$83.14 ~(\pm 0.74)$& $10.27 ~(\pm 0.06)$\\
\qEHVI{} (\textit{q}=1) &$2.74 (\pm0.01)$&$5.05 ~(\pm 0.04)$&$96.19 ~(\pm 0.9)$& $3.26 ~(\pm 0.05)$\\
\qNEHVI{} (\textit{q}=1) &$54.04 (\pm0.67)$&$22.71 ~(\pm 0.48)$&$541.13 ~(\pm 6.83)$& $267.67 ~(\pm 4.09)$\\
$q$NParEGO (\textit{q}=1) &$21.39 (\pm0.2)$&$5.6 ~(\pm 0.07)$&$3.38 ~(\pm 0.05)$& $12.05 ~(\pm 0.17)$\\
$q$ParEGO (\textit{q}=1) &$2.33 (\pm0.02)$&$3.06 ~(\pm 0.03)$&$1.91 ~(\pm 0.02)$& $\bm{1.56 ~(\pm 0.03)}$\\
DGEMO (\textit{q}=1) &$84.48 (\pm 0.64)$&NA&NA& NA \\
DGEMO (\textit{q}=8) &$65.99 (\pm 0.51)$&NA&NA&NA\\
DGEMO (\textit{q}=16) &$69.57 (\pm 0.59)$&NA&NA&NA\\
DGEMO (\textit{q}=32) &$72.82 (\pm 0.8)$&NA&NA&NA\\
TSEMO (\textit{q}=1) &$3.1 (\pm 0.01)$&NA&$14.91 (\pm 0.16)$&NA\\
TSEMO (\textit{q}=8) &$3.58 (\pm 0.01)$&NA&$100.78 (\pm 3.75)$&NA\\
TSEMO (\textit{q}=16) &$3.87 (\pm 0.02)$&NA&$188.44 (\pm 9.88)$&NA\\
TSEMO (\textit{q}=32) &$4.55 (\pm 0.03)$&NA&$326.51 (\pm 24.03)$&NA\\
MOEA/D-EGO (\textit{q}=1) &$59.3 (\pm 0.18)$&NA&$79.87 (\pm 0.16)$&NA\\
MOEA/D-EGO (\textit{q}=8) &$65.5 (\pm 0.18)$&NA&$85.34 (\pm 0.21)$&NA\\
MOEA/D-EGO (\textit{q}=16) &$65.81 (\pm 0.2)$&NA&$85.18 (\pm 0.3)$&NA\\
MOEA/D-EGO (\textit{q}=32) &$65.44 (\pm 0.3)$&NA&$84.8 (\pm 0.39)$&NA\\
\midrule
\textbf{GPU} &DTLZ2& AutoML & CarSideImpact& Constrained BraninCurrin\\
\midrule
MESMO (\textit{q}=1) &$17.01 (\pm0.13)$& $28.83 ~(\pm 0.2)$ & $26.0 ~(\pm 0.06)$ & NA\\
PFES (\textit{q}=1) &$46.63 (\pm0.51)$& $85.73 ~(\pm 0.4)$ & $55.41 ~(\pm 0.16)$ & NA\\
TS-TCH (\textit{q}=1) &$0.84 (\pm0.01)$& $1.49 ~(\pm 0.01)$ & $2.17 ~(\pm 0.01)$ & NA \\
TS-TCH (\textit{q}=8) &$3.71 (\pm0.07)$& $3.64 ~(\pm 0.1)$ & $4.15 ~(\pm 0.07)$ & NA \\
TS-TCH (\textit{q}=16) &$5.85 (\pm0.14)$& $6.16 ~(\pm 0.21)$ & $6.9 ~(\pm 0.19)$ & NA \\
TS-TCH (\textit{q}=32) &$9.33 (\pm0.3)$& $9.47 ~(\pm 0.47)$ & $10.2 ~(\pm 0.4)$& NA\\
\qEHVI-PM{}-CBD (\textit{q}=1) &$4.29 (\pm0.01)$& $5.26 (\pm0.05)$&$52.7 (\pm0.41)$  & $15.89 (\pm0.08)$\\
\qEHVI-PM{}-CBD (\textit{q}=8) &$40.85 (\pm0.41)$&$39.55 (\pm0.8)$ &  $460.18 (\pm9.76)$& $135.89 (\pm1.43)$\\
\qEHVI-PM{}-CBD (\textit{q}=16) &$93.39 (\pm1.87)$& $71.8 (\pm1.81)$&  $866.12 (\pm26.46)$& $314.42 (\pm5.44)$\\
\qEHVI-PM{}-CBD (\textit{q}=32) &$194.12 (\pm5.6)$&$143.83 (\pm4.88)$ &  $1682.28 (\pm72.5)$& $823.01 (\pm24.48)$\\
\qEHVI{} (\textit{q}=1) &$2.93 (\pm0.02)$& $4.67 ~(\pm 0.1)$ & $9.63 ~(\pm 0.05)$ & $5.69 ~(\pm 0.11)$\\
\qEHVI{} (\textit{q}=8) &$39.64 (\pm0.57)$& $104.48 ~(\pm 1.34)$ & OOM & $68.95 ~(\pm 2.57)$\\
\qNEHVI{} (\textit{q}=1) &$4.91 (\pm0.01)$& $7.95 ~(\pm 0.1)$ & $82.66 ~(\pm 0.63)$ & $20.47 ~(\pm 0.12)$\\
\qNEHVI{} (\textit{q}=8) &$39.96 (\pm0.35)$& $67.28 ~(\pm 1.87)$ & $683.06 ~(\pm 13.82)$ & $168.04 ~(\pm 1.85)$\\
\qNEHVI{} (\textit{q}=16) &$74.41 (\pm0.63)$& $145.66 ~(\pm 4.45)$ & $1289.4 ~(\pm 36.81)$ & $362.15 ~(\pm 9.08)$\\
\qNEHVI{} (\textit{q}=32) &$142.18 (\pm1.59)$& $247.92 ~(\pm 11.93)$ & $2480.41 ~(\pm 102.38)$ & $654.66 ~(\pm 23.48)$\\
\TSHVI{} (\textit{q}=1) &$\bm{0.42 (\pm0.0)}$& $\bm{0.53 (\pm0.0)}$ &  $\bm{2.11 (\pm0.01)}$& $\bm{1.57 (\pm0.02)}$\\
\TSHVI{} (\textit{q}=8) &$3.26 (\pm0.03)$& $4.55 (\pm0.09)$ &  $16.88 (\pm0.2)$& $11.01 (\pm0.16)$\\
\TSHVI{} (\textit{q}=16) &$6.34 (\pm0.04)$&  $7.22 (\pm0.19)$&  $34.78 (\pm0.63)$& $20.56 (\pm0.4)$\\
\TSHVI{} (\textit{q}=32) &$14.85 (\pm0.45)$&  $12.66 (\pm0.54)$&  $67.27 (\pm1.61)$ & $40.6 (\pm1.1)$\\
$q$NParEGO (\textit{q}=1) &$6.41 (\pm0.03)$& $4.86 ~(\pm 0.1)$ & $3.25 ~(\pm 0.06)$ & $6.17 ~(\pm 0.07)$\\
$q$NParEGO (\textit{q}=8) &$57.17 (\pm0.43)$& $48.62 ~(\pm 1.18)$ & $30.65 ~(\pm 0.92)$ & $38.66 ~(\pm 0.95)$\\
$q$NParEGO (\textit{q}=16) &$114.91 (\pm1.3)$& $122.09 ~(\pm 3.5)$ & $81.04 ~(\pm 3.44)$ & $84.27 ~(\pm 3.04)$\\
$q$NParEGO (\textit{q}=32) &$263.56 (\pm4.03)$& $275.41 ~(\pm 8.29)$ & $219.98 ~(\pm 9.75)$ & $199.6 ~(\pm 9.68)$\\
$q$ParEGO (\textit{q}=1) &$2.62 (\pm0.03)$& $3.31 ~(\pm 0.07)$ & $3.03 ~(\pm 0.06)$ & $2.83 ~(\pm 0.08)$\\
$q$ParEGO (\textit{q}=8) &$25.75 (\pm0.66)$& $33.84 ~(\pm 1.76)$ & $22.0 ~(\pm 0.88)$ & $20.28 ~(\pm 1.18)$\\
$q$ParEGO (\textit{q}=16) &$67.89 (\pm2.27)$& $77.25 ~(\pm 3.9)$ & $56.09 ~(\pm 3.09)$ & $50.06 ~(\pm 3.84)$\\
$q$ParEGO (\textit{q}=32) &$159.98 (\pm7.41)$& $217.55 ~(\pm 19.15)$ & $139.25 ~(\pm 9.41)$ & $135.39 ~(\pm 13.04)$\\

\bottomrule
\end{tabular}
}
\end{sc}
\end{small}
\caption{\label{table:walltime_cpu2} Acquisition function optimization wall time (including box decompositions) in seconds on a CPU (2x Intel Xeon E5-2680 v4 @ 2.40GHz) and a Tesla V100 SXM2 GPU (16GB RAM). The mean and two standard errors are reported. DGEMO, TSEMO, MOEA/D-EGO are omitted for AutoML because they have package requirements that are not easily compatible with our distributed training and evaluation pipeline, and they are omitted for ConstrainedBraninCurrin because they do not support constraints in the open-source implementation at \url{https://github.com/yunshengtian/DGEMO/tree/master}. DGEMO is omitted on CarSideImpact because the open-source implementation does not support more than 3 objectives.}
\end{table*}


\begin{figure*}[t]
\centering\subfloat[Sequential Optimization performance on additional benchmark problems.]{%
    \centering
    \includegraphics[width=0.82\textwidth]{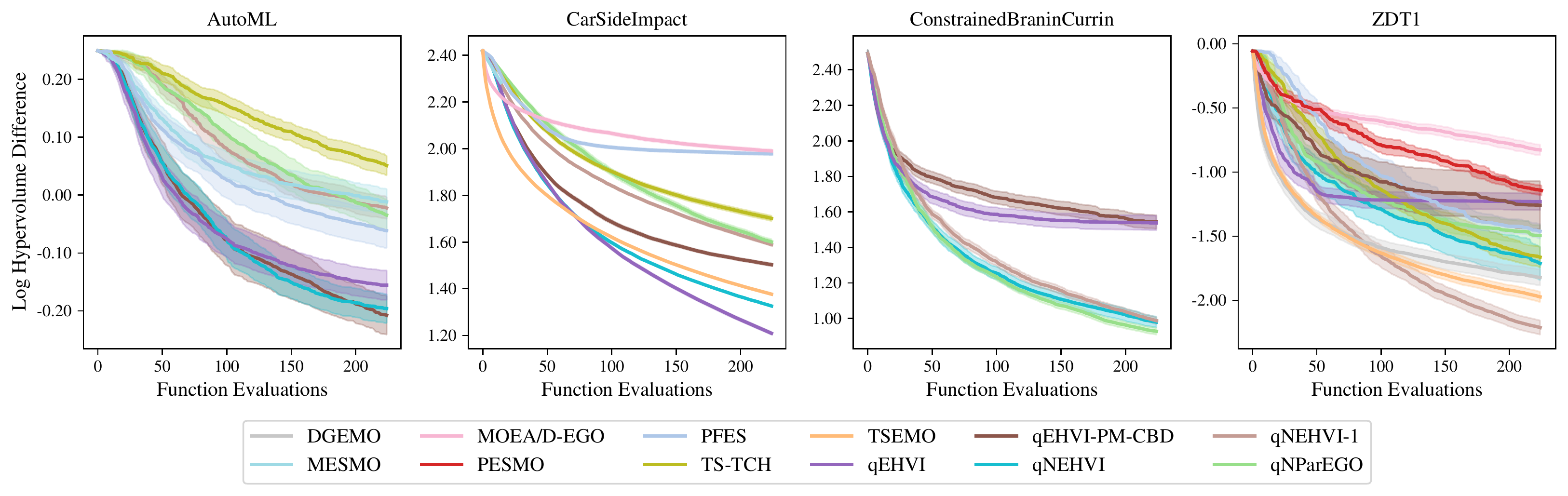}
    \vspace{-5pt}
    \vspace{-1ex}
}\\
\centering\subfloat[Parallel optimization performance vs batch iterations (1/2).]{%
    \centering
    \includegraphics[width=\textwidth]{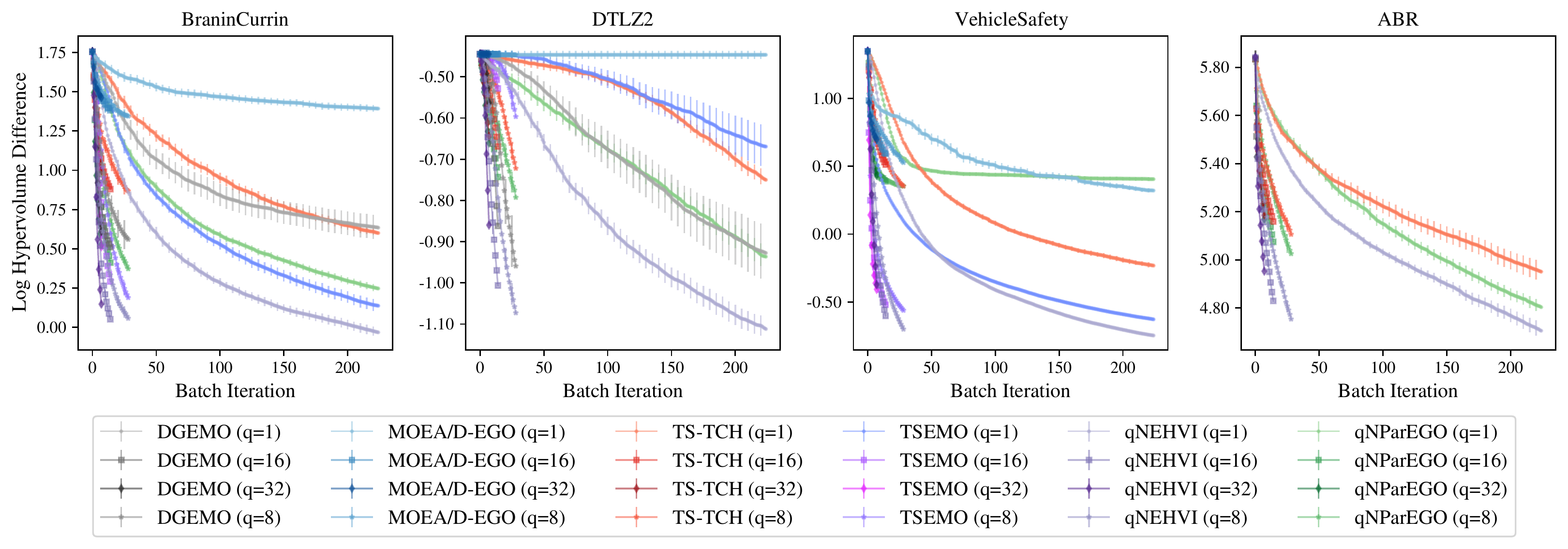}
    \vspace{-5pt}
    \vspace{-1ex}
}\\
\centering\subfloat[Parallel optimization performance vs batch iterations (2/2).]{%
    \centering
    \includegraphics[width=\textwidth]{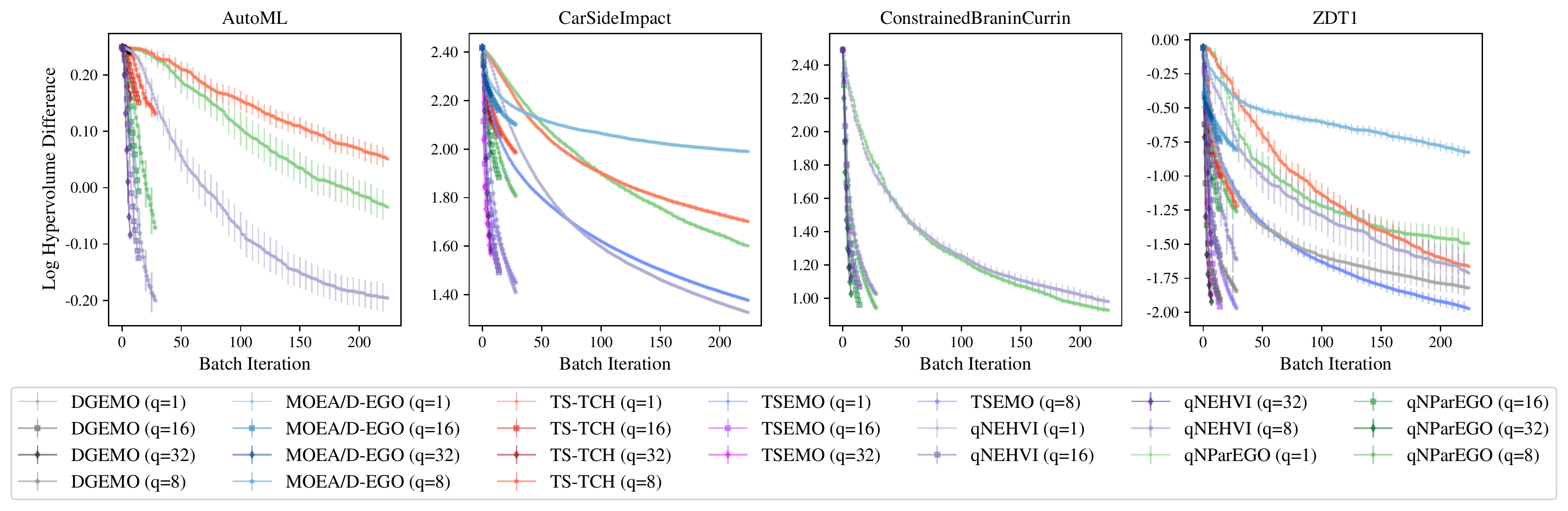}
    \vspace{-5pt}
    \vspace{-1ex}
    }

\caption{\label{fig:hv_sequential_additional_reuslts} Optimization performance on additional problems. (a) Sequential optimization performance. (b) and (c) Optimization performance of batch acquisition functions using various $q$ over the number of BO iterations. To improve readability, we omit \qEHVI{}(\textsc{-PM}) in this figure because the \iep{} cannot scale beyond $q=8$  because of the exponential time and space complexity (running it on a GPU runs out of memory and running it on a CPU results in prohibitively slow wall times). See Figure \ref{fig:batch_vs_qehvi} for results using \qEHVI{}(\textsc{-PM}).}
\end{figure*}

\begin{figure*}[t]
\centering\subfloat[]{%
    \centering
    \includegraphics[width=\textwidth]{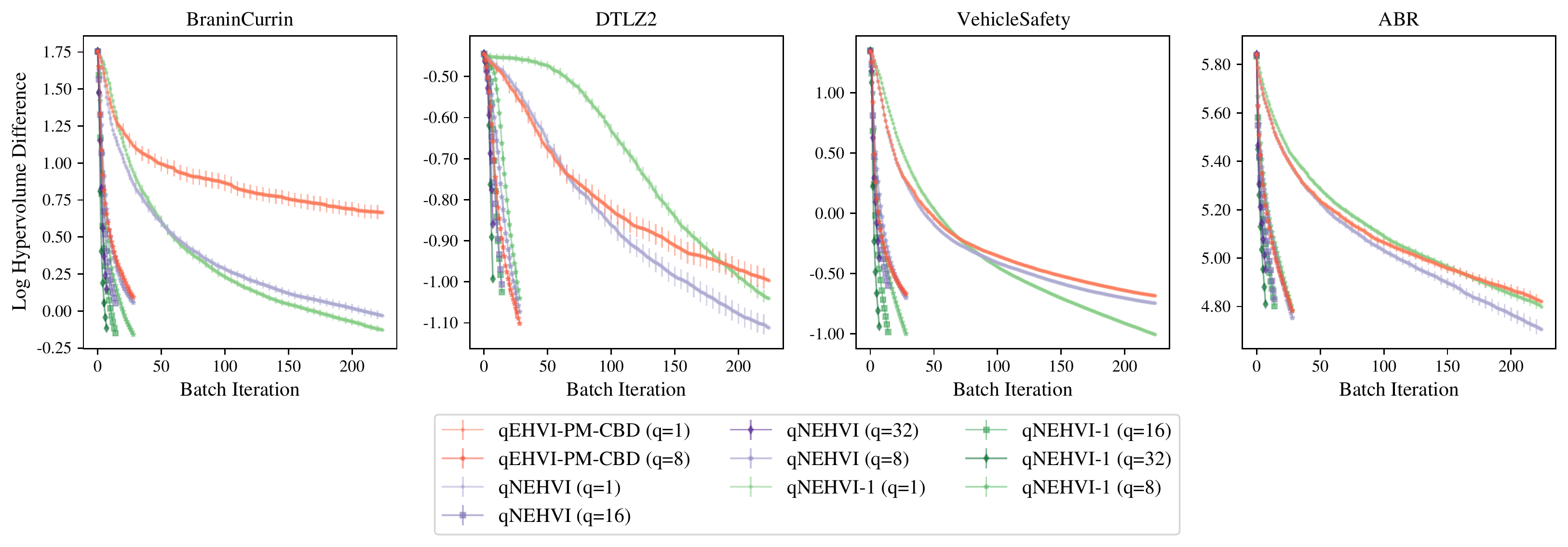}
    \vspace{-5pt}
    \vspace{-1ex}
}\\
\centering\subfloat[]{%
    \centering
    \includegraphics[width=0.91\textwidth]{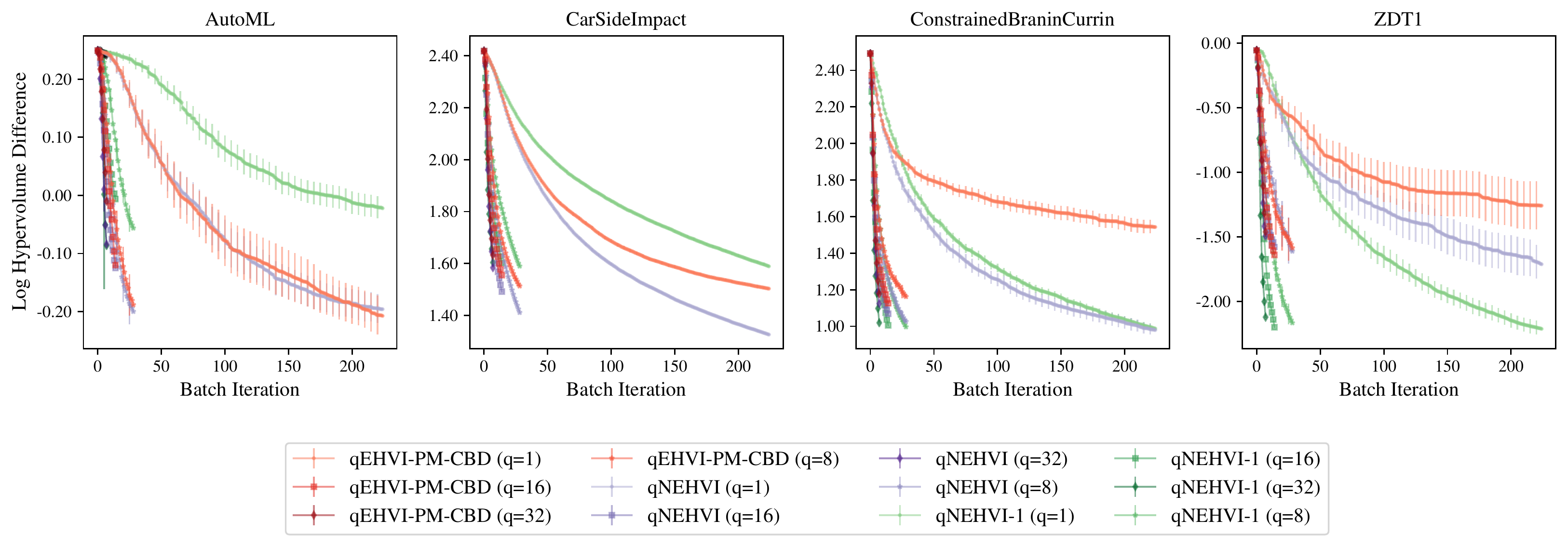}
    \vspace{-5pt}
    \vspace{-1ex}
    }
    \caption{\label{fig:batch_vs_qehvi} Optimization performance of \qNEHVI{} under various batch sizes $q$ vs \qEHVI{}(\textsc{-PM}). Note that using the  \iep{}, \qEHVI{}(\textsc{-PM}) cannot scale beyond $q=8$  because of the exponential time and space complexity (running it on a GPU runs out of memory and running it on a CPU results in prohibitively slow wall times).}
\end{figure*}

\begin{figure*}[t]
\centering\subfloat[]{%
    \centering
    \includegraphics[width=\textwidth]{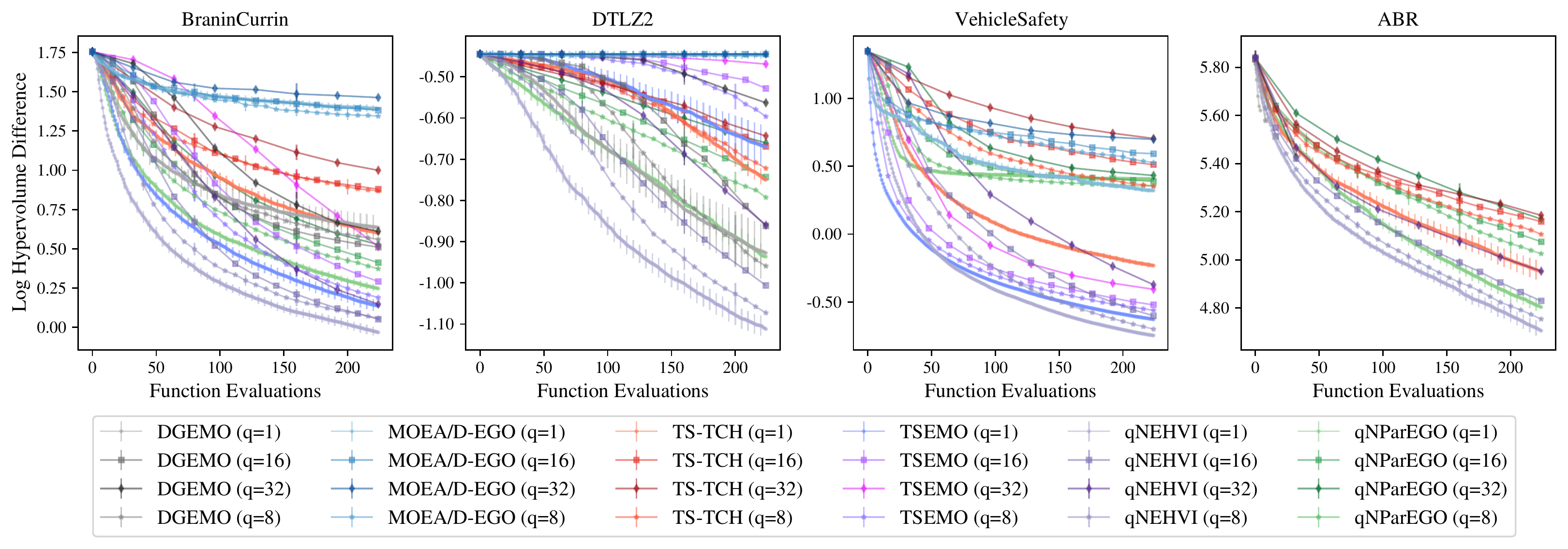}
    \vspace{-5pt}
    \vspace{-1ex}
}\\
\centering\subfloat[]{%
    \centering
    \includegraphics[width=0.91\textwidth]{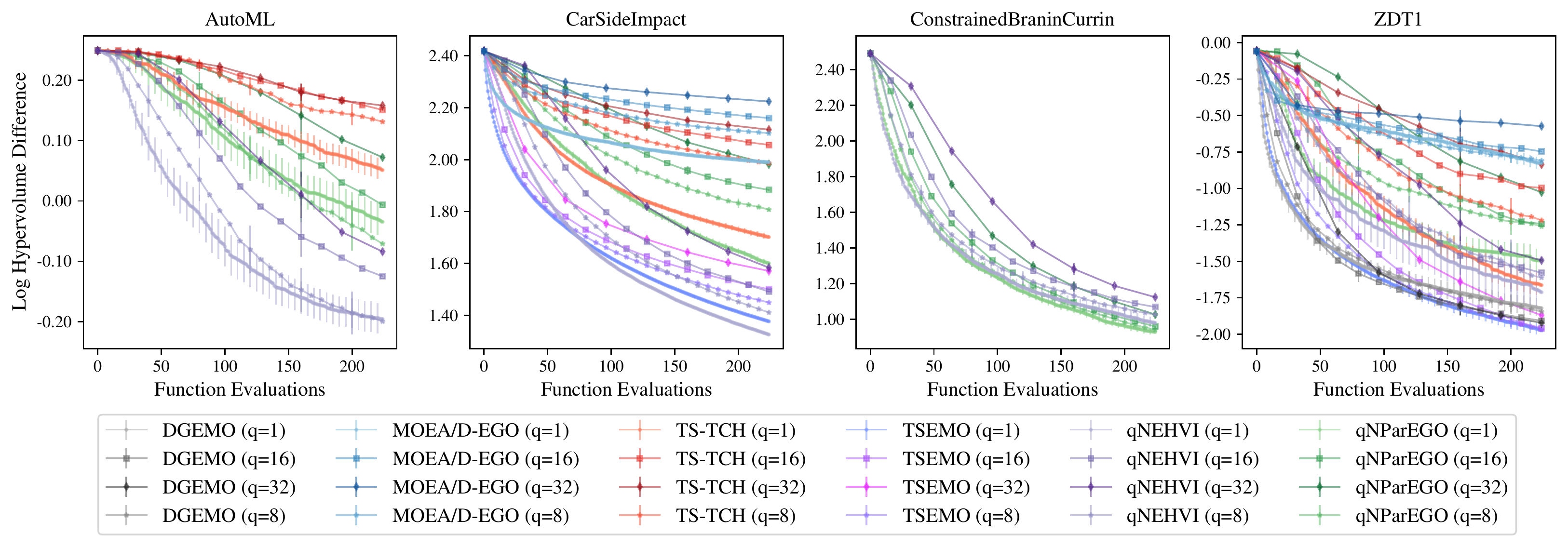}
    \vspace{-5pt}
    \vspace{-1ex}
    }
    \caption{\label{fig:anytime} Anytime optimization performance of batch acquisition functions using various $q$ over the number of function evaluations. To improve readability, we omit \qEHVI{}(\textsc{-PM}) in this figure because the \iep{} cannot scale beyond $q=8$  because of the exponential time and space complexity (running it on a GPU runs out of memory and running it on a CPU results in prohibitively slow wall times).}
\end{figure*}

\begin{figure*}[t]
\centering\subfloat[]{%
    \centering
    \includegraphics[width=\textwidth]{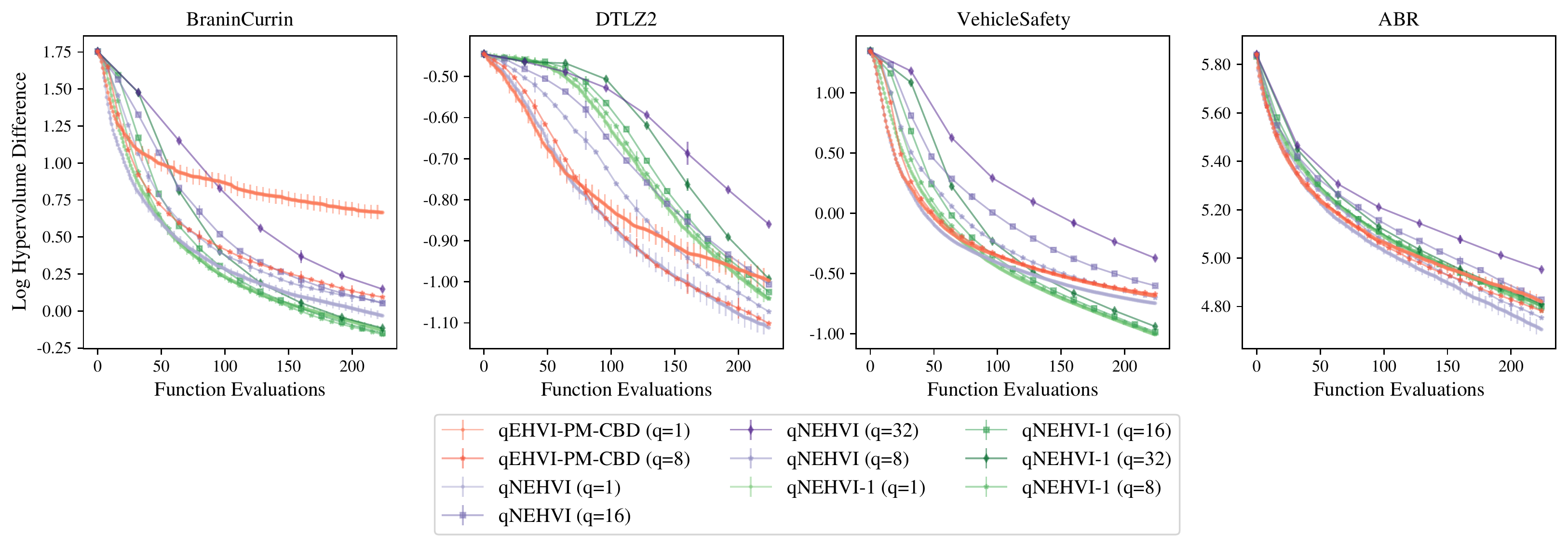}
    \vspace{-5pt}
    \vspace{-1ex}
}\\
\centering\subfloat[]{%
    \centering
    \includegraphics[width=0.91\textwidth]{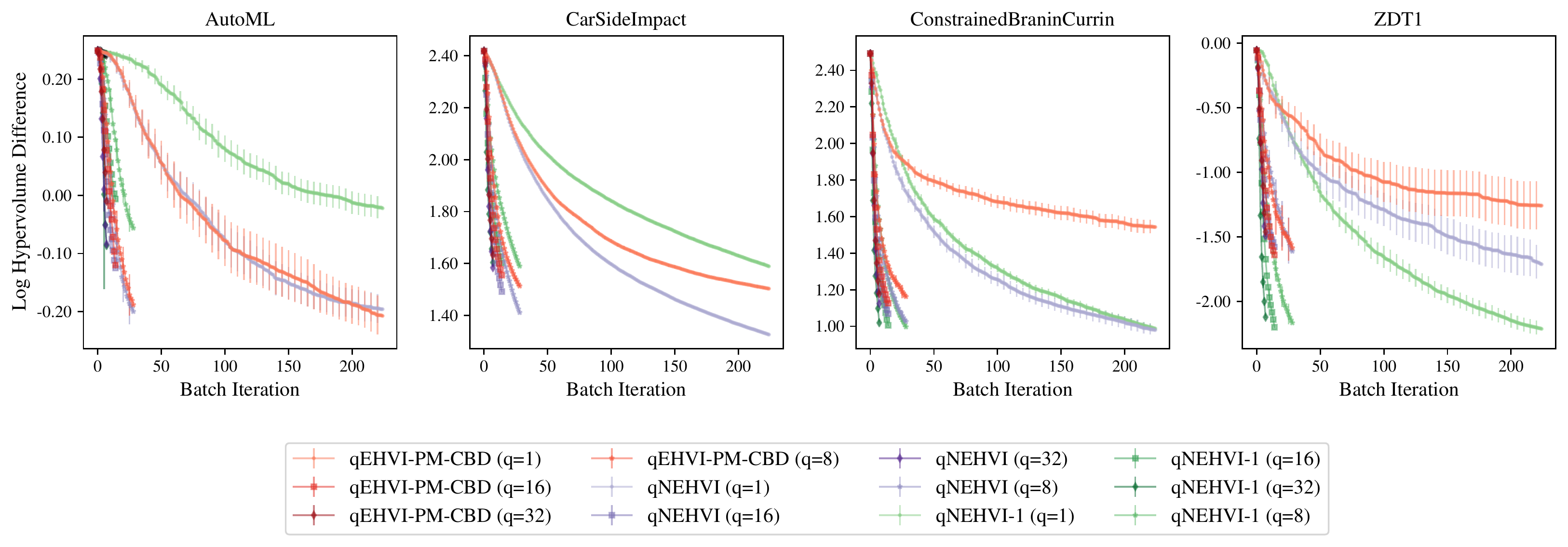}
    \vspace{-5pt}
    \vspace{-1ex}
    }
    \caption{\label{fig:anytime_ehvi}Anytime optimization performance of batch EHVI-based acquisition functions using various $q$ over the number of function evaluations.}
\end{figure*}

\begin{figure*}[h]
    \centering
    \includegraphics[width=0.83\textwidth]{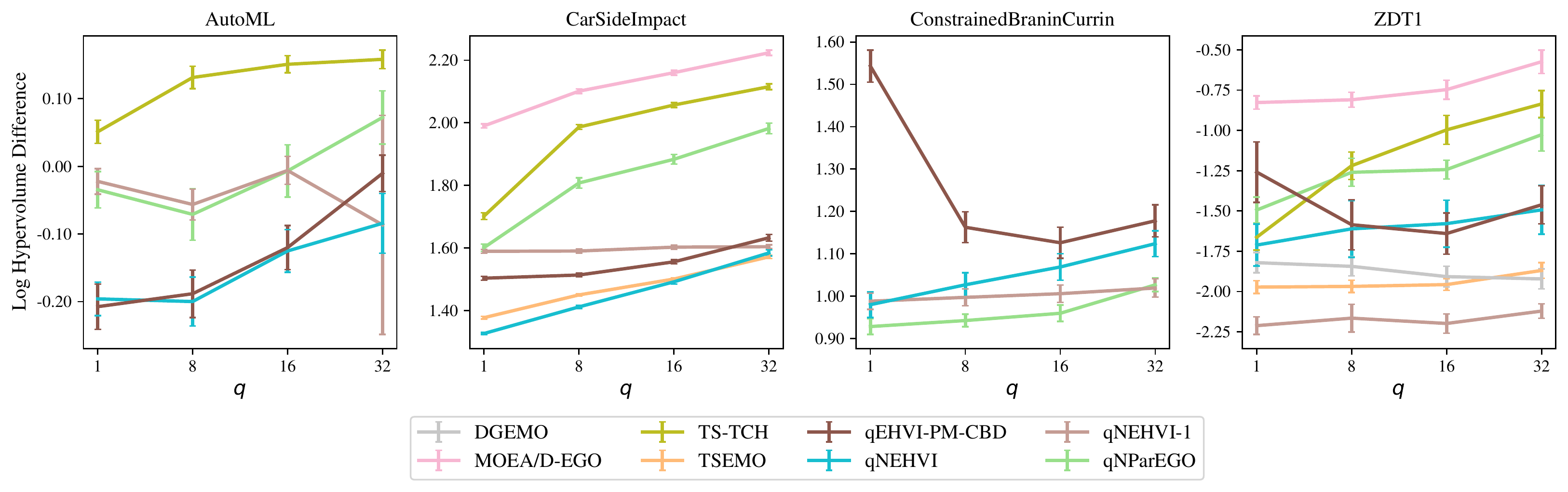}
    \caption{\label{fig:hv_over_q_additional}Final log hypervolume difference with various $q$ under a budget of 224 function evaluations. Smaller log hypervolume differences are better.}
\end{figure*}
\FloatBarrier
\subsection{Optimization Performance under Increasing Noise Levels}
Figure~\ref{fig:various_noise_levels} shows the sequential optimization performance of \qNEHVI{} and \TSHVI{} relative to \qEHVI{} and $q$NParEGO under increasing noise levels. \TSHVI{} achieves the best final hypervolume when the noise standard deviation $\sigma$ is less than $15\%$ of the range of each objective, but performs worse than \qNEHVI{} earlier in the optimization. \qNEHVI{} is the top performer in high-noise environments. We observe that all methods degrade as the noise level increases, however \qNEHVI{} consistently exhibits excellent performance relative to other methods and only \TSHVI{} is competitive and only in the low-noise regime.
\begin{figure*}[h]
    \centering
    \includegraphics[width=\textwidth]{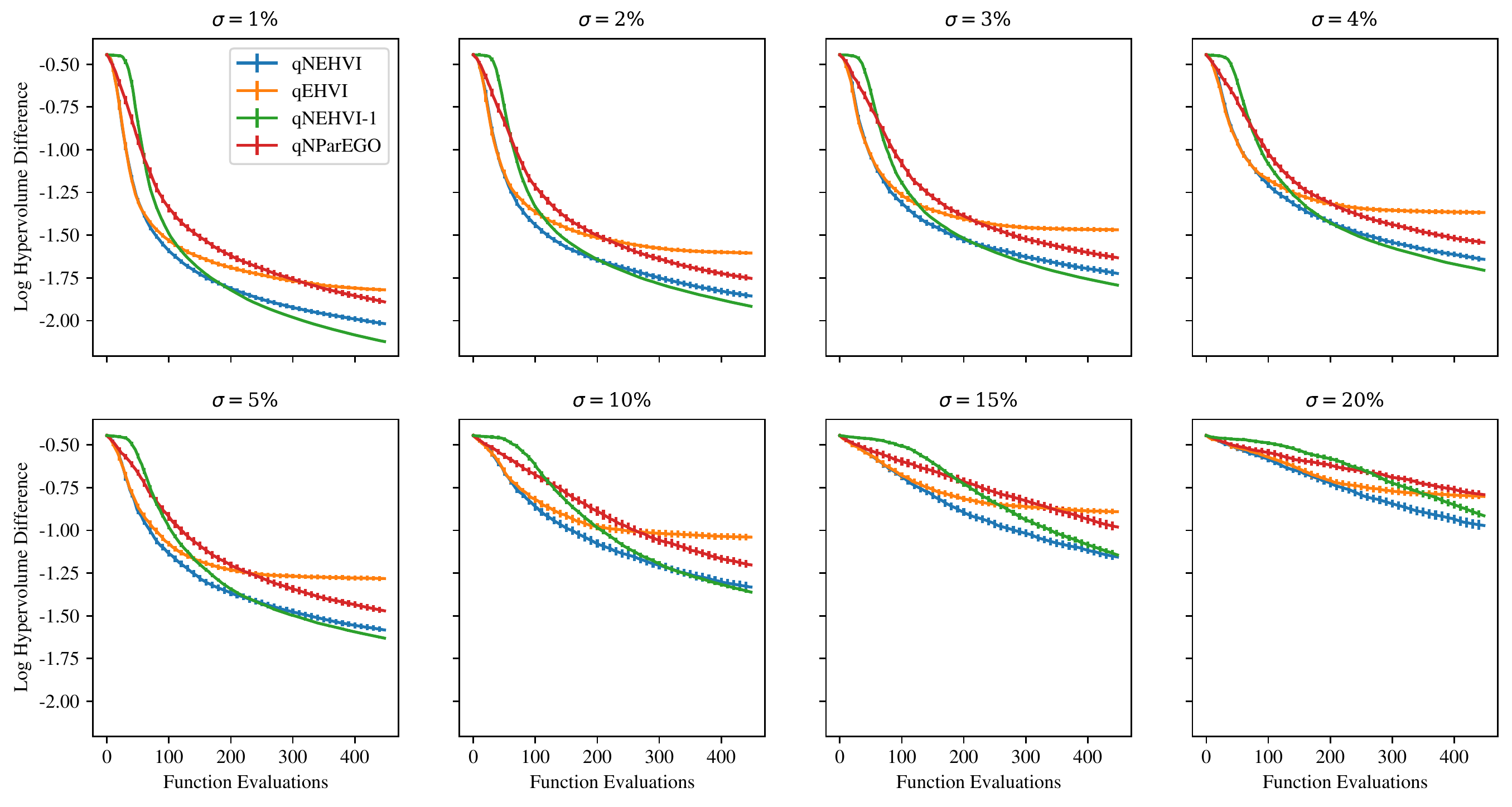}
    \caption{\label{fig:various_noise_levels}Sequential optimization performance under increasing noise levels on a DTLZ2 problem ($d=6$, $M=2$). $\sigma$ is the noise standard deviation, which we define as a percentage of the range of each objective over the entire search space. A noise level of 20\% is very high; for comparison, previous work on noisy MOBO has only considered noise levels of 1\% \citep{pesmo}.}
\end{figure*}

\subsection{Optimization Performance on Noiseless Benchmarks}
We include a comparison of optimization performance on \emph{noiseless benchmarks}. Figure \ref{fig:noiseless} shows that \qNEHVI{} performs competitively with \qEHVI{}(-PM-CBD) and outperforms DGEMO, TS-TCH and $q$NParEGO across all benchmark problems. \TSHVI{} is also a top performer on noiseless problems and both \qNEHVI{} and \TSHVI{} show little degradation in performance with increasing levels of parallelism.
\begin{figure*}[t]
\centering\subfloat[]{%
    \centering
    \includegraphics[width=\textwidth]{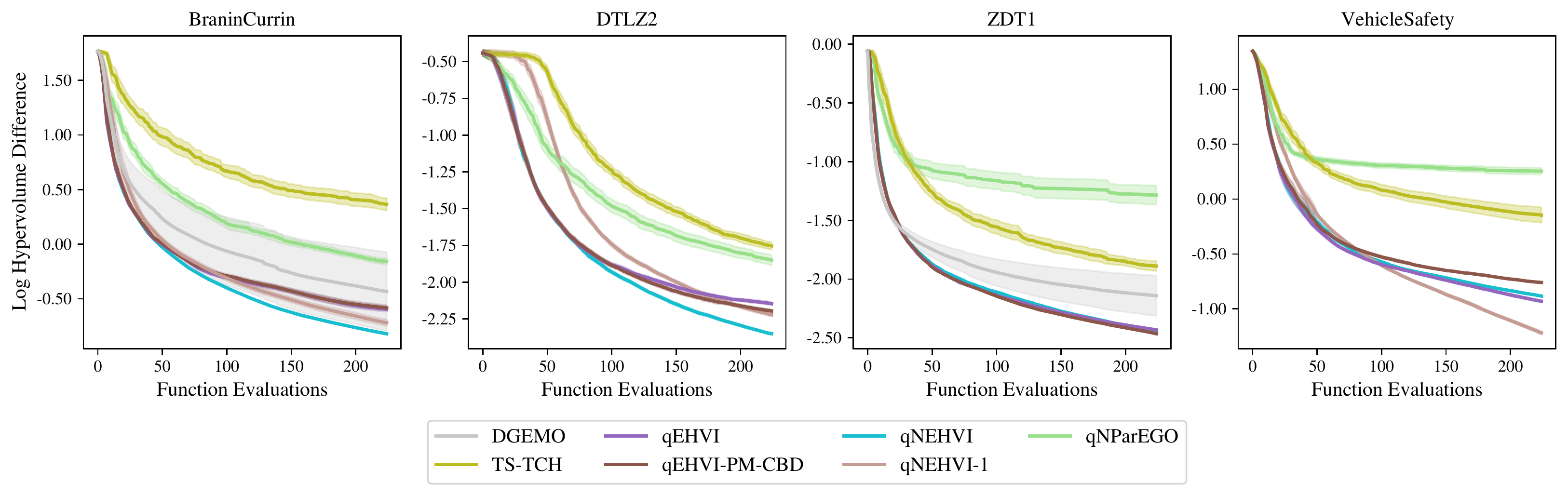}
    \vspace{-5pt}
    \vspace{-1ex}
}\\
\centering\subfloat[]{%
    \centering
    \includegraphics[width=0.91\textwidth]{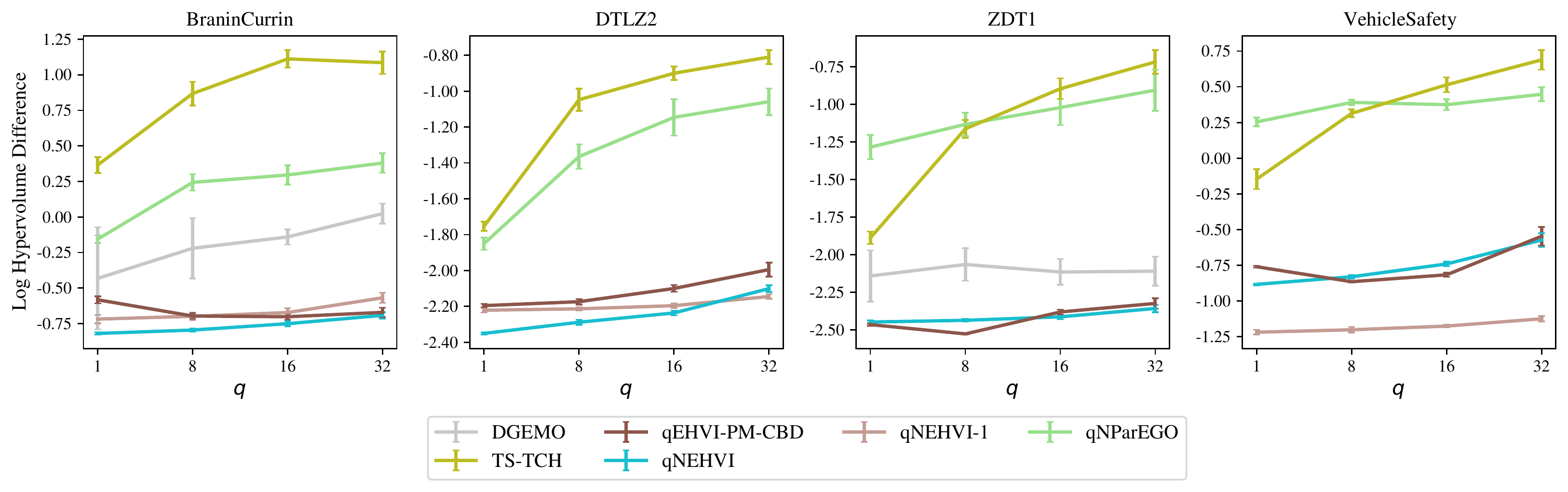}
    \vspace{-5pt}
    \vspace{-1ex}
    }
    \caption{\label{fig:noiseless}Sequential (a) and parallel (b) optimization Performance on \emph{noiseless} benchmarks.}
\end{figure*}
\FloatBarrier
\subsection{Performance of \TSHVI{} on 5-Objective Optimization}
\label{appdx:5obj}
We demonstrate that \TSHVI{} enables scaling to 5-objective problems. To our knowledge, no previous methods leveraging \EHVI{} or \HVI{} (e.g. DGEMO, TSEMO) considers 5-objective problems because of the super-polynomial complexity of the hypervolume indicator. Nevertheless, we show that using \cbd{} and a single sample path approximation, \TSHVI{} can be used for 5-objective optimization. As shown in Figure~\ref{fig:5obj}, \TSHVI{} outperforms $q$NParEGO and Sobol search. \TSHVI{} takes on average 73.53 seconds (with an SEM of 1.74 seconds) to generate each candidate, whereas
$q$NParEGO takes 11.37 seconds (with an SEM of 0.97 seconds).
\begin{figure}[h!]
\centering
\caption{\label{fig:5obj} Optimization performance on a 5-objective DTLZ2 problem ($d=6$) with $\sigma=5\%$.}
\includegraphics[width=0.32\textwidth]{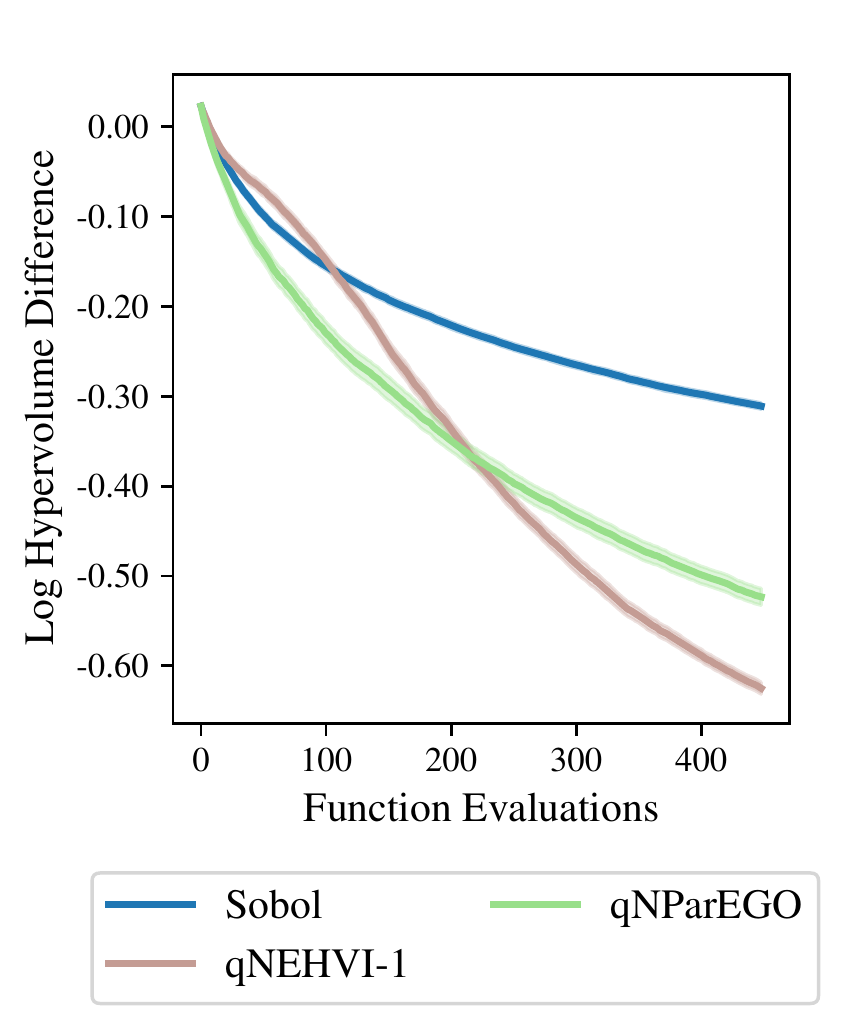}
\end{figure}

\FloatBarrier
\subsection{Performance compared against a Multi-Objective CMA-ES}
CMA-ES is an evolutionary strategy that is a strong method in single objective optimization, and many works have proposed extensions of CMA-ES to the multi-objective setting \citep{mocmaes, como_cma_es}. We compare \qNEHVI{} against the COMO-CMA-ES algorithm, which has been shown to outperform MO-CMA-ES on a variety of problems \citep{como_cma_es}.\footnote{COMO-CMA-ES is also the only multi-objective CMA-ES that we could find with an open-source Python implementation. We use the implementation available at \url{https://github.com/CMA-ES/pycomocma} under the BSD 3-clause license.}. We evaluate performance on the SphereEllipsoidal function from Bi-objective Black-Box Optimization Benchmarking Test Suite \citep{brockhoff2019using}, and we add zero-mean Gaussian noise to each objective with $\sigma=5\%$ of the range of each objective. We run COMO-CMA-ES with 5 kernels, the same initial quasi-random design as the BO methods, a population size of 10, and an initial step size of 0.2. As shown in Figure~\ref{fig:sphere_ellipsoidal}, the BO methods vastly outperform COMO-CMA-ES. \qNEHVI{} and $q$NParEGO perform best and are closely followed by \TSHVI{}.
\begin{figure}[h!]
\centering
\caption{\label{fig:sphere_ellipsoidal} Optimization performance on a 2-objective Sphere-Ellipsoidal problem ($d=5$) with $\sigma=5\%$.}
\includegraphics[width=0.48\textwidth]{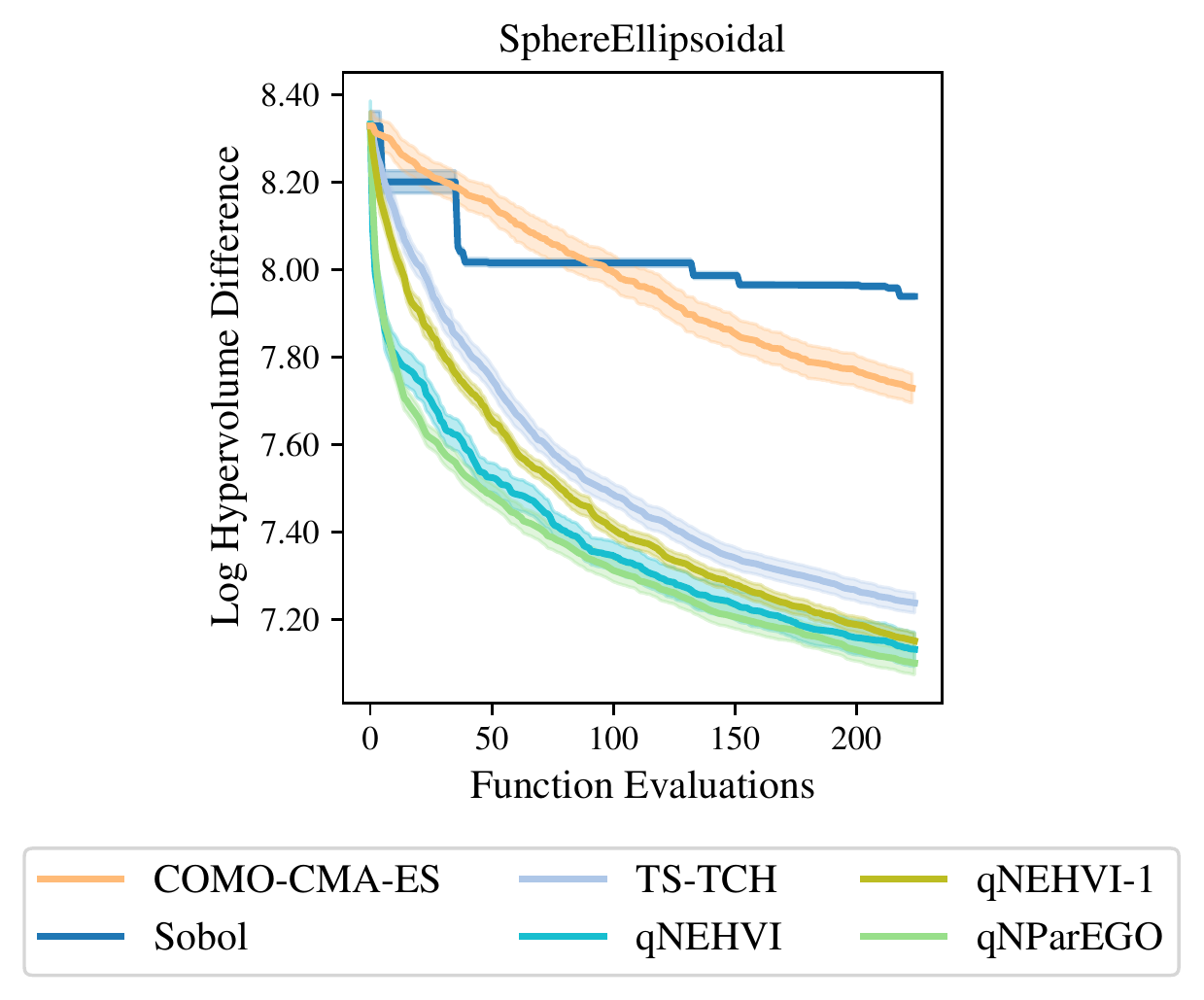}
\end{figure}

\FloatBarrier
\subsection{Importance of Accounting for Noise in DGEMO}
\label{appdx:sec:dgemo_augmented}
\begin{figure}[h!]
\centering
\caption{\label{fig:dgemo_noise_study} An illustration of the effect of noisy observations on the true noiseless Pareto frontiers identified by DGEMO (right) DGEMO-PM-NEHVI (left, see Appendix \ref{appdx:sec:dgemo_augmented}). Both algorithms are tested on a BraninCurrin synthetic problem, where observations are corrupted with zero-mean, additive Gaussian noise with a standard deviation of 5\% of the range of respective objective. All methods use sequential ($q=1$) optimization.}
\includegraphics[width=\textwidth]{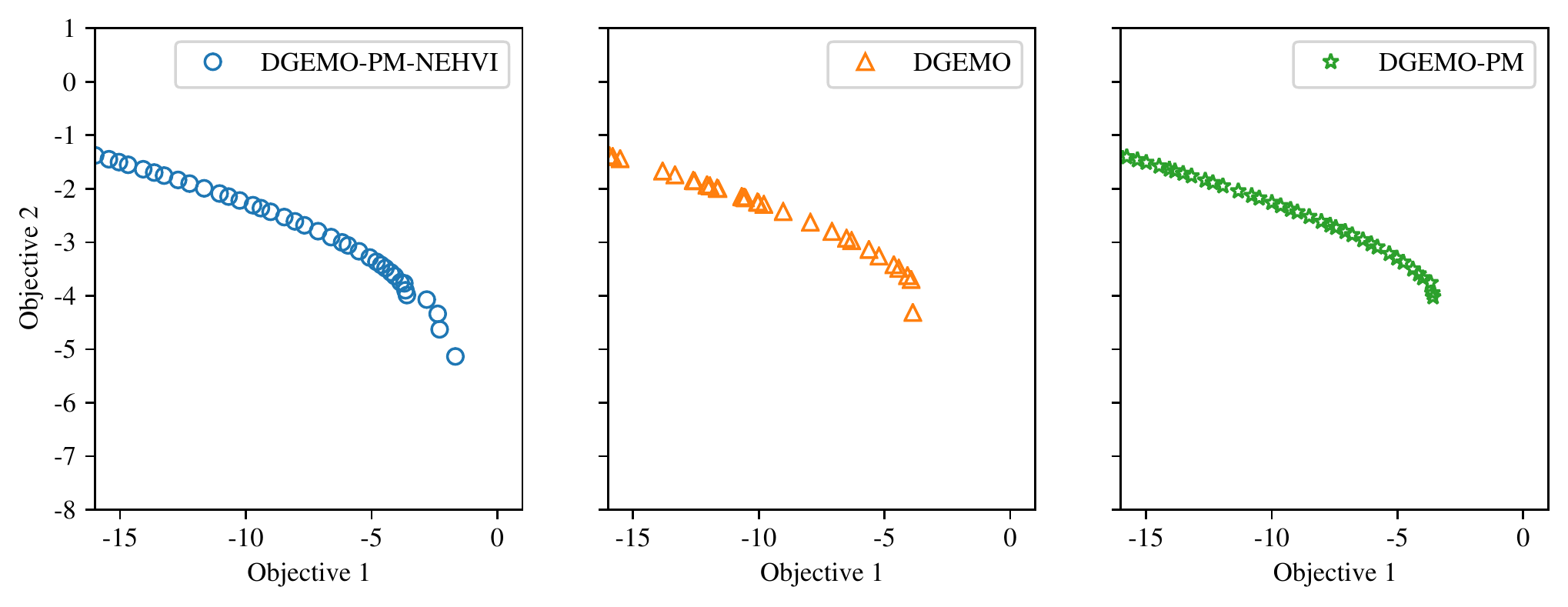}
\end{figure}
Similar to EHVI, DGEMO relies on the observed (noisy) Pareto frontier for batch selection. The right plot in Figure \ref{fig:dgemo_noise_study} shows that DGEMO exhibits the same clumping behavior in objective space in the noisy setting as \EHVI{}. While DGEMO's diversity constraints (with respect to the input parameters) make it slightly more robust to noise, the solutions are clustered and the bottom right corner of the Pareto frontier is not identified. In an attempt to mitigate these issues, we propose an augmented version of DGEMO, which we call DGEMO-PM-NEHVI, as follows: (i) we use the posterior mean at the previously evaluated points to estimate the in-sample Pareto frontier, which we hope will improve robustness to noise when selecting a discrete set of potential candidates using the DGEMO's first-order approximation of the Pareto frontier, and (ii) we use \qNEHVI{} rather than \HVI{} under the posterior mean as the batch selection criterion over the discrete set, subject to DGEMO's diversity constraints. We find that using \qNEHVI{} to  integrate over the uncertainty in the Pareto frontier over the previously evaluated points, results in identifying higher quality Pareto frontiers as shown in Figure \ref{fig:dgemo_noise_study}. We also include DGEMO-PM, which uses (i) but not (ii) for completeness. Not only does DGEMO-PM-NEHVI identify much more diverse solutions that provide better coverage across the Pareto frontier, but DGEMO-PM-NEHVI also identifies much better solutions on the lower right portion of the Pareto frontier than DGEMO and DGEMO-PM.  DGEMO-PM performs much better than DGEMO and is competitively with DGEMO-PM-NEHVI, which we speculate is because DGEMO uses a first-order approximation of the Pareto frontier (using the observed values or using the posterior mean for -PM variants) to generate a discrete set of candidates. Using the posterior mean in this step is important for regularizing against extreme observed values due to noise. \qNEHVI{} is only used as a filtering criterion for batch selection over that discrete set of candidates, subject to DGEMO's diversity constraints. Hence, \qNEHVI{} has limited control over the batch selection procedure.

DGEMO's first-order approximation fundamentally does not account for uncertainty in the Pareto frontier over previously evaluated points. Although one could integrate over the uncertainty in the in-sample Pareto frontier by generating a first-order approximation of the Pareto frontier under different sample paths, the graph cut algorithm would yield different families under each sample path. It is unclear how to set the diversity constraints in that setting. We leave this for future work.

\FloatBarrier
\section{Noisy Outcome Constraints}
While the focus of this work is on developing a scalable parallel hypervolume-based acquisition function for noisy settings, our MC-based approach naturally lends itself to support for constraints.
\subsection{Derivation of Constrained \NEHVI{}}
\label{appdx:sec:DqNEHVI:OutcomeConstraints}
The \NEHVI{} formulation in \eqref{eqn:ideal_nehvi} can be extended to handle noisy observations of outcome constraints. We consider the scenario where we receive noisy observations of $M$ objectives $\bm f(\bm x) \in \mathbb{R}^M$ and $V$ constraints $\bm c^{(v)} \in \mathbb{R}^V$, all of which are assumed to be ``black-box'': $\mathcal D_n = \{\bm x_i, \bm y_i, \bm b_i\}_{i=1}^n$ where $
\begin{bmatrix}
\bm y_i \\ 
\bm b_i
\end{bmatrix} \sim \mathcal N \bigg(\begin{bmatrix}
\bm f(\bm x_i) \\ 
\bm c(\bm x_i)
\end{bmatrix}, \Sigma_i\bigg), ~ \Sigma_i\in \mathbb R^{(M+V) \times (M+V)}$. We assume, without loss of generality, that $\bm c^{(v)}$ is feasible iff $\bm c^{(v)} \geq 0$. In the constrained optimization setting, we aim to identify the a finite approximate feasible Pareto set $$\mathcal P_\text{feas} = \{\bm f(\bm x) ~|~\bm x \in X_n, \bm c(\bm x) \geq \bm 0, ~\nexists ~~\bm x' :  \bm c(\bm x') \geq \bm 0 ~s.t. ~\bm f(\bm x') \succ \bm f(\bm x)\}$$ of the true feasible Pareto set
$$\mathcal P_\text{feas}^* = \{\bm f(\bm x)~~s.t.~~ \bm c(\bm x) \geq \bm 0, ~\nexists ~~\bm x' :  \bm c(\bm x') \geq \bm 0 ~s.t. ~\bm f(\bm x') \succ \bm f(\bm x)\}.$$
The natural improvement measure in the constrained setting is \emph{feasible} HVI, which we define for a single candidate point $\bm x$ as 
$$\HVIc{}(\bm f(\bm x), \bm c(\bm x) | \mathcal P_\text{feas}) := \HVI{}[\bm f(\bm x) | \mathcal P_\text{feas}] \cdot \mathbbm{1}[\bm c(\bm x) \geq \bm 0].$$
Taking the expectation over $\HVIc{}$ gives the constrained expected hypervolume improvement:
\begin{equation}
    \label{eqn:constrained_ideal_ehvi}
    \aEHVIc{}(\bm x) = \int \HVIc{}(\bm f(\bm x), \bm c(\bm x)| \mathcal P_{\text{feas}})p(\bm f,\bm c|\mathcal D)d \bm f d\bm c 
\end{equation}
For brevity, we define $\mathcal C_n = \bm c(X_n), \mathcal F_n = \bm f(X_n)$. The \emph{noisy expected hypervolume improvement} is then defined as:
\begin{equation}
    \label{eqn:constrained_ideal_nehvi}
    \aNEHVIc{}(\bm x) = \int \aEHVIc{}(\bm x| \mathcal P_{\text{feas}})p(\mathcal F_n, C_n|\mathcal D_n) d\mathcal F_n d\mathcal C_n.
\end{equation}

Performing feasibility-weighting on the sample-level allows us to include such auxiliary outcome constraints into the full Monte Carlo formulation given in \eqref{eqn:mc_nehvi} in a straightforward way:

\begin{equation*}
   \hataNEHVI{}_c(\bm x)
= \frac{1}{N}\sum_{t=1}^{N} \sum_{k=1}^{K_t}\Bigg[\prod_{m=1}^M \big[z_{k,t}^{(m)} - l_{k,t}^{(m)}\big]_{+} \prod_{v=1}^V\mathbbm{1}[c_t^{(v)}(\bm x) \geq 0]\Bigg]
\end{equation*}
where $z_{k,t}^{(m)} := \min \big[u_{k,t}^{(m)},\tilde{ f}_t^{(m)}(\bm x)\big]$ and $l_{k,t}^{(m)}, u_{k,t}^{(m)}$ are the $m^\text{th}$ dimension of the lower and upper vertices of the rectangle $S_{k,t}$ in the non-dominated partitioning $\{S_{1, t}, ..., S_{K_t, t}\}$ under the feasible sampled Pareto frontier $$\mathcal P_{\text{feas},t} = \mathcal P_{\text{feas}} = \{\tilde{\bm f}_t(\bm x) ~|~\bm x \in X_n, \tilde{\bm c}_t(\bm x) \geq \bm 0, ~\nexists ~~\bm x' :  \tilde{\bm c}_t(\bm x') \geq \bm 0 ~s.t. ~\tilde{\bm f}_t(\bm x') \succ \tilde{\bm f}_t(\bm x)\}.$$
In this formulation, the $\prod_{v=1}^V\mathbbm{1}[c_t^{(v)}(\bm x) \geq 0]$ indicates feasibility of the $t$-th sample.

To permit gradient-based optimization via exact sample-path gradients, we replace the indicator function (which is non-differentiable) with a differentiable sigmoid approximation with a  temperature parameter $\tau$, which becomes exact as $\tau\rightarrow \infty$:
\begin{align}
    \mathbbm{1}[c^{(v)}(\bm x) \geq 0] \approx s(c^{(v)}(\bm x); \tau) := \frac{1}{1 + \exp(-c^{(v)}(\bm x) /\tau)}
\end{align}
Hence,
\begin{equation*}
   \hataNEHVI{}_c(\bm x)
\approx \frac{1}{N}\sum_{t=1}^{N} \sum_{k=1}^{K_t}\Bigg[\prod_{m=1}^M \big[z_{k,t}^{(m)} - l_{k,t}^{(m)}\big]_{+} \prod_{v=1}^Vs(c_t^{(v)}(\bm x), \tau)\Bigg]
\end{equation*}
\subsection{Derivation of Parallel, Constrained \NEHVI{}}
The constrained \NEHVI{} can be extended to the parallel setting in a straightforward fashion. The joint constrained hypervolume improvement of a set of points $\{\bm x_i\}_{i=1}^q$ is given by
\begin{equation*}
    \HVIc{}(\{\bm f(\bm x_i), \bm c(\bm x_i)\}_{i=1}^q) =
    \sum_{k=1}^{K}\sum_{j=1}^q\sum_{X_j \in \mathcal X_j}(-1)^{j+1} \Bigg[\bigg(\prod_{m=1}^M\big[\bm z_{k, X_j}^{(m)} - l_k^{(m)}\big]_{+}\bigg)\prod_{\bm x' \in X_j} \prod_{v=1}^V\mathbbm{1}[c^{(v)}(\bm x') \geq 0]\Bigg].
\end{equation*}
and the constrained \qEHVI{} is \citep{daulton2020ehvi}:
$$\aqEHVIc{}(\xcand| \mathcal P_{\text{feas}}) = \int\HVIc(\bm f(\xcand),\bm c(\xcand) | P_\text{feas})p(\bm f, \bm c|\mathcal D_n)d\bm f d\bm c$$
Hence, the constrained \qNEHVI{} is given by:
\begin{equation}
\label{eqn:ideal_qnehvic}
\begin{aligned}
\aqNEHVI{}_c(\xcand) &= \int \aqEHVIc{}(\xcand| \mathcal P_{\text{feas}})p(\mathcal F_n, C_n|\mathcal D_n) d\mathcal F_n d\mathcal C_n\\
&= \int \HVIc(\bm f(\xcand),\bm c(\xcand) | P_\text{feas})p(\mathcal F_n, C_n|\mathcal D_n) d\mathcal F_n d\mathcal C_n
\end{aligned}
\end{equation}
Using MC integration for the integral in \eqref{eqn:ideal_qnehvic}, we have
\begin{align}
\hataqNEHVI{}_c(\xcand) &= \frac{1}{N}\sum_{t=1}^{N} \HVI_c(\tilde{\bm f}_t( \xcand), \tilde{\bm c}_t( \xcand) | \mathcal P_{\text{feas}, t}).
\end{align}
Under the \cbd{} formulation, the constrained \qNEHVI{} is given by
\begin{equation}
\begin{aligned}
    \hat{\alpha}_{q\textsc{NEHVI}}(\{\bm x_1, ...,\bm x_i\}) &=
    \frac{1}{N}\sum_{t=1}^N \HVIc\big(\{\tilde{\bm f}_t(\bm x_j), \tilde{\bm c}_t(\bm x_j)\}_{j=1}^{i-1}\}\mid \mathcal P_{\text{feas}, t}\big) 
    \\& \qquad + \frac{1}{N}\sum_{t=1}^N\HVIc\bigl(\tilde{\bm f}_t(\bm x_i), \tilde{\bm c}_t(\bm x_i) \mid \mathcal P_{\text{feas}, t} \cup \{\tilde{\bm f}_t(\bm x_j), \tilde{\bm c}_t(\bm x_j)\}_{j=1}^{i-1}\}\bigr).
\end{aligned}
\end{equation}

As in \eqref{eqn:cbd_qnehvi}, the first term is a constant when generating candidate $i$ and the second term is the \NEHVI{} of $\bm x_i$.

\FloatBarrier
\section{Evaluating Methods on Noisy Benchmarks}
Given noisy observations, we can no longer compute the true Pareto frontier over the in-sample points $X_n$. Moreover, the subset of Pareto optimal designs $\mathcal X^*_n = \{\bm x \in X_n, ~\nexists~ \bm x' \in X_n  ~s.t.~ \bm f(\bm x') \succ \bm f(\bm x)\}$ from the previously evaluated points may not be identified due to noise. For the previous results reported in this paper, we evaluate each method according to hypervolume dominated by the true unknown Pareto frontier of the noiseless objectives over the \emph{in-sample} points. In practice, decision-makers would often select one of the in-sample points according to their preferences. If the decision maker only has noisy observations, selecting an in-sample point may be preferable to evaluating a new out-of-sample point according to the model's beliefs. An alternative evaluation method would be use the model's posterior mean to identify what it believes is the Pareto optimal set of in-sample designs. The hypervolume dominated by the true Pareto frontier of noiseless objectives over that set of selected designs could be computed and used for comparing the performance of different methods. Results using this procedure are shown in Figure \ref{fig:hv_sequential_in_sample}. The quality of the Pareto set depends on the model fit. Several methods have worse performance over time (e.g. \qEHVI{} and \qParego{} on the ZDT1 problem), likely due to the collection of outlier observations that degrade the model fit. Nevertheless, \qNEHVI{} consistently has the strongest performance.

An alternative to the \emph{in-sample} evaluation techniques described above would be to use the model to identify the Pareto frontier across the entire search space (in-sample or \emph{out-of-sample}). For example, \citet{pesmo} used NSGA-II to optimize the model's posterior mean and identify the model estimated Pareto frontier. For benchmarking purposes on expensive-to-evaluate functions (e.g. in AutoML or ABR), this is prohibitively expensive. Moreover, such a method is less appealing in practice because a decision-maker would have to select out-of-sample points according to the posterior mean and then evaluate a set of preferred designs on the noisy objective to verify that the model predictions are fairly accurate at those out-of-sample designs. Therefore, in this work we evaluate methods based on the in-sample designs.

\begin{figure*}[t]
\centering\subfloat[]{%
    \centering
    \includegraphics[width=\textwidth]{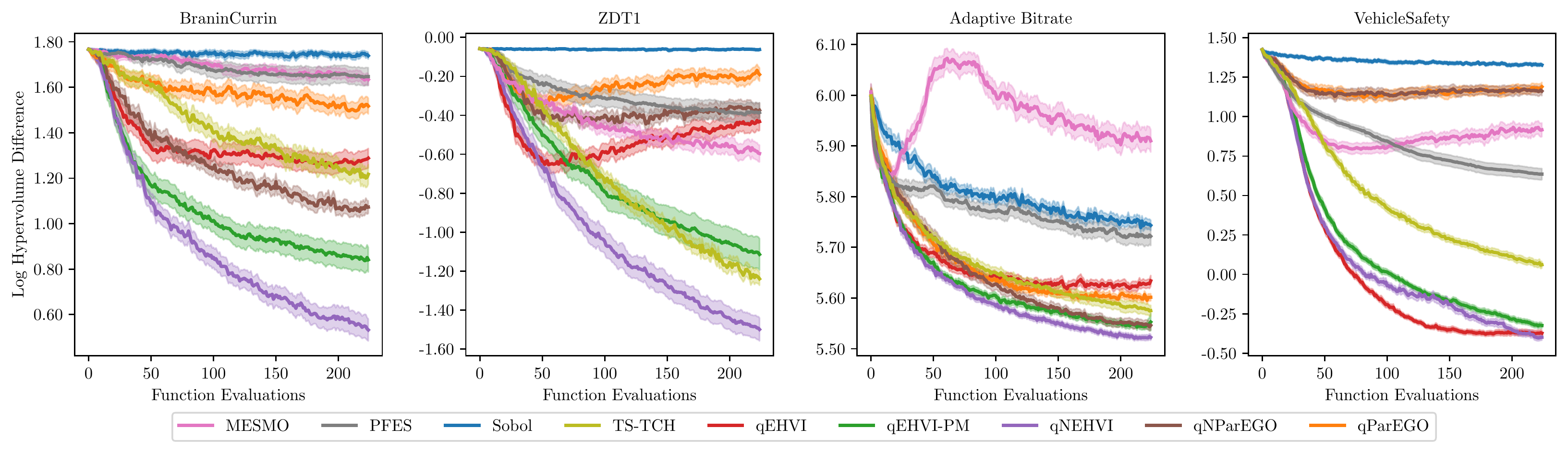}
    \vspace{-5pt}
    \vspace{-1ex}
}\\
\centering\subfloat[]{%
    \centering
    \includegraphics[width=0.82\textwidth]{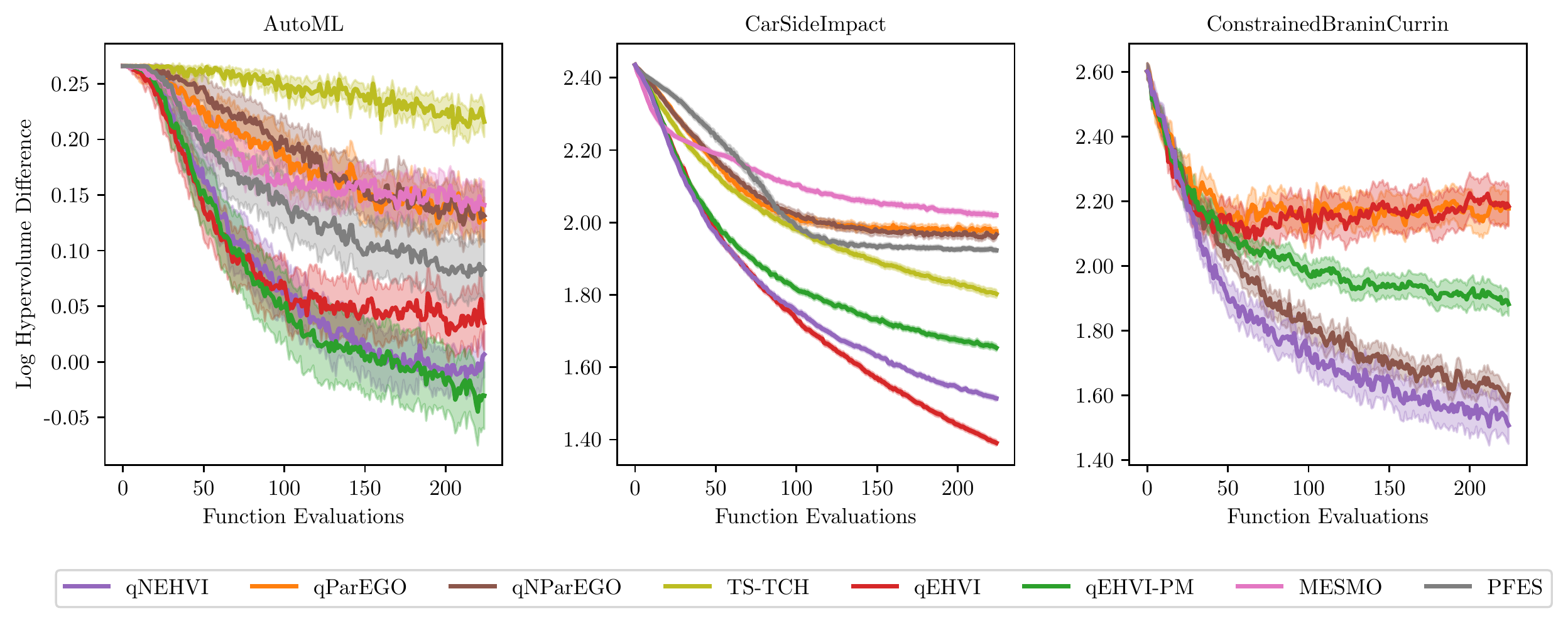}
    \vspace{-5pt}
    \vspace{-1ex}
    }
    \caption{\label{fig:hv_sequential_in_sample}Sequential optimization performance using based on the model-identified Pareto set across in-sample points.}
\end{figure*}


\end{document}